\documentclass{article} 
\usepackage[margin=1in]{geometry}
\usepackage{times}
\usepackage{natbib}
\setcitestyle{authoryear,round,citesep={;},aysep={,},yysep={;}}

\usepackage{amsmath,amsfonts,bm}

\def\eqref#1{equation~\ref{#1}}

\def\1{\bm{1}}

\DeclareMathAlphabet{\mathsfit}{\encodingdefault}{\sfdefault}{m}{sl}
\SetMathAlphabet{\mathsfit}{bold}{\encodingdefault}{\sfdefault}{bx}{n}

\usepackage{hyperref}
\hypersetup{hidelinks,
  pdftitle={What to Preserve in Recursive Computation: A Local Predictive Sufficiency Principle},
  pdfauthor={Peilin Wang, Feng Shiyang, Hongfu Gao, Cencheng Zhao, Di Yuan, Hui Chen, Guiguang Ding}}
\usepackage{url}
\usepackage{float}
\usepackage[most]{tcolorbox}
\newtcolorbox{propbox}{
    colback=white,
    colframe=blue!55!black,
    boxrule=0.7pt,
    arc=1.5pt,
    left=6pt,
    right=6pt,
    top=5pt,
    bottom=5pt
}
\usepackage{graphicx}
\usepackage{tikz}
\usepackage{caption}
\usepackage{placeins}
\usepackage{booktabs}
\usepackage{multirow}
\usepackage{makecell}
\usepackage{wrapfig}
\usepackage{needspace}
\usepackage{tikz}
\usepackage{tikzpagenodes}
\usepackage{capt-of}
\usepackage{needspace}
\usepackage{wrapstuff}
\usepackage{tabularx}
\usepackage{amsthm}
\usepackage{tabularx}
\usepackage{array}
\usepackage{wrapfig}
\usepackage{booktabs}
\usepackage{multirow}
\usepackage{placeins}
\usepackage{wrapfig}
\DeclareCaptionType{algorithm}
\usepackage{caption}
\theoremstyle{plain}
\newtheorem{prop}{Proposition}

\theoremstyle{definition}
\newtheorem{definition}{Definition}[section]

\newtcolorbox{algobox}{
    colback=white,
    colframe=black,
    boxrule=0.5pt,
    arc=0pt,
    outer arc=0pt,
    left=4pt,
    right=4pt,
    top=4pt,
    bottom=4pt,
    boxsep=0pt,
    before skip=0pt,
    after skip=0pt
}
\usetikzlibrary{positioning,arrows.meta}
\title{
What to Preserve in Recursive Computation:\\[-0.1em]
A Local Predictive Sufficiency Principle
}

\author{Peilin Wang$^{1,*}$, Feng Shiyang$^{1,*}$, Hongfu Gao$^{2}$,\\
Cencheng Zhao$^{1}$, Di Yuan$^{1}$, Hui Chen$^{3}$, Guiguang Ding$^{3}$\\[0.6em]
{\small $^{1}$Shandong University \quad $^{2}$National University of Singapore}\\
{\small $^{3}$Tsinghua University}\\
{\small $^{*}$Equal contribution.}}
\date{}

\begin{document}

\maketitle

\begin{abstract}
Recursive computation repeatedly compresses or reuses intermediate states, creating a simple tension: information that must remain useful across longer recursive paths is also exposed to more opportunities for loss before reaching the final prediction. Existing reconstruction or local-prediction objectives provide tractable supervision, but do not ensure that the retained information remains sufficient for subsequent recursive computation. We identify local predictive sufficiency with recursive predictive closure: controlling local predictive deficiencies at individual interfaces controls the resulting discrepancy at the root. We then turn this principle into a tractable training procedure. Starting from a variational characterization, we derive finite predictive tests and an empirical predictive deficiency that measures predictive value retained across compression. Its predictive sensitivities define margin-relaxed half-space constraints on parameter updates, and we project the host optimizer's proposed update onto their intersection only when predictive preservation would otherwise be violated. Across temporal graphs, language memory, vision-language-action control, and recursive self-improvement, the method matches or improves the corresponding host models under matched compression budgets, with larger gains under heavier recursive or memory demands, while better preserving predictive information across successive transformations. Crucially, the same task-agnostic predictive-preservation principle is instantiated across all four settings through host-compatible interventions while keeping the endpoint task, backbone, and evaluation protocol fixed. These results establish predictive preservation at recursive interfaces as a general training principle for recursive compression.

\end{abstract}

\section{Introduction}
\label{sec:introduction}

Recursive computation is widely used across modern learning systems, including 
graph and tree-structured propagation 
\citep{Goller1996,Scarselli2009,Tai2015,Wang2025TaskTrees,Feng2025GWM,Guo2025ToM}, 
recurrent context compression, embodied memory, and recursive self-improvement 
\citep{Kim2024CCM,Shi2026MemoryVLA,kim2026metan}. Despite their different 
implementations, these systems share a common structure: information is 
repeatedly summarized into an intermediate state, passed through a recursive 
interface, and reused by subsequent computation. Earlier information therefore 
undergoes more rounds of aggregation and compression, and long-range information 
must pass through more recursive interfaces before affecting the final prediction. 
Thus, the information that must remain useful over longer paths is 
exposed to more opportunities for loss along the way. Compression may attenuate information 
\citep{Alon2021,Topping2022,Banerjee2022Oversquashing}.
 
At each recursive interface, the model therefore faces a concrete question: 
what should the intermediate state retain? Reconstruction objectives encourage 
it to preserve information needed to recover its input \citep{Socher2011}, 
while predictive-state approaches represent state through predictions of future 
observations \citep{Littman2001,Kulesza2015Insufficient,Hefny2018}. 
Information-theoretic criteria characterize predictive information more directly 
\citep{Tishby1999,TishbyZaslavsky2015,EstellaAguerri2018DIB,Poole2019Variational}, 
but are difficult to estimate and optimize for high-dimensional intermediate 
states. Recursive computation therefore calls for a tractable criterion that 
specifies what each interface should preserve for subsequent prediction. 
 
\textbf{Local Predictive Sufficiency.} 
A natural approach is to ask every intermediate state to preserve information 
about the final prediction. Yet the states most exposed to repeated compression 
are also those for which the final target is separated from the 
recursive interface by more subsequent transformations, making this preservation 
signal increasingly indirect \citep{PolyanskiyWu2017}. 
We instead assess preservation where compression occurs. At each recursive 
interface, we ask whether the compressed state retains the predictive information 
available before compression for the adjacent predictive pair at that interface. 
We call this property \emph{local predictive sufficiency}. Under the recursive 
model in Sec.~\ref{sec:problem-formulation}, recursive predictive closure connects 
these local conditions to protection of the root target. We then derive an 
empirical predictive deficiency that approximates the predictive value lost 
across compression. In projection-based instantiations, the optimizer first 
produces a task-driven proposal, and predictive feasibility projection modifies 
that proposal only when it violates the margin-relaxed predictive constraints. 
The broader framework changes how information is retained at recursive interfaces 
while leaving the endpoint task and evaluation protocol unchanged.
 
We evaluate the framework in four settings where recursive state compression 
arises naturally: temporal graphs, language memory, vision-language-action 
control, and recursive self-improvement. Across these settings, preserving 
predictive structure at recursive interfaces consistently improves downstream 
performance, with larger gains when relevant information must survive longer 
recursive chains. On DailyDialog, our method consistently improves over the 
corresponding CCM-merge host across the tested memory depths. On LIBERO-Mem T3, 
it increases full-task success from $1.7\pm2.9\%$ to $11.7\pm2.9\%$ across 
three independently trained seeds. Similar depth-dependent gains appear in 
temporal graphs and across recursive self-improvement iterations. Mechanism 
analyses further show slower attenuation of the predictive statistic across 
successive recursive interfaces. 
 
Our contributions are threefold. \textit{(1) Theory.} We formulate local 
predictive sufficiency and establish recursive predictive closure and root 
predictive stability. \textit{(2) Learning.} We derive an empirical predictive 
deficiency and use its predictive sensitivities to constrain task proposals 
that would degrade predictive structure across recursive interfaces. 
\textit{(3) Evaluation.} We validate the principle in four recursive settings 
and show that its benefits become more pronounced as information must survive 
longer recursive chains. Together, these contributions provide a common view of 
what recursive intermediate states should preserve and a practical mechanism for 
protecting it during learning.

\section{Preliminaries and Problem Formulation}
\label{sec:problem-formulation}

\subsection{Recursive Compression Setup}

\noindent\textbf{Recursive setting.}
We unroll recursive computation over time into a finite rooted tree with node set \(V\) and root \(\rho\), where each \(v\in V\setminus\{\rho\}\) denotes a recursive interface and \(\rho\) denotes the root target site. We write \(\operatorname{pa}(v)\) for the next interface receiving information from \(v\). At each non-root interface \(v\), \(U_v\) denotes the state assembled from the information available in the corresponding subtree immediately before compression. The variable \(C_v\) denotes the observed interface information that is retained directly across this update without passing through the compression map, and \(C:=(C_v)_{v\in V}\) collects these retained contexts. A measurable compression map \(\Gamma\) produces the width-\(r\) representation \(Z_v^\Gamma:=\Gamma(U_v)\), which is then passed to the remaining recursive computation.

\medskip
\noindent\textbf{Predictive task.}
Each recursive interface $v$ is associated with an observed predictive event
$Y_v\in\mathcal{Y}$,\footnote{
Examples include interaction outcomes, subsequent responses, and recorded
actions. The same observed event may be associated with multiple recursive
interfaces, while intermediate transformations without an associated event are
included in the recursive transition between adjacent recursive interfaces.
}
with $Y_\rho$ denoting the root target. Our goal is to learn
$\Gamma$ such that repeated compression preserves the predictive information
required for $Y_\rho$. The Information Bottleneck frames compression relative
to a relevance variable \citep{Tishby1999,TishbyZaslavsky2015}, and future
outcomes provide a natural choice for predictive representations
\citep{Still2014,MarzenCrutchfield2016}. We formalize this requirement through
predictive sufficiency.

\begin{definition}[Predictive Sufficiency]
\label{def:predictive-sufficiency}
Following statistical sufficiency \citep{CoverThomas2006}, $Z_v^\Gamma$ is \emph{predictively sufficient} for target $Y$ given $C$ if
\[
\mathcal{L}(Y\mid U_v,C)
=
\mathcal{L}(Y\mid Z_v^\Gamma,C)
\qquad \text{a.s.},
\]
or equivalently,
\[
I(Y;U_v\mid Z_v^\Gamma,C)=0.
\]
\end{definition}

\subsection{Predictive Preservation under Recursive Structure}

For each \(v\in V\), let \(Y_v\) denote the predictive event associated with that recursive stage, with \(Y_\rho\) denoting the root target. The root target \(Y_\rho\) is the global relevance variable, since it is ultimately the prediction that compression should preserve. For a deeper state, however, its relation to \(Y_\rho\) is mediated by a longer sequence of recursive transformations. By the data-processing principle, mutual information cannot increase along such a chain and may contract under repeated transformations \citep{PolyanskiyWu2017}. Hence, as recursive depth grows, the root target can become an increasingly indirect guide for deciding what information an intermediate state should retain.

Recursive computation provides a way to exploit nearby predictive events. Once unrolled as a rooted tree, information from a subtree reaches the rest of the computation through successive parent--child interfaces, connecting nearby predictive events along the same rootward path. Proximity alone does not determine relevance to the root target. Conditional on \(C\), the tree structure induces a local separation: subtrees that branch apart interact only through their shared ancestral path \citep{Lauritzen1996}. At each recursive interface, incoming child representations are aggregated into \(U_v\) before recursive propagation. We require the pre-compression interface state to screen the adjacent predictive events from those incoming child states once \((U_v,C)\) is given. Appendix~\ref{app:recursive-closure} states this condition as (M3), together with the predictive-tree condition (M1) and recursive host construction (M2). When the complete interface record $U_v$ contains the incoming child tuple, (M3) holds automatically. This naturally motivates using adjacent predictive events along each rootward path. For every non-root recursive interface \(v\), define the adjacent predictive pair \(S_v:=(Y_v,Y_{\operatorname{pa}(v)})\). Applying Definition~\ref{def:predictive-sufficiency} to $S_v$ gives local predictive sufficiency. Consecutive pairs share an intermediate predictive event, allowing predictive sufficiency to compose recursively toward the root.

\begin{prop}[Recursive Predictive Closure]
\label{prop:recursive-closure}
Suppose the recursive computation satisfies (M1)--(M3) of Appendix~\ref{app:recursive-closure}.
Let
\(U := (U_v)_{v\neq\rho}\) and
\(Z^\Gamma := (Z_v^\Gamma)_{v\neq\rho}\).
Then

$$
\left[
I(S_v;U_v\mid Z_v^\Gamma,C)=0,
\quad
\forall\,v\in V\setminus\{\rho\}
\right]
\Longrightarrow
I(Y_\rho;U\mid Z^\Gamma,C)=0.
$$

\end{prop}

Proposition~\ref{prop:recursive-closure} shows that root-target sufficiency follows from local predictive sufficiency at the recursive interfaces. If
\(Z_v^\Gamma\) preserves the information in \(U_v\) required for \(S_v\) at every recursive interface, the shared predictive events allow these conditions to compose toward
\(Y_\rho\). The proof is given in Appendix~\ref{app:recursive-closure}. This motivates the local predictive deficiency
\(I(S_v;U_v\mid Z_v^\Gamma,C)\) under a prescribed compression width. Driving
these local predictive deficiencies to zero recovers the root-level guarantee. Since the
criterion depends on unknown conditional laws, direct estimation is difficult
for finite, high-dimensional representations.
Section~\ref{sec:predictive-compression-learning} develops a tractable
empirical realization.

\section{Learning for Predictive Preservation}
\label{sec:predictive-compression-learning}

\subsection{Predictive Deficiency Estimation}

Local predictive sufficiency specifies what compression should preserve:
replacing the pre-compression state $U_v$ with $Z_v^\Gamma$ should retain
the information needed to predict the adjacent predictive pair $S_v$. The exact
variational criterion $\inf_{q_v}\mathcal V_v(\Gamma,q_v)$ captures this
requirement, but depends on an unknown conditional law and is unavailable
during training. We make it tractable by restricting prediction to finite
tests, profiling the predictive readout, and estimating the resulting
population quantity from finite samples. Figure~\ref{fig:lpse_construction}
summarizes this construction; formal derivations and approximation bounds are
given in Appendices~\ref{app:local-variational-characterization}
and~\ref{app:tractable-local-objective}.

\begin{figure}[H]
    \centering
    \includegraphics[width=\linewidth]{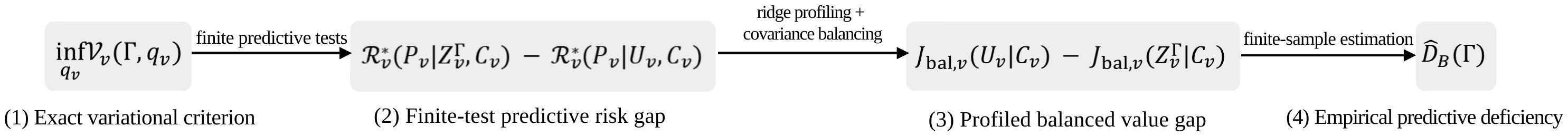}
    \caption{
    \textbf{From the exact predictive criterion to empirical predictive deficiency.}  
The construction uses finite tests, covariance-balanced readouts, and estimates the population quantity.
}
    \label{fig:lpse_construction}
\end{figure}
\vspace{0.4em}

\textbf{Restricting to finite predictive tests.}
Because the conditional law of $S_v$ is unavailable, we replace the full-law
comparison by the fixed test representation
$P_v := T_S\!\left(\chi(C_v)\otimes\phi(Y_v,Y_{\mathrm{pa}(v)})\right)$,
where $\chi$ and $\phi$ are fixed feature maps and $T_S$ is a bounded-width
sketch. For a candidate representation $X_v$, let
$\mathcal R_v^\star(P_v\mid X_v,C_v)
:= \inf_g \mathbb E[\|P_v-g(X_v,C_v)\|_2^2]$
denote its optimal predictive risk. The finite-test predictive risk gap
$\Delta_v^{\mathrm{test}}(\Gamma)
:= \mathcal R_v^\star(P_v\mid Z_v^\Gamma,C_v)
-\mathcal R_v^\star(P_v\mid U_v,C_v)$
measures the predictive loss induced by compression on this fixed test family.
The localization from the retained global context $C$ to the visible $C_v$ is
characterized in Appendix~\ref{app:tractable-local-objective}.

\textbf{Profiling the predictive readout.}
The finite-test risk still optimizes over unrestricted measurable predictors.
We therefore restrict the readout to ridge regression and profile out its
parameters analytically. Full-covariance balancing then yields the profiled
balanced predictive value $J_{\mathrm{bal},v}$, giving
$\Delta_v^{\mathrm{bal}}(\Gamma)
:= J_{\mathrm{bal},v}(U_v\mid C_v)
- J_{\mathrm{bal},v}(Z_v^\Gamma\mid C_v)$.
Here $\mathcal R_v^\star$ is a risk to be minimized, whereas
$J_{\mathrm{bal},v}$ is a predictive value to be maximized, hence the reversed difference.

\textbf{Estimating predictive deficiency from finite samples.} 
We replace the remaining population covariance operators by regularized 
empirical estimates over an optimization window $\mathcal B$, so that the
predictive geometry is estimated from a shared set of recursive realizations,
which then yields
$\widehat{\mathcal D}_{\mathcal B}(\Gamma) 
:= \widehat J_{\mathrm{bal},\mathcal B}(U) 
- \widehat J_{\mathrm{bal},\mathcal B}(Z^\Gamma)$.
The pre-compression term is treated as a stop-gradient reference, so the 
first-order sensitivities used for optimization are taken through 
$Z^\Gamma$ while the reference predictive value remains fixed within the
window. Full-covariance balancing is used in the main method, while the 
diagonal form is evaluated only as a covariance approximation in the 
ablation study. Section~\ref{sec:predictive-constraints} uses these 
predictive sensitivities to constrain task proposals.

\subsection{Margin-Relaxed Predictive Constraints}
\label{sec:predictive-constraints}

\begin{wrapstuff}[
    r,
    width=0.54\textwidth,
    hsep=1.0em,
    vsep=0.2pt,
    type=figure
]
    \centering
    \includegraphics[width=\linewidth]{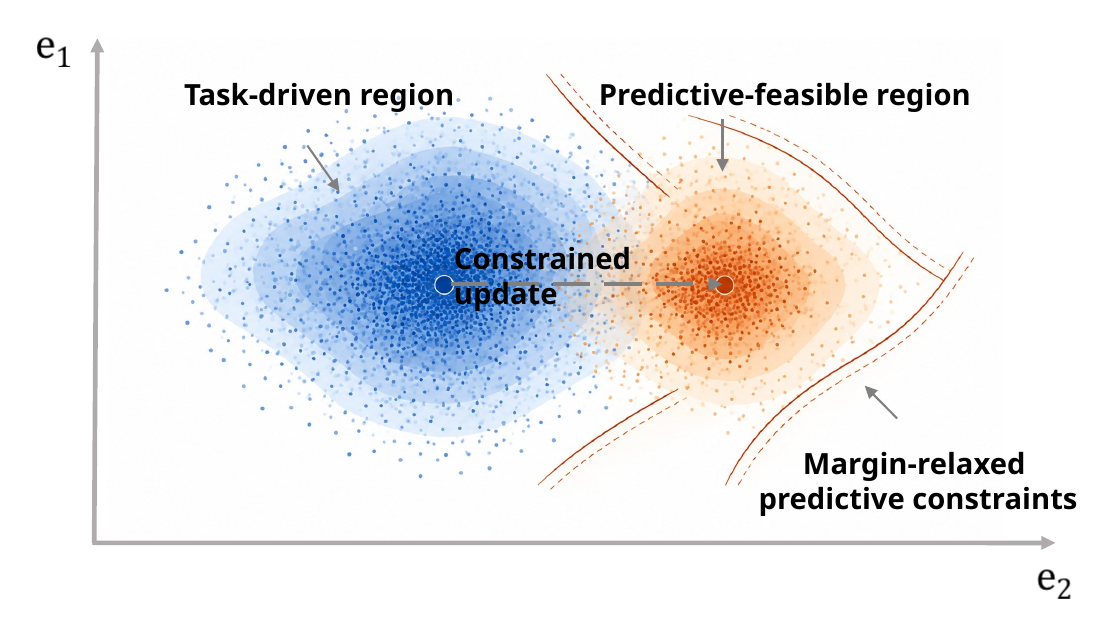}
    \captionsetup{skip=1pt}
    \caption{\textbf{Margin-relaxed predictive constraints.}
    Task proposal and constrained update in the predictive-feasible region.}
    \label{fig:predictive-feasibility}
\end{wrapstuff}

Task optimization may improve the root-target objective while increasing 
empirical predictive deficiency at intermediate recursive interfaces. 
We therefore constrain task proposals using these predictive 
sensitivities, with a margin for task adaptation. For a realized compressed 
state $z_i^\Gamma$, let 
$a_i:=\nabla_{z_i^\Gamma}\widehat{\mathcal D}_{\mathcal B}(\Gamma)$ 
denote its predictive sensitivity in representation space. It describes how 
the empirical predictive deficiency changes under a perturbation of 
that compressed state. Pulling this sensitivity back through the compressor 
Jacobian gives 
$q_i:=(\nabla_\Gamma z_i^\Gamma)^\top a_i$ 
in the trainable parameter space. For a small update $d$, 
$q_i^\top d$ measures first-order contribution to 
the change in empirical predictive deficiency. Thus, $a_i$ identifies the 
predictive sensitivity of the compressed representation, while $q_i$ is the 
corresponding predictive sensitivity in the trainable parameter space. Because 
all sensitivities are differentiated from the same window-level statistic 
$\widehat{\mathcal D}_{\mathcal B}$, they reflect a shared empirical surrogate, not distinct population $\delta_v$ values.

We aggregate these local directions at the granularity of the predictive 
constraints within the current optimization window. For group 
$\mathcal G_g$, let 
$\bar q_g:=\sum_{i\in\mathcal G_g}q_i$ and 
$h_g:=\bar q_g/\|\bar q_g\|_2$ when $\bar q_g\ne0$; zero groups impose no 
constraint. In temporal graphs, each group contains realizations from one 
computation tree and one interface type, so each constraint captures their 
joint first-order effect without introducing additional depth-dependent 
weighting. Stacking $h_g^\top$ gives 
$H_t\in\mathbb R^{m\times p}$. For the task proposal $t_t$, define 
$b_t:=\kappa\|t_t\|_2\mathbf 1_m$ and 
$\mathcal F_t:=\{d:H_td\le b_t\}$. Each row of $H_t$ limits motion along one 
predictive sensitivity, and their intersection defines the 
predictive-feasible region. The tolerance $\kappa\ge0$ controls the allowed 
first-order increase in predictive deficiency. 
Figure~\ref{fig:predictive-feasibility} shows how the margin-relaxed predictive constraints restrict the task proposal to yield a predictive-feasible update.

\Needspace{11\baselineskip}
\begin{prop}[Dual Characterization of Predictive Feasibility Projection]
\label{prop:predictive-feasibility}
The predictive feasibility projection of the task proposal onto $\mathcal F_t$ has a unique constrained update
$d_t^\star$. There exists a nonnegative dual optimal multiplier, and every
dual optimum yields the same $d_t^\star$. Moreover, the predictive correction
lies in the span of the predictive constraint directions. Specifically,
\[
\alpha_t^\star
\in
\underset{\alpha\ge0}{\operatorname{argmax}}
\left\{
-\frac12\alpha^\top(H_tH_t^\top)\alpha
+\alpha^\top(H_tt_t-b_t)
\right\},
\qquad
d_t^\star=t_t-H_t^\top\alpha_t^\star .
\]
\end{prop}

Proposition~\ref{prop:predictive-feasibility} gives a tractable
characterization of the constrained update. The
correction is restricted to the span of the predictive constraint directions,
while the dual optimization is carried out in the $m$-dimensional constraint
space rather than the full $p$-dimensional parameter space. The primal--dual
derivation, feasibility properties, and computational complexity are given
in Appendix~\ref{app:tractable-local-objective}; solver and feasibility
settings are reported in Appendix~\ref{app:implementation}.
The dual characterization permits any nonnegative $b_t$; $b_t=\kappa\|t_t\|_2\mathbf 1_m$ is our full-covariance choice.

\subsection{Predictive Feasibility Optimization}
\label{sec:predictive-optimization}

\begin{wrapfigure}{r}{0.54\textwidth} 
    \vspace{-0.7em} 
    \small 
 
    \begin{minipage}{\linewidth} 
 
    \hrule 
    \vspace{0.18em} 
 
    \captionof{algorithm}{Predictive Feasibility Optimization} 
    \label{alg:predictive-feasibility} 
 
    \vspace{-0.35em} 
    \hrule 
    \vspace{0.18em} 
 
    \noindent\textbf{Require:} 
 
        Window $\mathcal B$, compressor $\Gamma_t$, 
        optimizer $\mathcal O$ with state $s_t$, tolerance $\kappa$ 
 
        \par\noindent\textbf{Ensure:} Updated compressor $\Gamma_{t+1}$\par\begin{enumerate} 
 
        \item Run the host forward pass and cache $(P,U,Z^\Gamma)$ 
 
        \item Construct the grouped predictive constraints $(H_t,b_t)$ 
        from the window statistics and tolerance $\kappa$ 
 
        \item Obtain the task proposal 
        $t_t \gets 
        \Delta_{\mathcal O} 
        (\Gamma_t,s_t,\nabla_\Gamma\mathcal L_{\mathrm{task}})$ 
 
        \item Solve for $d_t^\star$ using 
        Proposition~\ref{prop:predictive-feasibility} 
 
        \item Advance $s_t$ once and set 
        $\Gamma_{t+1}\gets\Gamma_t+d_t^\star$ 
 
    \end{enumerate} 
 
    \vspace{0.15em} 
    \hrule 
 
    \end{minipage} 
 
    \vspace{-0.6em} 
\end{wrapfigure}

Algorithm~\ref{alg:predictive-feasibility} summarizes the training procedure.
Within each optimization window, the host forward pass provides the predictive
statistics, from which the interface sensitivities are pooled into the grouped
constraint matrix $H_t$ defined in Section~\ref{sec:predictive-constraints}.
The host optimizer produces the task proposal $t_t$ on the trainable
compressor parameters. Proposition~\ref{prop:predictive-feasibility} then
maps $(H_t,t_t)$ to the constrained update $d_t^\star$, which is written back
to $\Gamma$. The host forward computation and task objective remain unchanged.

\noindent\textbf{Optimization procedure.}
The optimizer preview constructs the task proposal $t_t$ without
advancing its state. Predictive feasibility projection then returns the closest constrained update
in the predictive-feasible region defined by $(H_t,b_t)$.
The optimizer state is advanced once per window, and the resulting constrained
update $d_t^\star$ is written back to the compressor parameters.

\begin{figure}[!t]
    \centering
    \includegraphics[width=\linewidth]{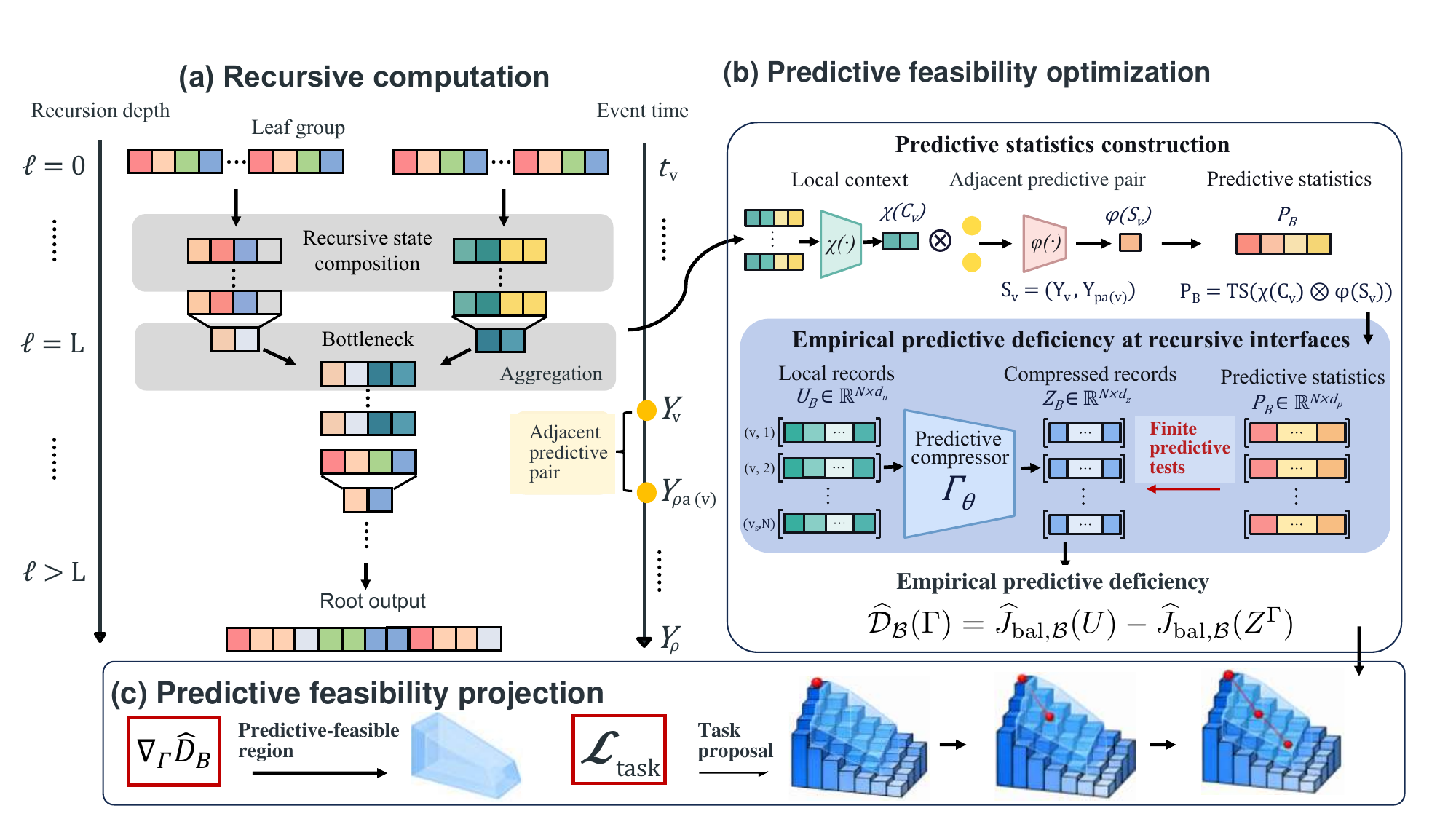}
    \caption{
    \textbf{Predictive feasibility optimization for recursive compression.}
    Predictive sensitivities constrain the task proposal before
    parameter write-back.
    }
    \label{fig:recursive-optimization}
\end{figure}

\subsection{Properties and Predictive Guarantees}

Local predictive deficiencies admit a quantitative global consequence.
We now relate accumulated local predictive deficiencies to the predictive law of the root target.

\begin{prop}[Root Predictive Stability]
\label{prop:root-predictive-stability}
Under (M1)--(M3) of Appendix~\ref{app:recursive-closure}, with $C$ the full
retained context and $Z_v^\Gamma=U_v$ at any uncompressed interface,
\[
\mathbb{E}\!\left[
\operatorname{TV}\!\left(
\mathcal{L}(Y_\rho\mid U,C),
\mathcal{L}(Y_\rho\mid Z^\Gamma,C)
\right)
\right]
\le
\left\{
\frac{1}{2}
\sum_{v\neq\rho}
I\!\left(
S_v;U_v
\mid
Z_v^\Gamma,C
\right)
\right\}^{1/2}.
\]
\end{prop}

Proposition~\ref{prop:root-predictive-stability} provides a quantitative
local-to-global guarantee. It shows that the change in the root predictive
law is controlled by the local predictive deficiencies at the individual
recursive interfaces, so small local predictive deficiencies imply a small global predictive discrepancy. When every local predictive deficiency vanishes, the compressed computation preserves the root target
exactly, recovering the conclusion of
Proposition~\ref{prop:recursive-closure}.
This guarantee concerns the population-level local predictive deficiencies and
applies when its stated local predictive deficiencies are controlled. 
The local calibration in Corollary~\ref{cor:controlled-surrogate} requires 
additional test-coverage, readout, context-localization, and independent-sample 
conditions; the projection bounds only a first-order change in the pooled 
training statistic. Appendix~\ref{app:training-dynamics} further shows
that the projected updates bound the finite-horizon change of the pooled
empirical predictive deficiency, up to the relaxation margin and residual
higher-order terms. Combined with the local calibration above, this connects
the quantity controlled during training to the population deficiencies in
Proposition~\ref{prop:root-predictive-stability}; details are given in the
appendix.
The complete proof is given in 
Appendix~\ref{app:root-predictive-stability}.

\medskip
\noindent\textbf{Local supervision beyond the root target.}
Root-target supervision identifies intermediate states through their root-target
posterior. Appendix~\ref{app:endpoint-task-quotient} shows that this quotient
may merge states with distinct recursive posteriors: root-target sufficiency
can hold while
\(I(S_v;U_v\mid Z_v^\Gamma,C)\) remains positive. Hence root-target sufficiency
alone does not imply local predictive sufficiency. Our update constrains
the grouped first-order change of the predictive surrogate. At zero margin,
each grouped constraint has a non-positive linearized contribution, while
a feasible task proposal is unchanged.
Appendix~\ref{app:predictive-feasibility-proof} proves this
guarantee, Appendix~\ref{app:tractable-local-objective} quantifies the
approximation to the local predictive criterion, and
Appendix~\ref{app:root-predictive-stability} relates the
local predictive deficiencies to root predictive stability.

\section{Experiments}

\label{sec:experiments}
\paragraph{Datasets and tasks.}
We evaluate our method across four recursive settings. 
For temporal graph learning, we use Wikipedia~\citep{Kumar2019JODIE} for node classification and UCI~\citep{Panzarasa2009} for link prediction, with propagation depths controlling recursive neighborhood expansion. 
For language memory, we use DailyDialog~\citep{Li2017DailyDialog} for response prediction across increasing conversation histories and Multi-Session Chat (MSC)~\citep{Xu2022MSC} for long-term response prediction across increasing session histories. 
For vision-language-action control, we use LIBERO-Mem~\citep{chung2026eslotssm} to evaluate memory during multi-stage manipulation under the T3 setting. For each training run, checkpoint selection uses an internal 20-demonstration development split from the released refined training data; the selected checkpoint is then frozen and evaluated on a separate set of 20 official held-out evaluation initializations. 
For SRI, we use a six-task cohort from CO-Bench~\citep{Sun2026COBench} as a mechanism-level evaluation of predictive preservation, where optimization candidates are recursively generated and refined from earlier candidates and execution feedback, and retained history conditions subsequent strategy updates. 
Detailed dataset statistics and matched evaluation protocols across all four domains are provided in Appendix~\ref{app:experimental_details}.

\paragraph{Models and baselines.}
Across all four domains, we instantiate our method on a representative recursive model: TGN~\citep{Rossi2020TGN} for temporal graphs, CCM-merge~\citep{Kim2024CCM} for language memory, MemoryVLA~\citep{Shi2026MemoryVLA} for vision-language-action control, and Meta$^{n}$~\citep{kim2026metan} for SRI. For temporal graphs, we compare against TGN~\citep{Rossi2020TGN}, TGAT~\citep{Xu2020TGAT}, TCL~\citep{wang2021tcl}, CTAN~\citep{gravina2024ctan}, and DyGFormer~\citep{Yu2023DyGFormer}. For language memory, we compare directly against the corresponding CCM-merge host under the same backbone, memory depth, and evaluation protocol. For vision-language-action control, we compare against MemoryVLA~\citep{Shi2026MemoryVLA} and report prior-work results for $\pi_0$, SlotVLA, and E-SlotSSM on the same LIBERO-Mem T3 benchmark as contextual references. For SRI, we compare against Meta$^{n}$~\citep{kim2026metan} with the same backbone, task cohort, recursive search procedure, and optimization budget. Ours changes the historical-trace retention rule used for subsequent recursive updates. Unless otherwise specified, our method modifies only the compression, memory, or recursive-context retention mechanism of the corresponding host while keeping the underlying backbone and evaluation protocol fixed.

\begin{table*}[t] 
\caption{ 
\textbf{Main results across recursive systems.} 
Results span temporal graphs, language memory, VLA control, and SRI; best results appear bold. 
} 
\label{tab:main_results} 
 
\centering 
\footnotesize 
\renewcommand{\arraystretch}{1.08} 
 
\begin{tabular*}{\textwidth}{ 
@{\extracolsep{\fill}} 
llcccccc 
@{} 
} 
\toprule 
 
\multicolumn{8}{@{}l}{ 
\textbf{(a) Temporal Graph Learning} 
\hspace{0.8em} 
\textit{Prediction across propagation depths; AP/AUC in \%} 
} 
\\ 
\midrule 
 
\textbf{Dataset / Task} 
& 
\textbf{Setting} 
& 
\textbf{TGN} 
& 
\textbf{TGAT} 
& 
\textbf{TCL} 
& 
\textbf{CTAN} 
& 
\textbf{DyGFormer} 
& 
\textbf{Ours} 
\\ 
\midrule 
 
\multirow{2}{*}{ 
\makecell[l]{Wikipedia\\Node Classification} 
} 
& 
Depth 2 AUC $\uparrow$ 
& 83.61 
& 85.45 
& 85.64 
& 87.31 
& 82.90 
& \textbf{88.81} 
\\ 
 
& 
Depth 3 AUC $\uparrow$ 
& 87.47 
& 82.08 
& 81.26 
& 87.38 
& 87.00 
& \textbf{89.17} 
\\ 
 
\midrule 
 
\multirow{4}{*}{ 
\makecell[l]{UCI\\Link Prediction} 
} 
& 
Depth 3 AP $\uparrow$ 
& 93.04 
& 86.08 
& 86.35 
& -- 
& \textbf{95.47} 
& 94.14 
\\ 
 
& 
Depth 3 AUC $\uparrow$ 
& 92.79 
& 85.44 
& 85.95 
& -- 
& \textbf{94.52} 
& 93.73 
\\ 
 
& 
Depth 4 AP $\uparrow$ 
& 94.07 
& 84.14 
& 83.78 
& -- 
& 94.92 
& 94.89 
\\ 
 
& 
Depth 4 AUC $\uparrow$ 
& 93.73 
& 83.09 
& 83.20 
& -- 
& 94.14 
& \textbf{94.47} 
\\ 
 
\bottomrule 
\end{tabular*}

\vspace{0.65em}

\begin{tabular*}{\textwidth}{ 
@{\extracolsep{\fill}} 
llcccccc 
@{} 
} 
\toprule 
 
\multicolumn{8}{@{}l}{ 
\textbf{(b) Language Memory} 
\hspace{0.8em} 
\textit{Performance under increasing memory depth} 
} 
\\ 
\midrule 
 
\multirow{2}{*}{\textbf{Dataset / Metric}} 
& 
\multirow{2}{*}{\textbf{Depth}} 
& 
\multicolumn{3}{c}{\textbf{LLaMA-7B}} 
& 
\multicolumn{3}{c}{\textbf{Gemma-4-E4B-It}} 
\\ 
 
\cmidrule(lr){3-5} 
\cmidrule(lr){6-8} 
 
& 
& 
\textbf{CCM-M} 
& 
\textbf{Ours} 
& 
\textbf{$\Delta$ (\%)} 
& 
\textbf{CCM-M} 
& 
\textbf{Ours} 
& 
\textbf{$\Delta$ (\%)} 
\\ 
\midrule 
 
\multirow{3}{*}{ 
\makecell[l]{DailyDialog\\PPL $\downarrow$} 
} 
& 
4 
& 5.98 
& \textbf{5.86} 
& \textbf{-1.98} 
& 7.00 
& \textbf{6.97} 
& \textbf{-0.54} 
\\ 
 
& 
6 
& 5.85 
& \textbf{5.71} 
& \textbf{-2.51} 
& 6.84 
& \textbf{6.78} 
& \textbf{-0.80} 
\\ 
 
& 
8 
& 5.80 
& \textbf{5.65} 
& \textbf{-2.56} 
& 6.85 
& \textbf{6.77} 
& \textbf{-1.05} 
\\ 
 
& 
11 
& 5.72 
& \textbf{5.55} 
& \textbf{-2.91} 
& 6.77 
& \textbf{6.66} 
& \textbf{-1.64} 
\\ 
 
& 
13 
& 5.63 
& \textbf{5.38} 
& \textbf{-4.53} 
& 6.73 
& \textbf{6.58} 
& \textbf{-2.24} 
\\ 
 
\midrule 
 
\multirow{3}{*}{ 
\makecell[l]{MSC\\PPL $\downarrow$} 
} 
& 
3 
& 6.82 
& \textbf{6.77} 
& \textbf{-0.84} 
& 7.96 
& 7.96
& 0.00
\\ 
 
& 
4 
& 7.02 
& \textbf{6.95} 
& \textbf{-1.09} 
& 8.23 
& \textbf{8.22} 
& \textbf{-0.12} 
\\ 
 
& 
5 
& 7.03 
& \textbf{6.93} 
& \textbf{-1.39} 
& 8.16 
& \textbf{8.14} 
& \textbf{-0.25} 
\\ 
 
\bottomrule 
\end{tabular*}

\vspace{0.65em}

\begin{tabular*}{\textwidth}{ 
@{\extracolsep{\fill}} 
llcccccc 
@{} 
} 
\toprule 
 
\multicolumn{8}{@{}l}{ 
\textbf{(c) Vision-Language-Action} 
\hspace{0.8em} 
\textit{Persistent control under T3; SR/SGC in \%} 
} 
\\ 
\midrule 
 
\textbf{Benchmark} 
& 
\textbf{Setting} 
& 
$\boldsymbol{\pi_0}^{*}$ 
& 
\makecell{\textbf{SlotVLA}$^{*}$\\$h{=}1$} 
& 
\makecell{\textbf{SlotVLA}$^{*}$\\$h{=}8$} 
& 
\textbf{E-SlotSSM$^{*\dagger}$} 
& 
\textbf{MemoryVLA} 
& 
\textbf{Ours} 
\\ 
\midrule 
 
\multirow{2}{*}{LIBERO-Mem} 
& 
T3 SR $\uparrow$ 
& 
0.00 
& 
0.00 
& 
0.00 
& 
-- 
& 
1.67 
& 
\textbf{11.67} 
\\ 
 
& 
T3 SGC $\uparrow$ 
& 
0.00 
& 
0.00 
& 
0.00 
& 
\textbf{33.30} 
& 
20.00 
& 27.22 
\\ 
 
\bottomrule 
\end{tabular*}

\vspace{0.65em}

\begin{tabular*}{\textwidth}{ 
@{\extracolsep{\fill}} 
llccccc 
@{} 
} 
\toprule 
 
\multicolumn{7}{@{}l}{ 
\textbf{(d) Self-Recursive Improvement} 
\hspace{0.8em} 
\textit{Predictive preservation in recursive strategy refinement} 
} 
\\ 
\midrule 
 
\textbf{Benchmark} 
& 
\textbf{Backbone} 
& 
\multicolumn{2}{c}{\textbf{Dev Gain} $\uparrow$} 
& 
\multicolumn{3}{c}{\textbf{Held-out Score} $\uparrow$} 
\\ 
 
\cmidrule(lr){3-4} 
\cmidrule(lr){5-7} 
 
& 
& 
\textbf{Meta$^n$} 
& 
\textbf{Ours} 
& 
\textbf{Task} 
& 
\textbf{Meta$^n$} 
& 
\textbf{Ours} 
\\ 
\midrule 
 
\multirow{4}{*}{CO-Bench} 
 
& 
gpt-5.5 
& 
81.72 
& 
\textbf{82.79} 
& 
Assortment 
& 
45.98 
& 
\textbf{49.37} 
\\ 
 
\cmidrule(lr){2-7} 
 
& 
gemini-3.1-pro 
& 
90.01 
& 
\textbf{95.47} 
& 
Aircraft 
& 
83.89 
& 
\textbf{92.34} 
\\ 
 
\cmidrule(lr){2-7} 
 
& 
DeepSeek-V4.1-Flash 
& 
80.70 
& 
\textbf{86.00} 
& 
Assignment 
& 
82.92 
& 
\textbf{100.00} 
\\ 
 
& 
& 
& 
& 
Capacitated 
& 
68.73 
& 
\textbf{70.91} 
\\ 
 
\bottomrule 
\end{tabular*}

\vspace{2pt} 
 
\begin{minipage}{\textwidth} 
\footnotesize 
\textit{Note.} 
For temporal graphs, depth denotes propagation layers.  
DailyDialog reports dialogue-equal PPL under the matched protocol.  
For MSC, depth denotes preceding session history.  
$^{*}$ Results are from prior work.  
$^{\dagger}$ E-SlotSSM uses oracle object-level subgoal representations.  
For SRI, Meta$^n$ and Ours share the backbone, task cohort, recursive search procedure, and optimization budget; they differ only in the historical traces retained for recursive updates.  
We report Dev Gain from one paired run and three-evaluation mean held-out scores for tasks modified during development.  
The modified-task set is fixed before held-out evaluation.  
Complete six-task results and repeat-level variability are provided in Appendix~\ref{app:sri-results}. 
\end{minipage} 
 
\end{table*}
\FloatBarrier

\paragraph{Metrics.}
We use task-specific metrics across the four domains.
For temporal graphs, we report average precision (AP) and ROC-AUC.
For language memory, we report dialogue-equal perplexity (PPL) on DailyDialog and session-opening PPL on MSC.
For vision-language-action control, we report task success rate (SR) and sequential-goal completion (SGC) under T3.
For SRI, we report Dev Gain from one paired recursive-search run and, for each task modified during development-side recursive refinement, the mean held-out score over three matched evaluations of each frozen strategy.
The complete six-task task-wise results and cohort-level average from one fixed held-out evaluation are reported in Appendix~\ref{app:sri-results}.
Together, these metrics measure task-level predictive or control quality under each recursive setting and across different forms of recursion.

\paragraph{Main Findings.}
As summarized in Table~\ref{tab:main_results}, our method consistently improves
the corresponding host across four substantially different recursive systems
by modifying only how intermediate information is retained. A particularly
informative pattern emerges as recursive dependence becomes more demanding.
In language memory, the improvement over the compressed host is largest at
longer histories, where information must survive repeated compression. In
temporal graphs, although our TGN-based method trails the strong DyGFormer
baseline at shallower propagation on UCI, this gap closes as propagation
deepens, reaching comparable AP and higher AUC at the deepest setting.
The VLA and SRI results extend this observation beyond explicit depth axes:
the same preservation mechanism remains effective in long-horizon closed-loop
execution and recursive strategy refinement, where retained information is
reused in later computation. Complete depth- and history-resolved results are
provided in Appendices~\ref{app:tg_results} and~\ref{app:language_results},
with additional long-horizon VLA and recursive-strategy analyses in
Appendices~\ref{app:vla-details} and~\ref{app:sri-results}.

\textit{(1) Are the end-task gains consistent with the predictive-preservation mechanism?}
End-task improvements alone do not identify why the method works. We therefore
use Fig.~\ref{fig:mechanism-analysis} as a mechanism-consistency analysis.
In temporal graphs, Fig.~\ref{fig:mechanism-analysis}(a) shows that our method
retains stronger finite-test predictive value \(J\) across recursive interfaces
than TGN and task-only training, including predictive information originating
from deeper sources. In SRI, Fig.~\ref{fig:mechanism-analysis}(b) shows that
information retained from earlier executions remains more useful when reused
during subsequent strategy refinement. The same association also appears
during language-memory training: Appendix~\ref{app:language_results},
Fig.~\ref{fig:ccm_training_diagnostics}, shows that the predictive-preservation
signal \(\Delta J\) and validation PPL improvement evolve broadly in the same
direction over the training trajectory. These diagnostics do not directly estimate the population deficiency \(\delta_v\); together, they provide mechanism-level evidence linking task gains to stronger predictive preservation and downstream reuse.

\begin{figure}[t]
\centering
\includegraphics[width=\linewidth]{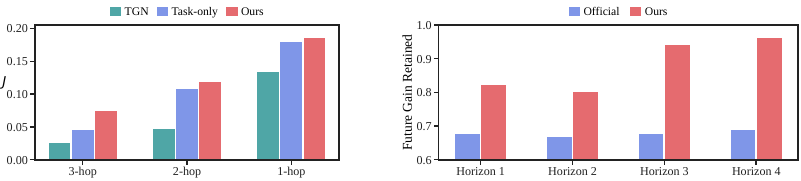}

\vspace{-0.5em}
\begin{minipage}[t]{0.48\linewidth}
    \centering
    \small (a) Predictive information preservation
\end{minipage}
\hfill
\begin{minipage}[t]{0.48\linewidth}
    \centering
    \small (b) Future utility of retained information
\end{minipage}

\vspace{-0.3em}
\caption{
\textbf{Mechanism analysis.}
(a) The finite-test predictive value \(J\) tracks predictive distinctions
retained across recursive interfaces.
(b) Future gain retained measures whether retained information remains useful
when reused in later recursive updates.
Together, the two diagnostics test whether end-task improvements are accompanied
by stronger predictive preservation and downstream reuse.
}
\label{fig:mechanism-analysis}
\end{figure}

\textit{(2) Does the optimization actually enforce the intended preservation constraint?}
The mechanism diagnostics above characterize the learned representations, but
do not by themselves show whether the constraint is active during optimization.
Appendix~\ref{app:tg_results}, Fig.~\ref{fig:tgn-constraint-conflict}, directly
examines this behavior: unconstrained task proposals continue to violate a
nonzero fraction of the grouped predictive constraints throughout training,
although the violation density decreases as optimization progresses. Thus, the
projection is not vacuous; task optimization continues to propose directions
that can interfere with predictive preservation. Moreover,
Table~\ref{tab:first_order_bridge} shows that the constrained first-order
quantity closely tracks the actual change in the corresponding window-level
predictive statistic, validating the local bridge used by the projection.
Table~\ref{tab:ablation} further separates this constrained update from
task-only training, an additive predictive-gradient update, and a single global
predictive constraint, while the complete controlled ablations are reported in
Appendix~\ref{app:ablation}. Their degradation relative to the full method
supports the role of interface-wise constrained optimization beyond merely
adding predictive supervision.

\begin{table}[t]
\centering
\caption{
\textbf{Single-seed diagnostic ablation on UCI at depth 3.}
We isolate the roles of the predictive target, the adjacent predictive pair,
and the way predictive supervision constrains updates.
}
\label{tab:ablation}

\vspace{-0.4em}
\small
\setlength{\tabcolsep}{7.5pt}
\renewcommand{\arraystretch}{1.08}

\begin{tabular}{lcc}
\toprule
\textbf{Variant}
&
\textbf{AP} $\uparrow$
&
\textbf{AUC} $\uparrow$
\\
\midrule

Task-Only Training
& 92.62
& 92.13
\\

Direct Future Prediction
& 92.76
& 92.24
\\

Single-Event Predictive Criterion
& 92.41
& 91.93
\\

Additive Predictive-Gradient Update
& 92.64
& 92.18
\\

Global Predictive Constraint
& 90.32
& 89.88
\\

\midrule

\textbf{Our Method}
& \textbf{94.52}
& \textbf{94.25}
\\

\bottomrule
\end{tabular}
\end{table}

\noindent\textbf{Ablation studies.}
Table~\ref{tab:ablation} isolates the main design choices under the same UCI depth-3 setting. Task-Only Training controls for the shared host, compressor, and trainable parameterization, testing whether additional learnability alone explains the improvement. Direct Future Prediction asks whether a conventional future-prediction objective can replace the proposed predictive criterion, while the Single-Event Predictive Criterion directly tests the adjacent predictive pair $S_v$ underlying our theoretical formulation. Additive Predictive-Gradient Update uses the same predictive supervision as an auxiliary gradient, separating the contribution of predictive supervision from its constrained parameter-space realization. Global Predictive Constraint further tests whether an aggregated constraint can replace the interface-wise structure used by our method. Across these controls, the observed performance degradations are substantial, supporting the respective roles of the predictive criterion, adjacent predictive pair, and constrained update. Additional ablations on statistical realization and supervision alignment are reported in Appendix~\ref{app:ablation}.

\noindent\textbf{Robustness and efficiency.}
We further examine the two choices that most directly control the constraint. The feasibility tolerance $\kappa$ exhibits a clear interior optimum: overly small values restrict the task proposal too strongly, whereas larger values weaken predictive preservation. Increasing the eligible-interface fraction improves test AP and then approaches saturation, indicating that the pooled statistic benefits from broader interface coverage once sufficient support is available. We therefore select $\kappa$ by a small validation sweep and retain broad eligible-interface coverage. The method acts only during training; its additional computation consists of pooled minibatch statistics, predictive sensitivities, and the predictive feasibility projection, while test-time computation remains that of the host. The complete sensitivity curves are reported in Appendix~\ref{app:robustness-efficiency}.

\section{Conclusion}
\label{sec:conclusion}

We study how to learn compact states for recursive computation without losing
information required by prediction. Our approach builds on local predictive deficiency and recursive predictive
closure, and uses predictive feasibility projection to constrain the task proposal. Across temporal graphs, language memory,
vision-language-action control, and recursive self-improvement, the method
improves host models under matched compression budgets. Mechanistic analyses
show preservation of predictive information from deeper sources, while
ablations confirm the roles of alignment, conditioning, and constrained
optimization. These results support predictive preservation at recursive
interfaces as a principle for compact states across domains.

\noindent\textbf{Limitations.}
Our formulation relies on statistical support for estimating local predictive constraints, which may become less stable in data-sparse regimes. The feasibility tolerance $\kappa$ also requires validation-based calibration and may vary across tasks and domains. Developing more data-efficient estimators and adaptive calibration strategies, including transfer across related tasks, remains a direction for future work. The VLA evidence is restricted to the held-out LIBERO-Mem T3 evaluation initializations and should therefore be interpreted as held-out closed-loop performance under this benchmark setting rather than as broad out-of-distribution generalization.

\bibliography{references}
\bibliographystyle{plainnat}

\appendix
\clearpage
\phantomsection
\pdfbookmark[1]{Appendix}{appendix-start}
{\LARGE\bfseries Appendix\par}
\vspace{1.35em}
{\Large\bfseries Table of Contents\par}
\vspace{0.25em}
\hrule
\vspace{0.75em}

\newcommand{\apptocmain}[2]{%
  \noindent
  \hyperref[#1]{\makebox[1.0cm][l]{\textbf{\ref*{#1}}}\textbf{#2}}%
  \nobreak\leaders\hbox{\normalfont\kern.12em.\kern.12em}\hfill\nobreak
  \hyperref[#1]{\pageref*{#1}}\par\vspace{0.32em}%
}
\newcommand{\apptocsub}[2]{%
  \noindent\hspace*{1.3em}%
  \hyperref[#1]{\makebox[1.15cm][l]{\ref*{#1}}#2}%
  \nobreak\leaders\hbox{\normalfont\kern.12em.\kern.12em}\hfill\nobreak
  \hyperref[#1]{\pageref*{#1}}\par\vspace{0.10em}%
}

\small
\apptocmain{app:ai-use}{AI Use Statement}
\apptocmain{app:related-work}{Related Work}

\apptocmain{app:experimental_details}{Experimental Protocols and Evaluation Details}
\apptocsub{app:temporal_graphs}{Temporal Graphs}
\apptocsub{app:language_memory}{Language Memory}
\apptocsub{app:vla}{Vision--Language--Action}
\apptocsub{app:sri-setting}{Self-Recursive Improvement}

\apptocmain{app:additional_results}{Extended Results and Diagnostic Analyses}
\apptocsub{app:tg_results}{Temporal Graph Results}
\apptocsub{app:language_results}{Language Memory Results}
\apptocsub{app:vla-details}{Vision--Language--Action Results}
\apptocsub{app:sri-results}{Recursive Self-Improvement Results}

\apptocmain{app:implementation}{Implementation, Ablations, and Reproducibility}
\apptocsub{app:impl_tg}{Temporal-Graph Implementation}
\apptocsub{app:impl_language}{Language-Memory Implementation}
\apptocsub{app:impl_vla}{Vision--Language--Action Implementation}
\apptocsub{app:impl_sri}{Self-Recursive Improvement Implementation}
\apptocsub{app:training-compute}{Training and Compute Configuration}
\apptocsub{app:efficiency}{Efficiency and Model Overhead}
\apptocsub{app:ablation}{Ablation Study}
\apptocsub{app:hyperparameters}{Method Hyperparameters and Sensitivity}

\apptocmain{app:theory}{Theoretical Results and Proofs}
\apptocsub{app:recursive-closure}{Recursive Predictive Closure}
\apptocsub{app:local-variational-characterization}{Local Variational Characterization}
\apptocsub{app:tractable-local-objective}{From Local Predictive Deficiency to Constrained Updates}
\apptocsub{app:training-dynamics}{Finite-Horizon Predictive Control}
\apptocsub{app:root-predictive-stability}{Root Predictive Stability}
\apptocsub{app:endpoint-task-quotient}{Why Root-Target Supervision Is Insufficient}
\normalsize

\vspace{0.6em}
\hrule
\vspace{0.8em}
\noindent
Appendices~\ref{app:experimental_details}--\ref{app:implementation} contain the full experimental protocol, extended analyses, and implementation details. Appendix~\ref{app:theory} collects all formal results and proofs in one place.

\clearpage

\section{AI Use Statement}
\label{app:ai-use}

Generative AI tools were used to assist with several supporting aspects of this work. They were used to polish and improve the clarity of the manuscript, while the technical content and final wording were reviewed and controlled by the authors. They were also used for literature and repository retrieval and discovery; all cited papers, repositories, and links were subsequently verified by the authors against the original sources.

Generative AI was additionally used to draft portions of standard descriptive text, such as preliminary summaries of experimental settings or results, which were subsequently revised by the authors based on the underlying experiments and technical details. For mathematical claims, the claims and main proof ideas were developed by the authors, while generative AI was used to assist with intermediate derivations, formalization, and proof checking during iterative development. All final mathematical statements and proofs were reviewed and verified by the authors.

For software development, the authors designed the code structure, implementation logic, and experimental pipeline, while generative AI was used to assist with code implementation. All AI-assisted code was subsequently reviewed, checked, and validated by the authors.

The authors take full responsibility for the final content of the paper.

\section{Related Work}
\label{app:related-work}

Recursive state compression appears across a broad range of modern learning systems. Temporal graph models repeatedly aggregate historical interactions into evolving node representations, as in TGN and TGAT \citep{Rossi2020TGN,Xu2020TGAT}; recurrent language-memory methods compress accumulated context into states that are reused across subsequent segments or interactions, as in CCM and AutoCompressor \citep{Kim2024CCM,Chevalier2023AutoCompressor}; embodied models such as MemoryVLA maintain compact perceptual--cognitive memories across multi-stage control \citep{Shi2026MemoryVLA}; and iterative self-improvement or program-search systems reuse selected candidates, execution feedback, or reflections across successive rounds, as in Reflexion, FunSearch, and ReEvo \citep{Shinn2023Reflexion,RomeraParedes2024FunSearch,Ye2024ReEvo}. Despite their architectural differences, these systems share a common structure: intermediate states are repeatedly aggregated, compressed, and reused. This makes the central question not only how to compress a state, but what information must survive compression so that subsequent computation remains predictive.

Existing representation-learning methods impose different criteria for deciding whether a compressed state is adequate. A classical approach uses reconstruction. Recursive autoencoders, for example, compose child representations into a parent state and train the parent representation to reconstruct its children \citep{Socher2011}. More generally, autoencoding and masked-reconstruction objectives encourage representations to retain enough information to recover designated inputs or local content. A second family instead uses local or future prediction. Contrastive Predictive Coding trains a context representation to predict future latent representations \citep{Oord2018CPC}, while joint-embedding predictive approaches such as I-JEPA predict representations of masked target regions from visible context without reconstructing raw inputs \citep{Assran2023IJEPA}. Such objectives place greater emphasis on predictive structure than direct reconstruction. However, both reconstruction and future-prediction objectives require a target to be specified in advance. In a recursively composed system, accurately reconstructing an immediate input, predicting an immediate parent, or predicting a particular future latent does not by itself imply that the target preserves all distinctions needed after subsequent aggregation and compression.

Information-theoretic and predictive-state formulations provide a more general perspective on what a representation should preserve. The Information Bottleneck (IB) formalizes compression through retained information about a designated relevance variable \citep{Tishby1999}; distributed and multivariate variants extend this principle to multiple observations and interacting variables \citep{EstellaAguerri2018DIB,Abdelaleem2025DVMIB}, while successive-refinement formulations study information preservation across multiple stages of compression \citep{Charvin2023Successive}. Predictive State Representations (PSRs) instead characterize state through conditional predictions of future observables rather than reconstruction of the complete history \citep{Littman2001,Singh2004PSR}. These perspectives establish task relevance and future predictability as principled criteria for representation quality. Our setting introduces an additional compositional requirement: the predictive distinctions retained at one interface must remain sufficient for the predictive requirements encountered at the next interface. We therefore do not introduce another reconstruction or fixed future-prediction target; we formulate a local predictive-sufficiency condition whose role is explicitly tied to recursive composition.

Directly enforcing such a condition through full conditional distributions or mutual information is generally impractical in high-dimensional learned systems. Mutual-information estimation from finite samples is statistically difficult, with variational estimators exhibiting substantial bias--variance trade-offs \citep{Poole2019Variational}. Predictive-state learning faces a related approximation problem because practical systems must restrict attention to a finite collection of predictive tests \citep{Kulesza2015Insufficient}. Spectral PSRs, conditional-distribution embeddings \citep{Song2009Hilbert}, and conditional mean or operator-regression methods \citep{Grunewalder2012ConditionalMean} therefore characterize predictive relationships through finite tests, regressions, or cross-covariance operators. Our empirical predictive deficiency follows this computational philosophy, but the direction of construction is important: the recursive predictive requirement is specified first, and finite tests, ridge readouts, and covariance balancing are introduced only as a tractable approximation to that requirement. Thus, empirical predictive deficiency approximates but does not define this preservation criterion.

A separate line of work concerns how preservation or compatibility constraints are enforced during optimization. GEM constrains a current gradient using losses on episodic memories from previous tasks and projects violating gradients through a quadratic program \citep{LopezPaz2017GEM}; A-GEM replaces the collection of per-task constraints with a reference-gradient constraint \citep{Chaudhry2019AGEM}; and PCGrad modifies multi-task gradients according to pairwise gradient conflicts \citep{Yu2020PCGrad}. More generally, projected and proximal optimization methods allow an optimizer step to be followed by a correction with respect to a prescribed feasible domain, including adaptive methods for constrained stochastic optimization \citep{Duchi2011AdaGrad,Reddi2018AMSGrad}. We therefore do not regard gradient projection, quadratic programming, or post-optimizer projection itself as a contribution.

Our distinction lies instead in what defines the constrained update and where that constraint is imposed. GEM-type methods derive admissible directions from task losses or gradient compatibility, while conventional projected optimization assumes a feasible parameter domain specified independently of the current predictive state. In our setting, the predictive-feasible region is generated dynamically from predictive sensitivities of empirical predictive deficiency at recursive interfaces. Moreover, because the first-order predictive constraint is defined over an actual parameter displacement, we apply it to the task proposal produced by the host optimizer rather than directly projecting the raw task gradient before Adam or AdamW is applied. Given the task proposal \(t_t\), predictive feasibility projection returns the closest constrained update \(d_t^\star\) in the predictive-feasible region \(\mathcal F_t\). The optimization contribution is therefore not a new projection operator, but the conversion of local predictive sufficiency into data-dependent margin-relaxed predictive constraints on the task proposal.

Taken together, prior work provides powerful mechanisms for reconstruction, future prediction, information-theoretic compression, predictive-state approximation, and constrained optimization. The remaining question addressed here is different: what local condition should an intermediate state satisfy when that state will itself be aggregated, compressed, and reused in later computation? We formulate this condition through local predictive sufficiency, approximate its local predictive deficiency with empirical predictive deficiency, and use predictive sensitivities to restrict updates that would erode the corresponding predictive distinctions.

\clearpage
\section{Experimental Protocols and Evaluation Details}
\label{app:experimental_details}

\subsection*{Reproduction Guide}

\paragraph{Organization.}
All four experimental domains follow the same high-level reproduction path:
data preparation, matched-host initialization, training of Ours,
validation- or development-only model selection, and frozen final evaluation.
Appendix~\ref{app:experimental_details} specifies the dataset construction,
split rules, selection criteria, metrics, and evaluation cohorts;
Appendix~\ref{app:additional_results} reports the complete results and
diagnostic analyses; and Appendix~\ref{app:implementation} documents the
host-specific integration, optimization settings, compute environment, and
method-specific frozen hyperparameters. Exact executable commands and frozen
configuration files are provided with the supplementary code release.

\begin{table}[t]
\centering
\caption{
\textbf{Reproduction map for the four experimental domains.}
Each row traces the reported results from data preparation through model
selection and frozen evaluation. Each row links its protocol, results, and implementation.
}
\label{tab:reproduction_map}
\small
\setlength{\tabcolsep}{3.5pt}
\renewcommand{\arraystretch}{1.10}

\begin{tabular}{p{0.105\linewidth}
                p{0.205\linewidth}
                p{0.205\linewidth}
                p{0.205\linewidth}
                p{0.205\linewidth}}
\toprule
\textbf{Domain}
&
\textbf{Data / split}
&
\textbf{Training / initialization}
&
\textbf{Selection}
&
\textbf{Frozen evaluation}
\\
\midrule

Temporal
graphs
&
Wikipedia and UCI;
chronological
70/15/15 split
&
TGN host with matched
task, depth, temporal
neighborhood, and seed;
five independent runs
&
Wikipedia:
validation ROC-AUC;
UCI:
full-validation AP
&
Selected checkpoint on
held-out test;
causal memory replay for
Wikipedia;
global AUC/AP.

\textit{Refs.:}
\ref{app:temporal_graphs},
\ref{app:tg_results},
\ref{app:impl_tg}
\\

\addlinespace[2pt]

Language
memory
&
DailyDialog original
splits with a fixed-target
depth cohort;
MSC chronological
five-session conversations
&
One fixed CCM-merge host
per backbone and dataset;
three independent Ours
runs initialized from the
same host
&
Pre-specified validation
criterion;
one checkpoint selected
per Ours run before
depth-wise evaluation
&
DailyDialog:
fixed pooled depth cohort;
MSC:
held-out test conversations;
dialogue-equal PPL.

\textit{Refs.:}
\ref{app:language_memory},
\ref{app:language_results},
\ref{app:impl_language}
\\

\addlinespace[2pt]

VLA
&
100 released refined T3
demonstrations:
80 training + 20 internal
development
&
MemoryVLA host trained
from the common OpenVLA
initialization;
Ours starts from the
matched trained host and
optimizes only $\Gamma$;
three independent runs
&
20-demonstration internal
development split only;
25k-step upper budget;
one checkpoint frozen
per run
&
20 separate official
held-out T3 evaluation
initializations;
matched MemoryVLA/Ours
rollouts;
SR and SGC.

\textit{Refs.:}
\ref{app:vla},
\ref{app:vla-details},
\ref{app:impl_vla}
\\

\addlinespace[2pt]

SRI
&
Fixed six-task CO-Bench
cohort with official
development and held-out
instances
&
Meta$^n$ and Ours begin
from the same realized
Layer-1 state and receive
the same recursive-search
opportunity and
retained-history capacity
&
Recursive refinement and
strategy selection use
development feedback only;
selected strategies are
then frozen
&
Official held-out instances;
matched evaluation of
frozen strategies;
three evaluation repeats
for the reported task-level
scores.

\textit{Refs.:}
\ref{app:sri-setting},
\ref{app:sri-results},
\ref{app:impl_sri}
\\

\bottomrule
\end{tabular}
\end{table}

\paragraph{Execution order.}
To reproduce a reported result, we first construct the corresponding dataset
and evaluation cohort using the domain-specific protocol in
Appendix~\ref{app:experimental_details}. We then initialize or train the
matched host model and train Ours from the prescribed host state using the
corresponding implementation in Appendix~\ref{app:implementation}.
Checkpoint or strategy selection uses only the validation or development data
specified for that domain. The selected state is subsequently frozen and
passed to the corresponding evaluation protocol without further parameter
updates or held-out-based selection. Metrics are aggregated according to the
rules specified in Appendix~\ref{app:experimental_details}, and the complete
run-level results are reported in Appendix~\ref{app:additional_results}.
Training budgets, repetitions, software, and compute configurations are
summarized in Appendix~\ref{app:training-compute}, while the method-specific
frozen hyperparameters are reported in Appendix~\ref{app:hyperparameters}.

\subsection{Temporal Graphs}
\label{app:temporal_graphs}

\paragraph{Datasets and provenance.}
We evaluate on Wikipedia \citep{Kumar2019JODIE} and UCI
\citep{Panzarasa2009}, two established continuous-time interaction
networks. Wikipedia contains timestamped interactions between users and
Wikipedia pages together with dynamic user-state labels, and is used for
node classification. UCI contains timestamped message exchanges between
users in an online community and is used for temporal link prediction.

We use the processed benchmark files released with DGB
\citep{Poursafaei2022}.%
\footnote{Dataset archive:
\url{https://zenodo.org/records/7213796};
DGB code:
\url{https://github.com/fpour/DGB}.}
Within each task, the same processed data are used for all compared
methods. Table~\ref{tab:tg_dataset_stats} summarizes the dataset statistics
and evaluation settings.

\begin{table}[t]
\centering
\caption{\textbf{Temporal-graph datasets and evaluation settings.}}
\label{tab:tg_dataset_stats}
\small
\setlength{\tabcolsep}{5.5pt}
\renewcommand{\arraystretch}{1.05}
\begin{tabular}{lrrlc}
\toprule
\textbf{Dataset} &
\textbf{Nodes} &
\textbf{Interactions} &
\textbf{Task} &
\textbf{Depth} \\
\midrule
Wikipedia & 9,227 & 157,474 & Node classification & 2, 3 \\
UCI       & 1,899 & 59,835  & Link prediction     & 3, 4 \\
\bottomrule
\end{tabular}
\end{table}

\paragraph{Task and host selection.}
The two datasets provide complementary prediction targets on
continuous-time interaction streams: Wikipedia evaluates node-level state
prediction, whereas UCI evaluates edge-level interaction prediction.
Together, they allow us to examine predictive preservation across both prediction targets and multiple temporal depths.

We use TGN \citep{Rossi2020TGN} as the temporal-graph host and as the
corresponding uncompressed baseline. TGN combines persistent node memory
with temporal-neighborhood aggregation and exposes intermediate
representations across successive propagation stages, providing natural
recursive interfaces at which our method can be applied. Comparing TGN
with our method keeps the host architecture, temporal encoding, and task
definition fixed, thereby isolating the effect of predictive preservation
at intermediate representations.

\paragraph{Baselines and common protocol.}
We compare against TGN \citep{Rossi2020TGN},
TGAT \citep{Xu2020TGAT},
TCL \citep{wang2021tcl},
CTAN \citep{gravina2024ctan},
and DyGFormer \citep{Yu2023DyGFormer}.
For reproducibility, baseline implementations follow their publicly
released reference codebases. We use the unified DyGLib implementations
\citep{Yu2023DyGFormer} for TGN, TGAT, TCL, and DyGFormer,\footnote{DyGLib:
\url{https://github.com/yule-BUAA/DyGLib}.}
and the official CTAN implementation for CTAN.\footnote{CTAN:
\url{https://github.com/gravins/non-dissipative-propagation-CTDGs}.}
Where applicable, we use the accompanying Temporal Graph Benchmark (TGB)
utilities for temporal-graph data handling and evaluation.\footnote{TGB:
\url{https://github.com/shenyangHuang/TGB}.}
Within each task, all methods use the same processed dataset,
chronological split, evaluation examples, and metric implementation.
Each baseline retains its native representation, propagation operator, and architecture.

Interactions are ordered by timestamp and split at the 70th and 85th
timestamp percentiles, yielding chronological 70\%/15\%/15\%
training, validation, and test partitions. No interaction from a later
split is used for parameter optimization in an earlier split.

Depth denotes the number of successive temporal representation or
propagation stages used by each model. We vary depth to examine how
predictive performance changes as historical information passes through
more successive representation transformations. The operation performed
at each stage remains architecture-specific; equal depth therefore denotes
a matched number of successive representation stages rather than identical
internal computation or an identical node-level receptive field across
different architectures.

For the TGN host, temporal neighborhoods are constructed causally using
only interactions preceding the query time. Wikipedia uses the 10 most
recent available interactions at depth 2 and 5 at depth 3, whereas UCI
uses 10 at depth 3 and 5 at depth 4. The smaller fanout in the deeper
settings limits the growth of recursively expanded temporal neighborhoods
as representation depth increases. Our method preserves the host temporal
encoding and temporal-neighbor selection rule and introduces no additional
source of future input information.

\paragraph{Wikipedia node classification.}
For Wikipedia, we follow the standard TGN/JODIE dynamic
node-classification setting. A temporal representation model is first
trained on the interaction stream, after which a node-classification head
is optimized using binary cross-entropy on the observed dynamic labels.
Training, validation, and test examples follow disjoint chronological
splits.

Wikipedia is evaluated at depths 2 and 3. Checkpoint selection is based
on ROC-AUC on the held-out validation split. Because TGN maintains
persistent node memory, the temporal state required at the beginning of
test evaluation must reflect the interaction history preceding the test
period. After loading the selected checkpoint, we therefore clear the
memory and replay the training and validation interactions in chronological
order without parameter updates before evaluating on the test stream.
This procedure reconstructs only the causal memory state; parameters remain fixed.

We report ROC-AUC on the complete held-out test split. Predictions and
labels from all evaluation mini-batches are first concatenated, and
ROC-AUC is then computed once over the full split. We do not average
independently computed mini-batch ROC-AUC values.

\paragraph{UCI temporal link prediction.}
For UCI, we evaluate transductive temporal link prediction at depths 3
and 4. The main experiment uses the chronological transductive setting;
new-node subsets are not included in the reported results.

Training negatives are sampled by drawing destination nodes from the
unique destination-node pool observed in the training stream. Validation
and test evaluation use fixed random negative samplers over the full-stream
node pools, with sampler seeds 0 and 2, respectively. Each positive
interaction is paired with one negative interaction during evaluation.
The same positive examples and negative-sampling protocol are used for all
compared methods.

Model selection uses only the held-out validation split. The best
checkpoint is selected by full-validation average precision (AP), with
ROC-AUC recorded as an additional validation metric. The selected
checkpoint is then evaluated on the held-out test split.

We report AP and ROC-AUC for UCI. Predictions and labels from all
evaluation mini-batches are concatenated before either metric is computed,
so the reported values are global split-level metrics rather than averages
of mini-batch AP or ROC-AUC values.

All formal temporal-graph comparisons are repeated over five independent
random seeds. Appendix~\ref{app:tg_results} reports mean $\pm$ standard
deviation across these runs. Evaluation examples and the validation/test
negative-sampling protocol are matched across methods within each run.

\begin{table}[t]
\centering
\caption{\textbf{Temporal-graph evaluation protocol.}}
\label{tab:tg_protocol}
\small
\setlength{\tabcolsep}{5.0pt}
\renewcommand{\arraystretch}{1.05}
\begin{tabular}{lll}
\toprule
\textbf{Setting} & \textbf{Wikipedia} & \textbf{UCI} \\
\midrule
Task
& Node classification
& Temporal link prediction \\
Evaluation setting
& Chronological
& Transductive \\
Split
& 70/15/15
& 70/15/15 \\
Depth
& 2, 3
& 3, 4 \\
TGN temporal neighbors
& 10 (depth 2); 5 (depth 3)
& 10 (depth 3); 5 (depth 4) \\
Evaluation negatives
& --
& 1 per positive \\
Checkpoint criterion
& Validation ROC-AUC
& Full-validation AP \\
Reported metrics
& ROC-AUC
& AP / ROC-AUC \\
Metric aggregation
& Global
& Global \\
\bottomrule
\end{tabular}
\end{table}
\FloatBarrier

\subsection{Language Memory}
\label{app:language_memory}

\paragraph{Datasets and tasks.}
We evaluate language-memory compression on DailyDialog
\citep{Li2017DailyDialog} and Multi-Session Chat (MSC)
\citep{Xu2022MSC}.
DailyDialog is a multi-turn dialogue corpus and is used for
next-utterance language modeling. MSC is a multi-session dialogue
corpus and is used for long-term response prediction across sessions.

For both datasets, we preserve their original dialogue structure and
chronological ordering. DailyDialog follows the data representation and
preprocessing of the released CCM implementation.%
\footnote{Official CCM implementation:
\url{https://github.com/snu-mllab/Context-Memory}.}
DailyDialog is obtained from the original release linked by the dataset
paper, and MSC is obtained through the public ParlAI task release.%
\footnote{DailyDialog: \url{https://aclanthology.org/I17-1099/};
MSC: \url{https://parl.ai/projects/msc/}.}
For MSC, conversation sessions are processed in chronological order so
that information accumulated from earlier sessions can be retained for
prediction in later sessions.

The two tasks expose complementary prediction interfaces over accumulated
context. DailyDialog varies the amount of preceding dialogue history while
holding the current context and prediction target fixed. MSC evaluates
response prediction across successive session depths as the amount of
available cross-session history increases.

\paragraph{Backbones and memory host.}
We instantiate the language-memory experiments with two language-model
backbones: LLaMA-7B and \texttt{google/gemma-4-E4B-it}.
CCM \citep{Kim2024CCM} serves as the common recursive memory-compression host
across both backbones.

For a dialogue with recursively compressed history
$(u_1,\ldots,u_L)$, CCM updates its memory through the sequence of
interfaces
\[
(M_{t-1},u_t)\longrightarrow M_t,
\qquad t=1,\ldots,L,
\]
and we identify the post-compression state with
\[
Z_t^\Gamma=M_t.
\]
We use $Z_t^\Gamma$ to denote the post-update memory state at this recursive interface; the host-specific realization of $\Gamma$ is described in Appendix~\ref{app:impl_language}.
Across both backbones, the base language-model parameters remain frozen
during RPBE adaptation, while the native compression parameters consist
of the conditional LoRA parameters in the attention projections together
with the trainable COMP/SUM embedding rows.

All matched comparisons within a backbone use the same raw task examples,
native tokenizer and serialization rule, recursive CCM memory layout, and
evaluation construction. Full-context and no-history variants, when
reported, use the same query and target and differ only in the treatment
of the preceding history. Detailed trainable-parameter, predictive-statistic, and optimization
configurations are reported in Appendix~\ref{app:implementation}.

\paragraph{DailyDialog protocol.}
We preserve the original DailyDialog training, validation, and test
partitions and the text normalization used by CCM. Following the CCM
preprocessing, dialogues shorter than three utterances are removed,
individual utterances are limited to 128 tokens, with total token limits of 400 in training and 600 in held-out evaluation.

To isolate the effect of recursive memory depth, we use a fixed-target
controlled-depth protocol. The evaluation cohort is restricted to
dialogues containing at least 15 utterances. For every evaluated depth,
the 15th utterance is used as the prediction target and the 14th
utterance as the current context. We vary only the number of immediately
preceding history utterances,
\[
L\in\{4,6,8,11,13\}.
\]
Thus, for a dialogue $(u_1,\ldots,u_{15})$, the input at depth $L$
contains the $L$ utterances immediately preceding $u_{14}$ as the
recursively compressed history, followed by the fixed context $u_{14}$,
while the prediction target remains $u_{15}$.

The dialogue cohort, current context, prediction target, and scored
target positions are held fixed across depths. Consequently, variation
along the depth axis changes only the amount of history processed through
the recursive memory-compression mechanism. All matched methods are evaluated
on exactly the same dialogue instances, history depths, context--target
pairs, and scored target positions. Within each backbone, the tokenizer,
serialization rule, and recursive CCM memory construction are also held fixed.

For the reported fixed-target depth analysis, the validation and test
partitions are combined after preprocessing to form a larger fixed
evaluation cohort. This increases the number of eligible long dialogues
and reduces sampling variability in the depth-wise comparison. Checkpoint
and hyperparameter selection are completed beforehand using the
pre-specified validation criterion; the pooled fixed-target scores are
never used for model selection. The same frozen cohort is used at every
depth and for every compared method.

We report dialogue-equal target perplexity. Prompt, history, compression,
padding, and terminal EOS positions are excluded from scoring. For
dialogue $i$ at depth $L$, let $Y_i$ denote its scored target tokens.
We first compute the mean target negative log-likelihood
\[
\ell_{i,L}
=
\frac{1}{|Y_i|}
\sum_{t\in Y_i}
-\log p(y_{i,t}\mid \operatorname{input}_{i,L}),
\]
and then aggregate dialogues with equal weight,
\[
\operatorname{PPL}^{\mathrm{dlg}}_{L}
=
\exp\!\left(
\frac{1}{N}
\sum_{i=1}^{N}
\ell_{i,L}
\right).
\]
This aggregation treats each dialogue as one evaluation unit and prevents
longer target utterances from receiving greater weight solely because
they contain more scored tokens. The same scoring and aggregation rule
is applied to every method in the matched comparison.

Because LLaMA-7B and Gemma-4 use different native tokenizers, absolute
perplexity values are compared within each backbone rather than across
backbone families.

\paragraph{MSC protocol.}
We preserve the original MSC training, validation, and test partitions
\citep{Xu2022MSC} and maintain the chronological ordering of sessions
and utterances. Each conversation spans five sessions, allowing
information accumulated from earlier interactions to remain available
during later sessions.

For the depth-controlled evaluation, the same set of complete
five-session conversations is evaluated at the session openings
\[
S\in\{3,4,5\}.
\]
For session $S$, the preceding sessions $(1,\ldots,S-1)$ constitute
the available conversation history. The opening user utterance of
session $S$ is used as the current context, and the corresponding
response is used as the prediction target. Increasing $S$ therefore
increases the amount of accumulated cross-session history available
before prediction while retaining the same conversation cohort.

For the recursive comparison, CCM and Ours process the same ordered
conversation history through the CCM recursive merge host. The session
ordering, current context, prediction target, and evaluation examples
are matched within each setting, and no future-session information is
provided as input.

Checkpoint and hyperparameter selection are performed using the
pre-specified validation criterion before final evaluation. The selected
checkpoint is then frozen and evaluated on the same held-out test conversations for all $S$.

We report dialogue-equal session-opening perplexity. Prompt, preceding
history, compression, padding, and terminal EOS positions are excluded
from scoring. For conversation $i$ at session depth $S$, let $Y_{i,S}$
denote the scored target tokens of the opening response. We first compute
the mean target negative log-likelihood
\[
\ell_{i,S}
=
\frac{1}{|Y_{i,S}|}
\sum_{t\in Y_{i,S}}
-\log p(y_{i,S,t}\mid \operatorname{input}_{i,S}),
\]
and then aggregate conversations with equal weight,
\[
\operatorname{PPL}^{\mathrm{dlg}}_{S}
=
\exp\!\left(
\frac{1}{N}
\sum_{i=1}^{N}
\ell_{i,S}
\right).
\]
The same set of complete conversations is evaluated at every session
depth, so $N$ is constant across $S$. This aggregation treats each
session-opening prediction as one evaluation unit and prevents longer
target responses from receiving greater weight solely because they
contain more scored tokens. All methods use identical scoring and aggregation.

As with DailyDialog, absolute perplexity is compared within each backbone, not across backbones.

\paragraph{Model selection and reporting.}
For each backbone and dataset, we first select one CCM-merge host checkpoint
using the pre-specified validation criterion and then keep this checkpoint
fixed. The same checkpoint provides the reported CCM-merge baseline and the
common initialization for three independent Ours training runs. Ours uses the
same pre-specified training and checkpoint-selection rule across the three
seeds, and the selected checkpoint from each run is evaluated unchanged over
the complete depth grid; no checkpoint or hyperparameter is selected per depth.

CCM-merge therefore contributes one fixed baseline result at each depth,
whereas Ours is reported as mean $\pm$ standard deviation over the three
independent training seeds. The relative PPL change is computed against the
same fixed CCM-merge baseline, with its reported variability induced by the
three Ours runs. Training seeds, checkpoint-selection rules, optimization
budgets, and trainable parameter counts are reported in
Appendix~\ref{app:implementation}.

For DailyDialog, the pooled fixed-target depth scores are computed only
after checkpoint selection and are used for the controlled depth analysis.
For MSC, the fixed CCM-merge checkpoint and the selected checkpoint from each
Ours run are evaluated on the held-out test conversations over the complete
session-depth grid $S\in\{3,4,5\}$.

\begin{table}[t]
\centering
\caption{\textbf{Language-memory evaluation protocol.}}
\label{tab:language_memory_protocol}
\small
\setlength{\tabcolsep}{4.2pt}
\renewcommand{\arraystretch}{1.06}
\begin{tabular}{lll}
\toprule
\textbf{Setting}
& \textbf{DailyDialog}
& \textbf{MSC} \\
\midrule

Task
& Next-utterance prediction
& Session-opening prediction \\

Backbones
& LLaMA-7B / Gemma-4
& LLaMA-7B / Gemma-4 \\

Host
& CCM-merge
& CCM-merge \\

Depth
& $L=4,6,8,11,13$
& $S=3,4,5$ \\

Evaluation cohort
& Fixed pooled $\geq15$-utterance dialogues
& Fixed five-session test conversations \\

History
& $L$ preceding utterances
& Sessions $1,\ldots,S-1$ \\

Current context
& Fixed $u_{14}$
& Opening user utterance of session $S$ \\

Prediction target
& Fixed $u_{15}$
& Opening response of session $S$ \\

Metric
& Dialogue-equal target PPL $\downarrow$
& Dialogue-equal session-opening PPL $\downarrow$ \\

Checkpoint selection
& \makecell[l]{CCM-M fixed once;\\ Ours validation-selected}
& \makecell[l]{CCM-M fixed once;\\ Ours validation-selected} \\

Depth evaluation
& \makecell[l]{Fixed CCM-M + each Ours checkpoint\\ over full $L$ grid}
& \makecell[l]{Fixed CCM-M + each Ours checkpoint\\ over full $S$ grid} \\

\bottomrule
\end{tabular}
\end{table}

\FloatBarrier

\subsection{Vision--Language--Action} 
\label{app:vla} 
 
\paragraph{Benchmark and data protocol.} 
We evaluate vision--language--action memory on the T3 task of 
LIBERO-Mem \citep{chung2026eslotssm}. T3 requires the robot to lift 
the same bowl and return it to the plate three times in sequence, so 
successful execution requires maintaining task progress across repeated 
interactions. 
 
LIBERO-Mem provides 100 refined T3 demonstrations for model development in 
the released training data. We partition these into 80 demonstrations for 
gradient-based training and 20 demonstrations used only as an internal 
development split for checkpoint selection. No development-split sample enters gradient updates. 
 
After checkpoint selection, the selected checkpoint is frozen and evaluated on 
the 20 official LIBERO-Mem validation initializations, which are held out from 
both gradient training and checkpoint selection and are used only for our final 
evaluation. These validation initializations are disjoint from both the 80 
training demonstrations and the 20-demonstration internal checkpoint-selection 
split.  
 
\paragraph{Pretrained initialization.} 
All runs start from the same OpenVLA-7B Prismatic checkpoint,\linebreak
\path{step-295000-epoch-40-loss=0.2200.pt}, downloaded from 
\texttt{openvla/openvla-7b-prismatic}. This checkpoint provides the 
pretrained visual encoders, Llama-2 backbone, and multimodal projector. 
It does not contain the LIBERO-Mem action model or memory modules used 
in our experiments. 
 
The parameters inherited from this pretrained checkpoint remain frozen. 
MemoryVLA-specific components are initialized separately, including 
the DiT-L action model (\texttt{action\_model}), cognitive memory bank 
(\texttt{cog\_mem\_bank}), perceptual memory bank 
(\texttt{per\_mem\_bank}), and perceptual compressor 
(\texttt{per\_compr}). The DiT-L action model predicts 
7-dimensional actions over a 16-step action horizon. 
 
\paragraph{MemoryVLA host training.} 
We first train the MemoryVLA host using AVG as its memory aggregation 
rule. During this stage, the trainable components are the LoRA 
adapters, DiT action model, \texttt{cog\_mem\_bank}, 
\texttt{per\_mem\_bank}, and \texttt{per\_compr}. The original 
pretrained Llama-2, DINOv2, SigLIP, and multimodal-projector parameters 
remain frozen. 
 
We use \texttt{--reg-head 0}; therefore, no auxiliary regression head is instantiated or trained. 
 
\paragraph{Training Ours.} 
Ours is initialized from the trained MemoryVLA host of the same 
training run. The complete host is then frozen, including its trained 
LoRA adapters, DiT action model, cognitive memory bank, perceptual 
memory bank, and perceptual compressor. 
 
We introduce our compressor \(\Gamma\), implemented as 
\texttt{GammaMerger}, and optimize only its parameters using the 
predictive-preservation objective. No existing host parameter is 
updated during this stage. We additionally verified that the parameters 
shared by MemoryVLA and Ours remain unchanged during Gamma-only 
training. 
 
Thus, Ours does not receive additional fine-tuning of the underlying 
VLA policy. The method-specific training stage only fits the newly 
introduced memory-compression module on top of an otherwise fixed 
MemoryVLA host. 
 
\paragraph{Independent training runs.} 
We report results over three independently trained runs. For each run, the 
MemoryVLA host is trained from the same pretrained OpenVLA initialization using 
the procedure above, and Ours is then trained from the corresponding host 
checkpoint with the complete host frozen and only $\Gamma$ optimized. The reported mean and standard deviation use the three independent training runs, irrespective of seed identifiers. 
 
Within each run, Ours is initialized from the corresponding trained MemoryVLA 
host, and the matched MemoryVLA baseline is evaluated from that same host 
training run. The intervention is the memory-compression mechanism and its 
predictive-preservation training. 
 
\paragraph{Checkpoint selection.} 
All formal VLA runs use a maximum budget of 25,000 optimizer steps and save 
checkpoints every 1,000 steps. Checkpoint selection is performed exclusively on 
the 20-demonstration internal development split drawn from the released refined 
training data. This split is used only for selection, with no gradient updates. 
 
Checkpoints are ranked by full-task success on the internal development split, 
with SGC used only to break ties. Once selected, a checkpoint is fixed before 
any rollout on the 20 official LIBERO-Mem validation initializations. 
 
The same checkpoint-selection rule and a maximum budget of 25,000 optimizer 
steps are applied to each method-specific training stage. After one checkpoint 
has been selected for each training run, that checkpoint is frozen before the 
official held-out evaluation is started. The checkpoint-selection traces and 
selected checkpoints are reported in Appendix~\ref{app:vla-details}; the 
training configuration is summarized in Appendix~\ref{app:training-compute}. 
 
\paragraph{Held-out matched MemoryVLA evaluation.} 
Within each training run, the selected MemoryVLA and Ours checkpoints are 
frozen and evaluated on the same fixed set of 20 official LIBERO-Mem validation 
initializations, with matched observation and action interfaces, rollout budget, 
rollout settings, and LIBERO-Mem evaluator. These validation rollouts are 
distinct from the internal checkpoint-selection demonstrations. 
 
For MemoryVLA, the host uses AVG aggregation. For Ours, the same host 
parameters are kept fixed and AVG is replaced by the trained 
\(\Gamma\) memory-compression module. This matched construction isolates the 
effect of the proposed memory-compression method from additional 
adaptation of the underlying policy. 
 
The same fixed held-out validation initializations are used for matched 
MemoryVLA--Ours comparisons across training runs. Reusing the benchmark 
evaluation cases does not create additional training runs; the reported 
uncertainty is computed across independently trained runs. 
 
\paragraph{SR and SGC.} 
We report full-task success rate (SR) and sequential-goal completion 
(SGC), both in percent. Each T3 rollout contains three sequential 
lift-and-return repetitions. 
 
Let \(N\) denote the number of evaluation rollouts. For rollout \(i\), 
let \(s_i\in\{0,1\}\) indicate whether all three repetitions are 
successfully completed according to the LIBERO-Mem evaluator, and let 
\(k_i\in\{0,1,2,3\}\) denote the number of repetitions successfully 
credited during that rollout. We compute 
\[ 
    \mathrm{SR} 
    = 
    \frac{100}{N} 
    \sum_{i=1}^{N} s_i, 
\] 
and 
\[ 
    \mathrm{SGC} 
    = 
    \frac{100}{3N} 
    \sum_{i=1}^{N} k_i. 
\] 
 
In our held-out evaluation, \(N=20\), corresponding to one closed-loop T3 
rollout from each official LIBERO-Mem validation initialization. Therefore, SR 
is the percentage of the 20 rollouts that successfully complete all three 
repetitions, whereas SGC is the total number of successfully completed 
repetitions divided 
by 
\[ 
    3 \times 20 = 60. 
\] 
For example, a rollout completing one, two, or all three repetitions 
contributes one, two, or three counts, respectively, to the SGC 
numerator. 
 
All repetition counting, stability checks, and overshoot handling use 
the original LIBERO-Mem evaluator unchanged. We introduce no 
method-specific success or partial-success criterion. 
 
\paragraph{Prior-work references and comparison policy.} 
For broader context, Table~\ref{tab:main_results} additionally includes 
the LIBERO-Mem results reported by Chung et al.~ 
\citep{chung2026eslotssm} for \(\pi_0\), SlotVLA with \(h=1\), 
SlotVLA with \(h=8\), and E-SlotSSM. These values are reproduced from 
the published study and are not re-evaluated under our local MemoryVLA 
pipeline. We therefore treat them as contextual reference points rather 
than protocol-matched baselines. 
 
All claims concerning the effect of our method are based on the matched 
MemoryVLA--Ours comparison described above. E-SlotSSM is marked with \(\dagger\) for using oracle subgoal information; it is an oracle-assisted reference. 
 
We preserve the quantities reported by the source wherever possible. 
If a prior method reports zero sequential-goal completion, we also 
report zero full-task success because completing none of the three 
required repetitions cannot constitute successful completion of T3. 
Conversely, partial sequential-goal completion does not determine 
full-task success, so SR is left unreported when it cannot be obtained 
from the source. 
 
Exact optimization hyperparameters, training budgets, checkpoint 
intervals, software versions, and hardware configuration are provided 
in Appendix~\ref{app:implementation}. 
\FloatBarrier

\subsection{Self-Recursive Improvement}
\label{app:sri-setting}

\paragraph{Purpose and evaluation setting.}
We use self-recursive improvement (SRI) as a mechanism-level test of
predictive preservation under explicit recursive strategy refinement.
In this setting, previously generated execution traces are repeatedly
reused to construct later strategies, making the retained recursive
history directly relevant to subsequent computation. Our comparison is
therefore designed to isolate the effect of predictive preservation while holding the realized initialization, capacity, recursive-search opportunity, backbone, and evaluator fixed.

\paragraph{Benchmark, data, and released host.}
We evaluate on a fixed six-task cohort from CO-Bench
\citep{Sun2026COBench}. CO-Bench provides task descriptions,
development instances, held-out test instances, and a common evaluator
that maps solver outputs to normalized task scores. We obtain the benchmark
data from the official CO-Bench release on Hugging Face,\footnote{
\url{https://huggingface.co/datasets/CO-Bench/CO-Bench}}
using the format and evaluator provided by the public CO-Bench
repository.\footnote{
\url{https://github.com/sunnweiwei/CO-Bench}}
No task instance is regenerated or custom development/test split introduced; the official evaluator and scoring rule remain unchanged.

The recursive host is the public Meta$^n$ implementation.\footnote{
\url{https://github.com/minnesotanlp/meta-n}}
We use its released CO-Bench integration
(\texttt{--benchmark co\_bench}) and evolutionary archive
orchestrator (\texttt{--use-archive}). The six-task cohort is fixed
before the paired comparison and is identical for Meta$^n$ and Ours.
The exact task identifiers, repository revisions, and run-level
configurations are retained in the experiment manifest.

\paragraph{Backbones and API execution.}
We evaluate three API backbones:
\texttt{gpt-5.5}, \texttt{gemini-3.1-pro}, and\linebreak
\path{DeepSeek-V4.1-Flash}. Within each backbone, Meta$^n$ and Ours
use the same provider-side model identifier and the same API route.
Following the released host protocol, the same backbone is used for the
Layer-1 solver and for every invocation of the recursive meta-operation
$\Omega$.

For each paired comparison, the provider route, decoding configuration,
maximum output length, retry policy, prompt templates, solver mode, and
execution timeout are frozen before recursive expansion and shared by the
two methods. Provider fallback is disabled, so failed requests follow the
same retry policy without silently switching to another model or endpoint.
Exact provider/model strings and decoding settings are recorded in the run configuration.

\paragraph{Paired realized initialization.}
Each Meta$^n$/Ours comparison begins from the same realized Layer-1
state. For a given backbone, the initial task programs, execution traces,
development scores, archive state, and associated metadata are generated
once and copied before the two recursive arms are expanded. All locally
controlled sources of randomness, including archive sampling and task
ordering, use paired settings across the two methods.

This design controls the realized API output at the point where the
methods begin to differ. Remote model calls need not be bitwise
deterministic, and independently regenerating the root would simultaneously
alter the initial programs, execution traces, and downstream recursive
environment. After the intervention begins, the two recursive trajectories
are allowed to diverge as a consequence of their different historical-trace retention
decisions; this downstream divergence is part of the mechanism being
evaluated.

\paragraph{Controlled predictive-preservation intervention.}
Meta$^n$ recursively refines candidate strategies using execution
information produced by earlier stages. In the paired comparison, both
methods use the same host procedure for constructing the available
execution-history pool, the same fixed retained-history capacity, the same
native solver stack and ordering convention, and the same surrounding
context construction.

Meta$^n$ uses its fixed historical-trace retention rule. Ours applies predictive
preservation to determine which already available execution traces are
retained for the subsequent recursive update. The
predictive-preservation configuration is fixed before the paired recursive
search and receives no held-out information.

All remaining components are shared between the two arms: the CO-Bench
tasks, Layer-1 solver, $\Omega$ prompt and output schema, archive data
structure, outer parent-selection procedure, code executor, task evaluator,
and recursive-search machinery. Predictive preservation operates only on
information already produced within the host run and creates no additional
candidate proposal, evaluator query, retrieval call, or external LLM
request. The comparison therefore changes the information retained at the
recursive interface while preserving the amount of retained history and
the external opportunities available for recursive search.

\paragraph{Matched recursive-search budget.}
We use equal recursive-search opportunity as the primary resource control.
Before each paired run, a common archive-expansion budget is fixed. Both
methods receive the same number of candidate-generation opportunities and
the same evaluation allocation associated with each candidate.

A candidate-generation slot is consumed once its generation has been
attempted. Parsing failures, invalid programs, execution errors, and
timeouts therefore remain part of the realized search cost and do not
receive free replacement proposals. The two arms also use the same
execution timeout, retry policy, maximum generation length, and maximum
number of host solver or repair calls associated with a candidate.

Method-dependent convergence statistics are retained as run diagnostics
but do not alter the number of recursive expansions available to either
method. Likewise, any difference in realized prompt length or token usage
is not converted into additional candidate generations, evaluator calls,
or model requests.

For resource accounting, we record candidate-generation attempts,
evaluator invocations, external LLM requests, input and output tokens,
and wall-clock time. Computation required internally by predictive
preservation is accounted for separately and does not enlarge the external
recursive-search budget.

\paragraph{Development-only recursive selection.}
The recursive search uses only the development instances and development
feedback supplied by the official CO-Bench evaluator. Held-out scores are
not exposed to the Layer-1 solver, $\Omega$, historical-trace retention, archive evolution, stopping logic, or predictive-preservation configuration.

After the common recursive-search budget is exhausted, the development-side
archive, recursive lineage, and selected strategies are frozen. We report
the paper-facing \emph{Dev Gain} under the same frozen development-based
selection protocol for Meta$^n$ and Ours. The same six-task development
cohort, aggregation rule, evaluator, and score scale are used in both arms.
Appendix~\ref{app:sri-results} reports cohort-level development results and search diagnostics.

Held-out results play no role in development-side candidate selection or in
the computation of Dev Gain.

\paragraph{Held-out evaluation.}
Held-out evaluation begins only after the recursive search and
development-side selection have been frozen. The main table reports
held-out scores for tasks whose recursive strategy is modified during the
frozen development-side refinement process. The modified-task set is
identified from the recursive run artifacts before held-out scores are
examined.

For every task included in this set, the frozen Meta$^n$ and Ours
solutions are evaluated on the same official held-out instances, and their
task scores are reported side by side. If several tasks are modified for a
backbone, all such tasks are retained in the comparison; held-out
performance is not used to choose among them.

The modified-task presentation in Table~\ref{tab:main_results} highlights
where the recursive intervention produces an actual downstream strategy
change. For completeness, Appendix~\ref{app:sri-results} reports the
task-wise held-out results for the full six-task cohort together with the
unconditional cohort-level averages for Meta$^n$ and Ours. Consequently,
the complete held-out record does not depend on the direction or magnitude
of the main-table differences.

\paragraph{Repeated evaluation of frozen strategies.}
For each backbone, we use one paired recursive-search realization beginning
from the shared realized Layer-1 state described above. Once
development-side selection is complete, the resulting strategies, recursive
configuration, prompt construction, model route, and held-out task set are
frozen.

Each frozen solution is evaluated three times on the same held-out
instances. Meta$^n$ and Ours use the same backbone, provider route,
decoding configuration, instance set, instance ordering, retry policy, and
official evaluator in every repetition. The three-evaluation protocol is fixed beforehand and recorded in the experiment manifest.

These repetitions quantify execution-time variability of an already
selected strategy under API stochasticity. They are repeated evaluations
of one frozen recursive-search outcome and are not treated as independent
recursive-search seeds. The main table reports the mean held-out score over
the three matched evaluations. Appendix~\ref{app:sri-results} reports one
fixed held-out evaluation for the complete six-task cohort and a compact
repeat-level variability summary.

\paragraph{Reporting and bias control.}
The SRI protocol is fixed before held-out evaluation. The task cohort,
realized Layer-1 initialization, history-pool construction, retained-history
capacity, recursive-search budget, failure accounting, development-side
selection procedure, modified-task identification rule, and repeated
held-out evaluation protocol are all determined without access to
held-out outcomes.

Dev Gain and Held-out Score serve separate roles. Dev Gain measures the
result of recursive refinement under the feedback available during search.
Held-out Score evaluates the frozen resulting strategies on unseen
instances. Both are computed from the common CO-Bench evaluator and neither
is defined from the internal predictive statistic.

The paper-facing SRI comparison is limited to these task-performance
quantities and resource-accounting statistics. We do not introduce
additional post-hoc thresholds or method-specific evaluation scores into
the comparison. Main-table modified-task results are accompanied by the
complete six-task results and unconditional cohort-level averages in the
appendix.

\begin{table}[t]
\centering
\caption{
\textbf{Controlled SRI protocol on CO-Bench.}
Meta$^n$ and Ours share the realized initialization, retained-history
capacity, recursive-search opportunity, and evaluation protocol; the
method-specific intervention changes which execution traces are retained
for subsequent recursive refinement.
}
\label{tab:sri_protocol}
\small
\setlength{\tabcolsep}{4.3pt}
\renewcommand{\arraystretch}{1.05}
\begin{tabular}{ll}
\toprule
\textbf{Component} & \textbf{Protocol} \\
\midrule
Data source
& Official CO-Bench release \\
Task cohort
& Same fixed six tasks \\
Host
& Public Meta$^n$ implementation \\
Backbones
& gpt-5.5 / gemini-3.1-pro / DeepSeek-V4.1-Flash \\
Model assignment
& Same backbone at Layer 1 and all $\Omega$ calls \\
Initialization
& Same realized Layer-1 state within each pair \\
History-pool construction
& Same host procedure \\
Retained-history capacity
& Same fixed capacity in both arms \\
Native stack
& Same construction and ordering \\
Method-specific change
& Historical-trace retention rule only \\
Recursive-search procedure
& Same host procedure \\
Search budget
& Same fixed expansion budget \\
Candidate failures
& Consume the assigned search slot \\
Evaluator
& Official CO-Bench evaluator, unchanged \\
Search feedback
& Development only \\
Held-out access during search
& None \\
Additional API calls for Ours
& None \\
Search realizations
& One paired realization per backbone \\
Development result
& Dev Gain $\uparrow$ \\
Held-out evaluation
& Three matched evaluations of each frozen solution \\
Main-table held-out rows
& Three-evaluation mean for all development-defined modified tasks \\
Appendix held-out results
& One fixed evaluation for all six tasks + unconditional average \\
Resource accounting
& Candidates / evals / calls / tokens / wall-clock \\
\bottomrule
\end{tabular}
\end{table}

\paragraph{Run artifacts and reproducibility.}
We retain the native Meta$^n$ artifacts for both arms, including
\texttt{config.json}, \texttt{summary.json}, convergence records,
\texttt{checkpoint.json}, \texttt{archive/index.json}, per-candidate
execution traces and injected-code records, and the JSONL log of external
LLM requests and responses.

The manifest also records the six CO-Bench task identifiers, CO-Bench and Meta$^n$ repository revisions, exact
provider/model identifiers, decoding and retry settings, paired local
randomness configuration, common recursive-search budget, evaluator
timeout, history-pool configuration, retained-history capacity,
historical-trace retention settings, development-selected strategies, modified-task set, and repeated held-out evaluation settings.

The minimal public data and software setup is:

\begin{verbatim}
# CO-Bench data
huggingface-cli download CO-Bench/CO-Bench \
    --repo-type dataset --local-dir data/co_bench

# Meta^n host
git clone https://github.com/minnesotanlp/meta-n
cd meta-n
python -m venv .venv
source .venv/bin/activate
pip install -e ".[bench]"
\end{verbatim}

API-backed runs use the exact provider/model identifiers stored in the
frozen run configuration. Authentication is supplied through environment
variables and is not serialized into released experiment artifacts.
Paper-facing commands are generated from the frozen configuration rather
than edited separately between Meta$^n$ and Ours.

Together, these artifacts reconstruct the common realized initialization, method-specific recursive lineages, development-side selection, and repeated held-out evaluation.

\FloatBarrier

\clearpage
\section{Extended Results and Diagnostic Analyses}
\label{app:additional_results}

This appendix reports complete results and domain-specific diagnostics for the main comparisons.
Temporal graphs include five-seed results and predictive-constraint diagnostics;
language memory includes the full depth-wise perplexity comparisons and
training-time predictive diagnostics; VLA includes internal checkpoint-selection,
held-out initialization-level, and rollout analyses together with the three-run
aggregate; and SRI
includes the complete six-task results, recursive candidate and strategy
evolution, and held-out variability.

\subsection{Temporal Graph Results}
\label{app:tg_results}

Table~\ref{tab:app_tg_full} reports the complete temporal-graph results.
All entries are averaged over five random seeds and are reported as
mean $\pm$ standard deviation.

\begin{table*}[t]
\caption{
\textbf{Complete temporal-graph results.}
Results are reported as mean $\pm$ standard deviation over five random
seeds. AP and AUC are reported in \%.
}
\label{tab:app_tg_full}

\centering
\footnotesize
\setlength{\tabcolsep}{3.2pt}
\renewcommand{\arraystretch}{1.08}

\resizebox{\textwidth}{!}{%
\begin{tabular}{
@{}llcccccc@{}
}
\toprule

\textbf{Dataset / Task}
&
\textbf{Setting}
&
\textbf{TGN}
&
\textbf{TGAT}
&
\textbf{TCL}
&
\textbf{CTAN}
&
\textbf{DyGFormer}
&
\textbf{Ours}
\\
\midrule

\multirow{2}{*}{
\makecell[l]{Wikipedia\\Node Classification}
}
&
Depth 2 AUC $\uparrow$
& $83.61 \pm 1.00$
& $85.45 \pm 0.26$
& $85.64 \pm 0.32$
& $87.31 \pm 0.29$
& $82.90 \pm 1.59$
& $\mathbf{88.81 \pm 0.77}$
\\

&
Depth 3 AUC $\uparrow$
& $87.47 \pm 1.05$
& $82.08 \pm 1.73$
& $81.26 \pm 1.77$
& $87.38 \pm 0.14$
& $87.00 \pm 0.71$
& $\mathbf{89.17 \pm 0.67}$
\\

\midrule

\multirow{4}{*}{
\makecell[l]{UCI\\Link Prediction}
}
&
Depth 3 AP $\uparrow$
& $93.04 \pm 1.73$
& $86.08 \pm 0.66$
& $86.35 \pm 0.82$
& --
& $\mathbf{95.47 \pm 0.27}$
& $94.14 \pm 0.33$
\\

&
Depth 3 AUC $\uparrow$
& $92.79 \pm 1.68$
& $85.44 \pm 0.61$
& $85.95 \pm 0.60$
& --
& $\mathbf{94.52 \pm 0.23}$
& $93.73 \pm 0.56$
\\

&
Depth 4 AP $\uparrow$
& $94.07 \pm 0.13$
& $84.14 \pm 1.70$
& $83.78 \pm 0.62$
& --
& $94.92 \pm 0.25$
& $94.89 \pm 0.09$
\\

&
Depth 4 AUC $\uparrow$
& $93.73 \pm 0.12$
& $83.09 \pm 1.79$
& $83.20 \pm 0.59$
& --
& $94.14 \pm 0.21$
& $\mathbf{94.47 \pm 0.12}$
\\

\bottomrule
\end{tabular}
}
\end{table*}
\begin{figure}[t]
    \centering
    \includegraphics[width=0.92\linewidth]{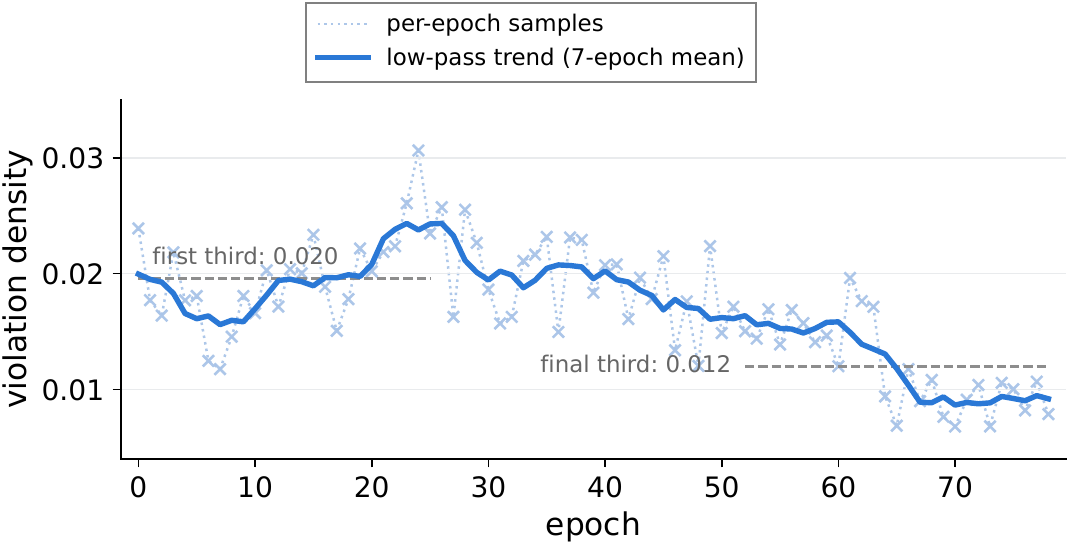}
    \caption{
    \textbf{Treewise predictive-constraint conflict during training.}
    Violation density is the fraction of grouped predictive constraints violated
    by the unconstrained task proposal. On UCI 3-hop (R0, $\kappa=0.05$), the
    mean violation density decreases from $0.020$ in the first third of training
    to $0.012$ in the final third.
    }
    \label{fig:tgn-constraint-conflict}
\end{figure}
Figure~\ref{fig:tgn-constraint-conflict} shows that predictive-constraint
conflicts remain present throughout training but become less frequent in later
epochs. The trajectory is not monotonic: conflict density rises temporarily in
the middle of training before declining, indicating that task optimization
continues to propose directions that can interfere with predictive preservation
as the representation evolves. The lower density in the final third suggests
that the task proposal becomes progressively more aligned with the
predictive-feasible region, reducing the amount of correction required from
the projection. This trend is consistent with the compressor and task objective
becoming better coordinated over training, while the remaining violations
motivate retaining the predictive constraints rather than assuming that such
conflicts disappear entirely.
\begin{table}[t]
\centering
\caption{
\textbf{Real-data first-order bridge audit on UCI.}
We compare the first-order change predicted by the constraint surrogate
with the measured finite change of the exact window statistic used by
\texttt{PairKFWindow}. The perturbation scale $\alpha$ multiplies the same
parameter-space update direction.
}
\label{tab:first_order_bridge}
\small
\setlength{\tabcolsep}{6.0pt}
\renewcommand{\arraystretch}{1.08}
\begin{tabular}{ccccc}
\toprule
$\alpha$
&
$\alpha\, q^\top \delta$
&
$\Delta \widehat J_{\mathrm{win}}$
&
$\Delta \widehat J_{\mathrm{win}} / (\alpha q^\top \delta)$
&
Relative scaling
\\
\midrule
1
&
$3.639\times10^{-7}$
&
$4.044\times10^{-7}$
&
$1.111$
&
$1.0\times$
\\

16
&
$5.822\times10^{-6}$
&
$6.386\times10^{-6}$
&
$1.097$
&
$15.8\times$
\\
\bottomrule
\end{tabular}
\end{table}

\paragraph{Interpretation.}
Table~\ref{tab:first_order_bridge} directly tests the local approximation
used to construct our predictive constraints. Here,
$q=\nabla_\theta \widehat J_{\mathrm{win}}$ denotes the local predictive
sensitivity represented by the constraint surrogate, $\delta$ is the
parameter-space perturbation direction, and
$\Delta \widehat J_{\mathrm{win}}$ uses the same \texttt{PairKFWindow} tree-weighted window aggregation.

The predicted and measured changes have the same sign and closely matched
magnitudes. At the base perturbation scale, the measured change is only
$1.111\times$ the first-order prediction; after increasing the perturbation
by a factor of $16$, this ratio remains nearly unchanged at $1.097$.
More importantly, the predicted change grows by exactly $16.0\times$,
while the measured change grows by $15.8\times$. Thus, over this
resolvable local range, the empirical window statistic responds almost
linearly to the perturbation and is well tracked by the directional
quantity used in the constraint construction.

This audit is important because the optimization procedure does not
directly constrain the finite post-update value of
$\widehat J_{\mathrm{win}}$. Instead, it constrains its local first-order
surrogate through terms of the form $q^\top\delta$. The close agreement
above therefore verifies, on a real UCI training window, the practical
bridge between the predictive statistic used in the implementation and
the linearized half-space constraints imposed on parameter updates.
The remaining approximately $10\%$ scale difference is stable across
perturbation sizes rather than increasing with $\alpha$, and therefore
does not exhibit the growth expected from a dominant nonlinear
higher-order error.
\FloatBarrier

\subsection{Language Memory Results}
\label{app:language_results}

We report the matched CCM-merge--Ours comparison on DailyDialog and
Multi-Session Chat (MSC) across LLaMA-7B and Gemma-4-E4B-It.
DailyDialog uses the fixed-target controlled-depth protocol described in
Appendix~\ref{app:language_memory}, while MSC evaluates session-opening
prediction as the amount of preceding cross-session history increases.

Within each backbone and dataset, CCM-merge and Ours use the same
evaluation examples, memory-depth construction, tokenizer, prediction
targets, and perplexity aggregation. The CCM-merge host checkpoint is
selected once by the pre-specified validation criterion and then kept fixed.
Ours is initialized from this same host checkpoint and trained under three
independent seeds. The selected checkpoint from each Ours run is evaluated
unchanged across the reported depth grid. Absolute perplexity values are
interpreted only within the same backbone because LLaMA-7B and
Gemma-4-E4B-It use different native tokenizers.

Table~\ref{tab:language_memory_complete} reports the complete matched
comparison at the evaluated depths. We abbreviate CCM-merge as CCM-M. For
Ours seed $r$, the relative PPL change with respect to the fixed CCM-M host is
\[
\Delta_r(\%)
=
100\,
\frac{
\operatorname{PPL}_{\mathrm{Ours},r}
-
\operatorname{PPL}_{\mathrm{CCM\text{-}M}}
}{
\operatorname{PPL}_{\mathrm{CCM\text{-}M}}
}.
\]
We report $\Delta_r$ mean and standard deviation across three Ours seeds. Negative values mean lower Ours PPL.

\begin{table*}[t] 
\centering 
\caption{ 
\textbf{Language-memory results across recursive depth.} 
DailyDialog reports dialogue-equal target PPL under the fixed-target 
controlled-depth protocol, while MSC reports dialogue-equal 
session-opening PPL. Lower PPL is better. CCM-merge reports one fixed 
validation-selected host checkpoint. Ours is initialized from that checkpoint 
and reports mean $\pm$ standard deviation over three independent training 
seeds. $\Delta(\%)$ reports the corresponding mean $\pm$ standard deviation 
of the seed-wise relative PPL change with respect to the fixed CCM-merge 
baseline. 
} 
\label{tab:language_memory_complete} 
 
\small 
\setlength{\tabcolsep}{5.0pt} 
\renewcommand{\arraystretch}{1.08} 
 
\resizebox{\textwidth}{!}{%
\begin{tabular}{ 
@{}ll 
ccc 
ccc 
@{} 
} 
\toprule 
 
\multirow{2}{*}{\textbf{Dataset / Metric}} 
& 
\multirow{2}{*}{\textbf{Depth}} 
& 
\multicolumn{3}{c}{\textbf{LLaMA-7B}} 
& 
\multicolumn{3}{c}{\textbf{Gemma-4-E4B-It}} 
\\ 
 
\cmidrule(lr){3-5} 
\cmidrule(lr){6-8} 
 
& 
& 
\textbf{CCM-M} 
& 
\textbf{Ours} 
& 
$\boldsymbol{\Delta}$ \textbf{(\%)} 
& 
\textbf{CCM-M} 
& 
\textbf{Ours} 
& 
$\boldsymbol{\Delta}$ \textbf{(\%)} 
\\ 
\midrule 
 
 
\multirow{6}{*}{ 
\makecell[l]{\textbf{DailyDialog}\\PPL $\downarrow$} 
} 
& 
\textit{Selected ckpt.} 
& \textit{s650} 
& \textit{s70} 
& -- 
& \textit{s600} 
& \textit{s40} 
& -- 
\\ 
\cmidrule(lr){2-8} 
 
& 
$L=4$ 
& 5.98 
& \textbf{$5.86\pm0.03$} 
& $-1.98\pm0.52$ 
& 7.00 
& \textbf{$6.97\pm0.01$} 
& $-0.54\pm0.19$ 
\\ 
 
& 
$L=6$ 
& 5.85 
& \textbf{$5.71\pm0.05$} 
& $-2.51\pm0.80$ 
& 6.84 
& \textbf{$6.78\pm0.01$} 
& $-0.80\pm0.11$ 
\\ 
 
& 
$L=8$ 
& 5.80 
& \textbf{$5.65\pm0.07$} 
& $-2.56\pm1.13$ 
& 6.85 
& \textbf{$6.77\pm0.05$} 
& $-1.05\pm0.75$ 
\\ 
 
& 
$L=11$ 
& 5.72 
& \textbf{$5.55\pm0.12$} 
& $-2.91\pm2.06$ 
& 6.77 
& \textbf{$6.66\pm0.07$} 
& $-1.64\pm0.98$ 
\\ 
 
& 
$L=13$ 
& 5.63 
& \textbf{$5.38\pm0.12$} 
& $-4.53\pm2.03$ 
& 6.73 
& \textbf{$6.58\pm0.07$} 
& $-2.24\pm0.97$ 
\\ 
 
\midrule 
 
 
\multirow{4}{*}{ 
\makecell[l]{\textbf{MSC}\\PPL $\downarrow$} 
} 
& 
\textit{Selected ckpt.} 
& \textit{s1000} 
& \textit{s40} 
& -- 
& \textit{s600} 
& \textit{s50} 
& -- 
\\ 
\cmidrule(lr){2-8} 
 
& 
$S=3$ 
& 6.82 
& \textbf{$6.77\pm0.00$} 
& $-0.84\pm0.03$ 
& 7.96 
& \textbf{$7.96\pm0.00$} 
& $0.00\pm0.00$ 
\\ 
 
& 
$S=4$ 
& 7.02 
& \textbf{$6.95\pm0.01$} 
& $-1.09\pm0.10$ 
& 8.23 
& \textbf{$8.22\pm0.00$} 
& $-0.12\pm0.00$ 
\\ 
 
& 
$S=5$ 
& 7.03 
& \textbf{$6.93\pm0.01$} 
& $-1.39\pm0.10$ 
& 8.16 
& \textbf{$8.14\pm0.00$} 
& $-0.25\pm0.00$ 
\\ 
 
\bottomrule 
\end{tabular} 
} 
\end{table*}

\paragraph{Training-time predictive diagnostics.}
Beyond the final depth-wise results, we examine whether the predictive-preservation
signal retained across the recursive memory interfaces evolves consistently with
downstream validation performance during training.
For the internal diagnostic, $J$ denotes the finite-test predictive value used
in the main mechanism analysis, and we compute
$\Delta J = J_{\mathrm{real}}-J_{\mathrm{shuffled}}$ at every optimization
window, where $J_{\mathrm{real}}$ uses the correctly paired future
observations and $J_{\mathrm{shuffled}}$ breaks this correspondence.
We additionally report $\Delta J_{\mathrm{within}\text{-}L}$, where
shuffling is restricted to examples at the same recursive depth, thereby
controlling for depth-dependent effects. Because the per-optimization-window statistic
is noisy, we report raw measurements and a $w=9$ smoothed trajectory.

Figure~\ref{fig:ccm_training_diagnostics} compares these internal predictive-preservation diagnostics
with the validation trajectory. The right panel reports the negative mean
relative PPL change of Ours with respect to the fixed CCM-merge host
baseline, averaged over all evaluated depths, so both panels follow
the same convention that
higher values indicate better behavior. The predictive-preservation
signals remain positive throughout training, while the validation
improvement also remains above the CCM-merge baseline. More importantly,
their training dynamics are broadly aligned: the smoothed $\Delta J$
signals strengthen through the intermediate stage and reach their strongest
region around the later-middle windows, while the validation gain similarly
increases toward s40; both subsequently weaken toward s50. This shared
trajectory links the internal predictive-preservation signal to the
observed downstream improvement and provides supporting evidence for the
validation-selected checkpoint.

\begin{figure*}[t]
    \centering
    \includegraphics[
        width=\textwidth
    ]{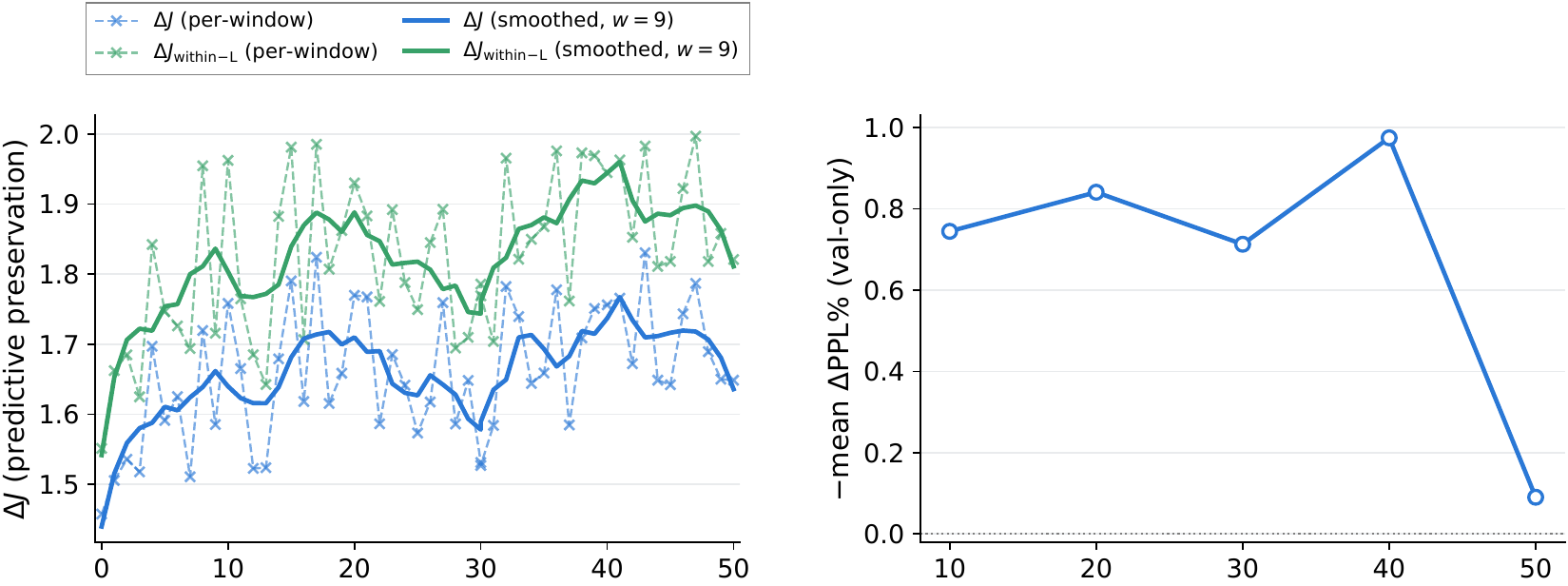}
    \vspace{-0.5em}

    \begin{minipage}[t]{0.48\linewidth}
        \centering
        \small (a) $\Delta J$ vs.\ shuffled controls
    \end{minipage}
    \hfill
    \begin{minipage}[t]{0.48\linewidth}
        \centering
        \small (b) Validation $\Delta$PPL (\%, inverted)
    \end{minipage}

    \vspace{-0.3em}
    \caption{
    \textbf{Training-time predictive diagnostics for CCM.}
    Internal predictive-preservation signals and downstream validation
    gains are shown over the same training trajectory, where higher values are better in both panels.
    }
    \label{fig:ccm_training_diagnostics}
\end{figure*}
\FloatBarrier

\subsection{Vision--Language--Action Results}
\label{app:vla-details}

Under the matched protocol in Appendix~\ref{app:vla}, we evaluate three
independently trained runs on the same 20 official held-out evaluation
initializations. Each evaluation initialization produces one closed-loop T3
rollout, within which up to three ordered lift-and-return repetitions can be
completed sequentially. For each training run, the checkpoint is selected on
the internal development split using strict T3 success, with T3 SGC used as a
tie-breaker; the resulting checkpoint step is reported below.

\begin{table*}[t]
\centering
\caption{
\textbf{Held-out initialization-level comparison for training run 1.}
Ours at 12k and MemoryVLA at 20k are evaluated on the same 20 official
held-out evaluation initializations. One closed-loop T3 rollout is run from
each initialization. ``Success'' reports the number of ordered T3 repetitions
completed sequentially within that rollout, from 0/3 to 3/3.
``Completion step'' gives the rollout step at which each credited repetition
is completed.
}
\label{tab:vla_seed1_demo}

\small
\setlength{\tabcolsep}{8.0pt}
\renewcommand{\arraystretch}{1.05}

\begin{tabular}{c cc cc}
\toprule
&
\multicolumn{2}{c}{\textbf{Ours @ 12k}}
&
\multicolumn{2}{c}{\textbf{MemoryVLA @ 20k}}
\\
\cmidrule(lr){2-3}
\cmidrule(lr){4-5}

\textbf{Eval. init.}
&
\textbf{Success}
&
\textbf{Completion step}
&
\textbf{Success}
&
\textbf{Completion step}
\\
\midrule

1  & 0/3 & --            & 0/3 & -- \\
2  & 1/3 & 105           & 1/3 & 105 \\
3  & 2/3 & 230, 259      & 3/3 & 230, 255, 293 \\
4  & 0/3 & --            & 0/3 & -- \\
5  & 1/3 & 135           & 1/3 & 135 \\
6  & 3/3 & 206, 256, 329 & 2/3 & 240, 452 \\
7  & 1/3 & 408           & 0/3 & -- \\
8  & 0/3 & --            & 0/3 & -- \\
9  & 0/3 & --            & 0/3 & -- \\
10 & 0/3 & --            & 0/3 & -- \\
11 & 0/3 & --            & 0/3 & -- \\
12 & 1/3 & 421           & 0/3 & -- \\
13 & 2/3 & 495, 543      & 0/3 & -- \\
14 & 1/3 & 483           & 0/3 & -- \\
15 & 2/3 & 202, 293      & 1/3 & 200 \\
16 & 3/3 & 222, 251, 326 & 2/3 & 242, 343 \\
17 & 0/3 & --            & 0/3 & -- \\
18 & 0/3 & --            & 1/3 & 231 \\
19 & 0/3 & --            & 0/3 & -- \\
20 & 0/3 & --            & 0/3 & -- \\

\midrule

\textbf{Completed repetitions}
&
\textbf{17/60}
&
--
&
11/60
&
--
\\

\textbf{Strict T3 SR}
&
\textbf{2/20}
&
--
&
1/20
&
--
\\

\bottomrule
\end{tabular}
\end{table*}

\begin{table*}[t] 
\centering 
\caption{ 
\textbf{Complete held-out T3 results on LIBERO-Mem across training runs.} 
Each training run is evaluated on the same 20 official held-out evaluation 
initializations, with one closed-loop T3 rollout from each initialization. 
Within each rollout, the task requires three ordered lift-and-return 
repetitions to be completed sequentially. Columns $0/3$--$3/3$ report the 
number of evaluation initializations whose rollout completes exactly zero, 
one, two, or three ordered repetitions, respectively. T3 SR denotes strict 
task success, i.e., the fraction of the 20 rollouts that complete all three 
ordered repetitions. T3 SGC is the total number of credited ordered 
repetitions completed within the 20 rollouts, divided by the maximum possible 
$20\times3=60$ ordered repetitions. The aggregate rows pool the corresponding 
initialization--training-run instances, and report the three-run mean $\pm$ standard deviation. 
} 
\label{tab:vla_t3_complete} 
 
\small 
\setlength{\tabcolsep}{5.0pt} 
\renewcommand{\arraystretch}{1.08} 
 
\begin{tabular}{lllrrrrcc} 
\toprule 
\textbf{Training run} 
& 
\textbf{Method} 
& 
\textbf{Ckpt. step} 
& 
$\mathbf{0/3}$ 
& 
$\mathbf{1/3}$ 
& 
$\mathbf{2/3}$ 
& 
$\mathbf{3/3}$ 
& 
\textbf{T3 SR $\uparrow$} 
& 
\textbf{T3 SGC $\uparrow$} 
\\ 
\midrule 
 
\multirow{2}{*}{Run 1} 
& MemoryVLA 
& 20k 
& 13 
& 4 
& 2 
& 1 
& 5.0\% 
& 18.3\% 
\\ 
 
& \textbf{Ours} 
& 12k 
& 10 
& 5 
& 3 
& 2 
& \textbf{10.0\%} 
& \textbf{28.3\%} 
\\ 
 
\cmidrule(lr){1-9} 
 
\multirow{2}{*}{Run 2} 
& MemoryVLA 
& 18k 
& 11 
& 6 
& 3 
& 0 
& 0.0\% 
& 20.0\% 
\\ 
 
& \textbf{Ours} 
& 17k 
& 12 
& 3 
& 3 
& 2 
& \textbf{10.0\%} 
& \textbf{25.0\%} 
\\ 
 
\cmidrule(lr){1-9} 
 
\multirow{2}{*}{Run 3} 
& MemoryVLA 
& 20k 
& 10 
& 7 
& 3 
& 0 
& 0.0\% 
& 21.7\% 
\\ 
 
& \textbf{Ours} 
& 18k 
& 11 
& 4 
& 2 
& 3 
& \textbf{15.0\%} 
& \textbf{28.3\%} 
\\ 
 
\midrule 
 
\multirow{2}{*}{\textbf{Aggregate}} 
& MemoryVLA 
& -- 
& 34 
& 17 
& 8 
& 1 
& $1.67 \pm 2.89$\% 
& $20.00 \pm 1.67$\% 
\\ 
 
& \textbf{Ours} 
& -- 
& 33 
& 12 
& 8 
& 7 
& $\mathbf{11.67 \pm 2.89}$\% 
& $\mathbf{27.22 \pm 1.92}$\% 
\\ 
 
\bottomrule 
\end{tabular} 
\end{table*}

\paragraph{Qualitative rollout comparison.}
To complement the aggregate and initialization-level results above, we
visualize two paired held-out rollouts with different sequential-completion
counts. On evaluation initialization 15, Ours completes 2/3 ordered
repetitions within the rollout, compared with 1/3 for MemoryVLA. On evaluation
initialization 16, Ours completes all 3/3 ordered repetitions, compared with
2/3 for MemoryVLA. The snapshots show the corresponding rollout evolution at
regularly sampled time steps under the same task setting.

\begin{figure*}[t]
    \centering

    \begin{minipage}[t]{0.49\textwidth}
        \centering
        \includegraphics[width=\linewidth]{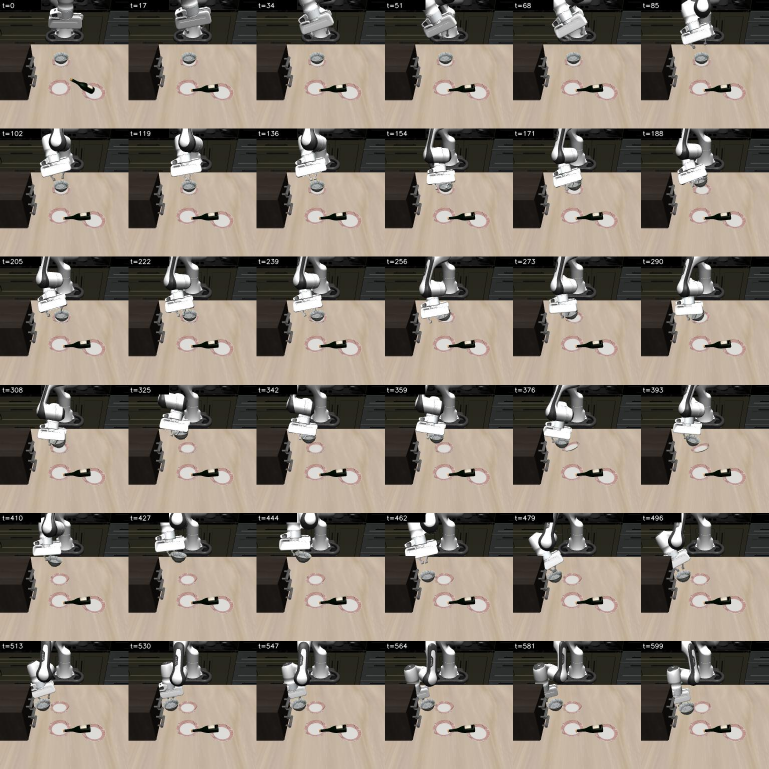}

        \vspace{0.25em}
        \textbf{(a) Ours @ 12k}
    \end{minipage}
    \hfill
    \begin{minipage}[t]{0.49\textwidth}
        \centering
        \includegraphics[width=\linewidth]{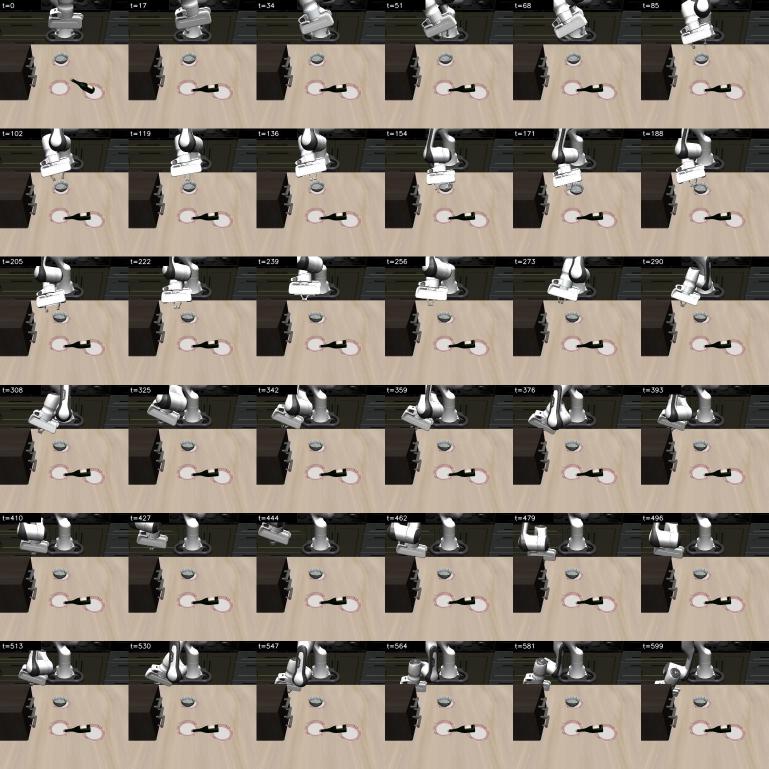}

        \vspace{0.25em}
        \textbf{(b) MemoryVLA @ 20k}
    \end{minipage}

    \caption{
    \textbf{Qualitative rollout comparison on LIBERO-Mem T3, held-out evaluation initialization 15.}
    Snapshots are sampled at regular intervals from paired rollouts of
    training run 1. Within this held-out T3 rollout, Ours completes 2/3
    ordered repetitions, while MemoryVLA completes 1/3
    (Table~\ref{tab:vla_seed1_demo}).
    }
    \label{fig:vla_demo15}
\end{figure*}

\begin{figure*}[t]
    \centering

    \begin{minipage}[t]{0.49\textwidth}
        \centering
        \includegraphics[width=\linewidth]{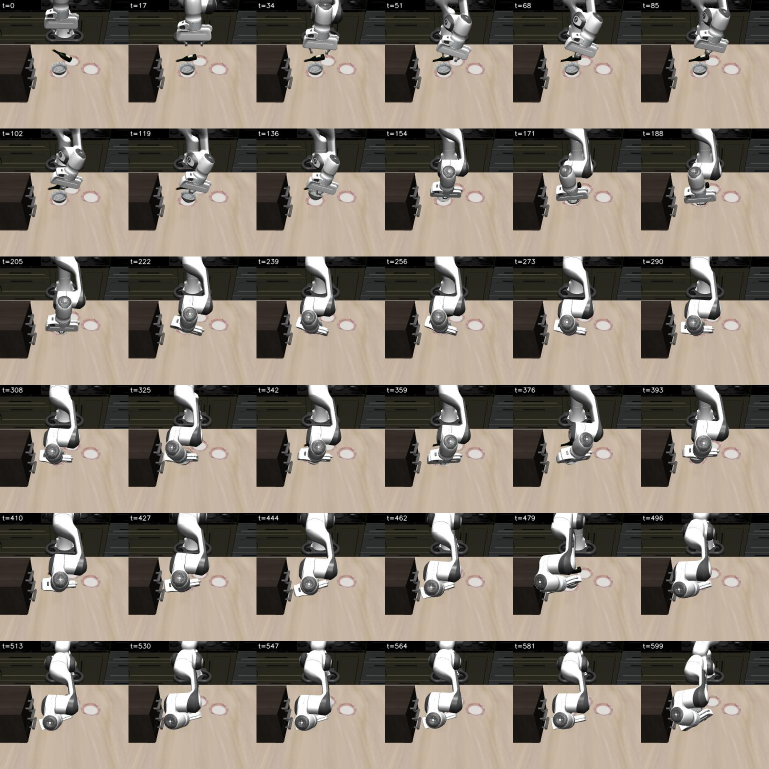}

        \vspace{0.25em}
        \textbf{(a) Ours @ 12k}
    \end{minipage}
    \hfill
    \begin{minipage}[t]{0.49\textwidth}
        \centering
        \includegraphics[width=\linewidth]{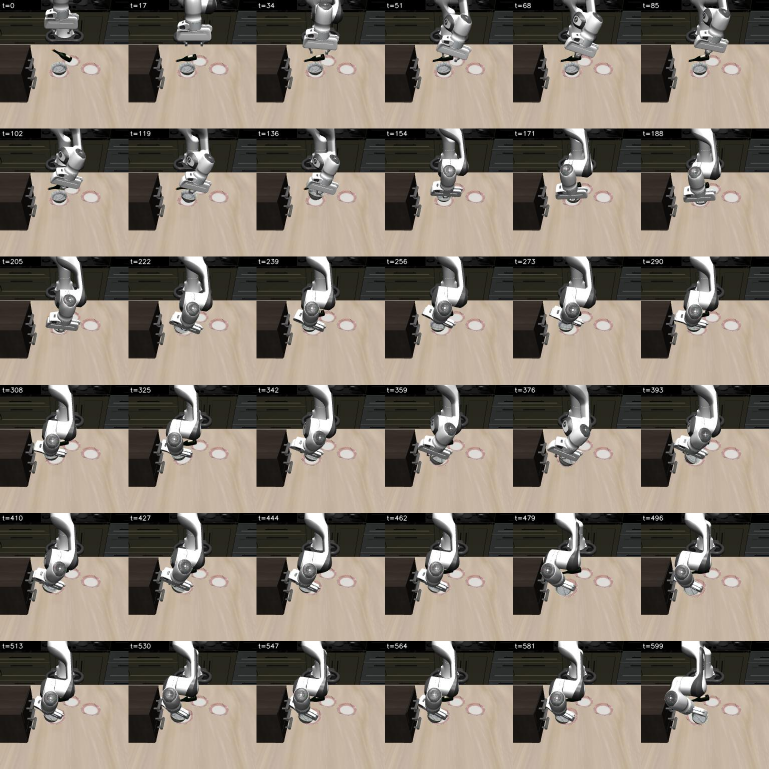}

        \vspace{0.25em}
        \textbf{(b) MemoryVLA @ 20k}
    \end{minipage}

    \caption{
    \textbf{Qualitative rollout comparison on LIBERO-Mem T3, held-out evaluation initialization 16.}
    Snapshots are sampled at regular intervals from paired rollouts of
    training run 1. Within this held-out T3 rollout, Ours completes all 3/3
    ordered repetitions, while MemoryVLA completes 2/3
    (Table~\ref{tab:vla_seed1_demo}).
    }
    \label{fig:vla_demo16}
\end{figure*}

\subsection{Recursive Self-Improvement Results}
\label{app:sri-results}

We provide additional results for recursive self-improvement (SRI),
including task-level held-out performance, candidate evolution during
recursive search, and a qualitative trace of the solver strategies
generated across successive rounds.
The experiments test whether predictive preservation changes the information
propagated through recursive improvement and, consequently, the downstream
solver trajectory. All three backbones follow the same SRI protocol described
in Appendix~\ref{app:sri-setting}.

\begin{table*}[t]
\centering
\caption{
\textbf{Summary of SRI results.}
We report development scores from the paired recursive-search run and one
fixed held-out evaluation of the frozen Meta$^n$ and Ours strategies.
The main table reports the mean over three matched held-out evaluations of
the same frozen strategies.
}
\label{tab:sri_summary}

\small
\setlength{\tabcolsep}{6.5pt}
\renewcommand{\arraystretch}{1.08}

\begin{tabular}{lccc ccc}
\toprule
&
\multicolumn{3}{c}{\textbf{Development}}
&
\multicolumn{3}{c}{\textbf{Held-out}}
\\
\cmidrule(lr){2-4}
\cmidrule(lr){5-7}

\textbf{Model}
&
\textbf{Meta$^n$}
&
\textbf{Ours}
&
$\boldsymbol{\Delta}$
&
\textbf{Meta$^n$}
&
\textbf{Ours}
&
$\boldsymbol{\Delta}$
\\
\midrule

gemini-3.1-pro
& 0.9001
& \textbf{0.9547}
& \textbf{+0.0546}
& 0.8338
& \textbf{0.8479}
& \textbf{+0.0141}
\\

gpt-5.5
& 0.8172
& \textbf{0.8279}
& \textbf{+0.0107}
& 0.7527
& \textbf{0.7579}
& \textbf{+0.0051}
\\

DeepSeek-V4.1-Flash
& 0.8070
& \textbf{0.8600}
& \textbf{+0.0530}
& 0.7924
& \textbf{0.8222}
& \textbf{+0.0298}
\\

\bottomrule
\end{tabular}
\end{table*}

\begin{table*}[t]
\centering
\caption{
\textbf{Task-level held-out SRI results.}
$\Delta$ denotes Ours minus Meta$^n$ in percentage points; scores are shown in percent.
Differences are computed before rounding, so displayed values may differ by 0.01 percentage points.
Each panel reports one fixed held-out evaluation of the frozen strategies
from the single paired recursive-search run for that backbone. The main table reports the mean over
three matched held-out evaluations of the same frozen strategies.
}
\label{tab:sri_taskwise}

\small
\setlength{\tabcolsep}{8pt}
\renewcommand{\arraystretch}{1.06}

\begin{tabular}{lrrr}
\toprule
\multicolumn{4}{l}{
\textbf{(a) gemini-3.1-pro}
}
\\
\midrule

\textbf{Task}
&
\textbf{Meta$^n$}
&
\textbf{Ours}
&
$\boldsymbol{\Delta}$
\\
\midrule

\texttt{aircraft\_landing}
& 83.89
& \textbf{92.82}
& \textbf{+8.93}
\\

\texttt{assignment\_problem}
& 100.00
& 100.00
& 0.00
\\

\texttt{assortment\_problem}
& 55.27
& 55.27
& 0.00
\\

\texttt{bin\_packing}
& \textbf{94.17}
& 93.57
& -0.60
\\

\texttt{capacitated\_warehouse\_location}
& 69.28
& \textbf{69.40}
& +0.12
\\

\texttt{common\_due\_date\_scheduling}
& 97.67
& \textbf{97.68}
& +0.01
\\

\midrule
\textbf{Mean}
& 83.38
& \textbf{84.79}
& \textbf{+1.41}
\\
\bottomrule
\end{tabular}

\vspace{0.8em}

\begin{tabular}{lrrr}
\toprule
\multicolumn{4}{l}{
\textbf{(b) gpt-5.5}
}
\\
\midrule

\textbf{Task}
&
\textbf{Meta$^n$}
&
\textbf{Ours}
&
$\boldsymbol{\Delta}$
\\
\midrule

\texttt{aircraft\_landing}
& 91.26
& 91.26
& 0.00
\\

\texttt{assignment\_problem}
& \textbf{53.13}
& 52.86
& -0.28
\\

\texttt{assortment\_problem}
& 45.98
& \textbf{49.33}
& \textbf{+3.35}
\\

\texttt{bin\_packing}
& 94.03
& 94.03
& 0.00
\\

\texttt{capacitated\_warehouse\_location}
& 69.85
& 69.85
& 0.00
\\

\texttt{common\_due\_date\_scheduling}
& 97.40
& 97.40
& 0.00
\\

\midrule
\textbf{Mean}
& 75.27
& \textbf{75.79}
& \textbf{+0.51}
\\
\bottomrule
\end{tabular}

\vspace{0.8em}

\begin{tabular}{lrrr}
\toprule
\multicolumn{4}{l}{
\textbf{(c) DeepSeek-V4.1-Flash}
}
\\
\midrule

\textbf{Task}
&
\textbf{Meta$^n$}
&
\textbf{Ours}
&
$\boldsymbol{\Delta}$
\\
\midrule

\texttt{aircraft\_landing}
& 87.04
& 87.04
& 0.00
\\

\texttt{assignment\_problem}
& 82.92
& \textbf{100.00}
& \textbf{+17.08}
\\

\texttt{assortment\_problem}
& \textbf{44.82}
& 44.51
& -0.30
\\

\texttt{bin\_packing}
& \textbf{95.69}
& 95.68
& -0.02
\\

\texttt{capacitated\_warehouse}
& 68.73
& \textbf{69.76}
& \textbf{+1.03}
\\

\texttt{common\_due\_date}
& 96.25
& \textbf{96.33}
& +0.08
\\

\midrule
\textbf{Mean}
& 79.24
& \textbf{82.22}
& \textbf{+2.98}
\\
\bottomrule
\end{tabular}

\end{table*}

\paragraph{Candidate evolution.}
The final task scores alone do not show whether predictive preservation
actually changes the recursive search trajectory.
We therefore inspect the development-score trajectory for DeepSeek-V4.1-Flash.
The two procedures begin from the same depth-1 state but retain different numbers of candidates and reach different scores.

\begin{table}[t]
\centering
\caption{
\textbf{Candidate evolution for DeepSeek-V4.1-Flash.}
\#Cand. denotes the number of retained candidates at each recursive
depth. Ours changes both the retained candidate set and the resulting
development trajectory.
}
\label{tab:sri_deepseek_candidates}

\small
\setlength{\tabcolsep}{4.8pt}
\renewcommand{\arraystretch}{1.08}

\begin{tabular}{lrrrrr}
\toprule
&
\multicolumn{2}{c}{\textbf{Meta$^n$}}
&
\multicolumn{2}{c}{\textbf{Ours}}
&
\\
\cmidrule(lr){2-3}
\cmidrule(lr){4-5}

\textbf{Depth}
&
\textbf{\#Cand.}
&
\textbf{Dev}
&
\textbf{\#Cand.}
&
\textbf{Dev}
&
$\boldsymbol{\Delta}$
\\
\midrule

d1
& 1 & 0.6845
& 1 & 0.6845
& 0.0000
\\

d2 (merge)
& 5 & 0.8073
& 5 & \textbf{0.8605}
& \textbf{+0.0532}
\\

d3
& 6 & 0.8073
& 14 & \textbf{0.8546}
& \textbf{+0.0473}
\\

d4
& 14 & 0.8037
& 4 & \textbf{0.8605}
& \textbf{+0.0568}
\\

d5
& -- & --
& 2 & 0.8546
& --
\\

\bottomrule
\end{tabular}
\end{table}

\paragraph{Qualitative evolution of solver strategies.}
To examine how the recursive context changes subsequent model responses,
we trace the gemini-3.1-pro recursive-search run across generations.
Table~\ref{tab:sri_strategy_evolution} summarizes the main solver
strategies proposed by the model.
The clearest trajectory occurs on \texttt{aircraft\_landing}, where
the model progressively separates discrete sequence search from
continuous landing-time optimization.

\begin{table*}[t]
\centering
\caption{
\textbf{Evolution of model-generated solver strategies in
gemini-3.1-pro.}
Each row summarizes the main strategy proposed at one generation and
the model's rationale for the subsequent change.
The deployed aircraft solver ultimately retains the Gen.~5 strategy:
simulated annealing over landing sequences followed by exact LP
refinement of the top candidate sequences.
}
\label{tab:sri_strategy_evolution}

\scriptsize
\setlength{\tabcolsep}{3.7pt}
\renewcommand{\arraystretch}{1.12}

\begin{tabular}{
p{0.043\textwidth}
p{0.245\textwidth}
p{0.315\textwidth}
p{0.315\textwidth}
}
\toprule

\textbf{Gen.}
&
\textbf{Meta$^n$ response}
&
\textbf{Ours response}
&
\textbf{Reason for the change}
\\
\midrule

1
&
Accelerate assignment solving and refine assortment packing
through restarts and stock selection.
&
Use \texttt{linear\_sum\_assignment} for assignment;
diversify assortment ordering and stock selection.
&
Runtime is the main limitation for assignment, while assortment
is sensitive to fixed ordering and stock choices.
\\

2
&
Continue assortment refinement using stronger packing heuristics
and stock combinations; one branch explores warehouse search.
&
Remove restrictive assortment choices, improve optional-piece
handling, and introduce MaxRects-style packing.
&
Earlier generated solvers do not fully realize the intended packing
strategy, leaving implementation-level restrictions.
\\

3
&
Continue assortment improvement with randomized ordering,
GRASP-style search, and richer placement heuristics.
&
Continue assortment refinement, introduce local search for warehouse,
and apply simulated annealing (SA) to aircraft landing sequences.
&
After assortment improves, the model identifies additional headroom
in warehouse and aircraft optimization.
\\

4
&
Further refine assortment using SA/hill climbing and additional
placement and stock-selection heuristics.
&
Strengthen warehouse subset search and add exact transportation-LP
evaluation; continue improving assortment packing helpers.
&
Heuristic warehouse evaluation can mis-rank candidate subsets,
motivating exact final evaluation.
\\

5
&
Continue assortment-oriented search and packing refinement.
&
For \texttt{aircraft\_landing}, retain SA for sequence search,
keep the top 10 candidate sequences, and optimize landing times
with \texttt{scipy.optimize.linprog}.
&
The model identifies greedy landing-time assignment as the main
remaining bottleneck: promising sequences can be mis-scored because
early/late penalties are not optimized globally.
\\

6
&
Continue assortment refinement and implementation cleanup.
&
Explore wider aircraft candidate funnels and exact-LP local
refinement, while continuing assortment helper integration.
&
After the Gen.~5 aircraft improvement, the model attributes the
remaining limitation to candidate pre-selection before exact
evaluation.
\\

\bottomrule
\end{tabular}
\end{table*}

The aircraft trajectory illustrates a progressive solver refinement.
The model first strengthens discrete search over landing sequences,
then identifies greedy landing-time assignment as the remaining
bottleneck, and finally introduces exact LP refinement over the best
SA candidates.
Although later responses explore wider candidate funnels, the selected
deployed solver retains the Gen.~5 SA--plus--LP strategy.

\paragraph{Held-out evaluation variability.}
Each main-table SRI held-out score is the mean over three matched evaluations
of the same frozen strategy. As a representative example, the three
held-out evaluations of the frozen Ours gemini-3.1-pro aircraft strategy score
$0.9065$, $0.9259$, and $0.9378$, yielding a mean of \textbf{0.9234}.
The solver code, prompt, task instances, and evaluation settings are fixed
across repeats; these repetitions quantify evaluation variability of one
recursive-search outcome rather than independent search seeds.

\FloatBarrier

\clearpage
\section{Implementation, Ablations, and Reproducibility}
\label{app:implementation}

This appendix describes how the common predictive-preservation framework is
instantiated in each host system, followed by controlled ablations and
hyperparameter sensitivity. Dataset construction, split rules, and evaluation
protocols are kept in Appendix~\ref{app:experimental_details}; complete
results and diagnostic analyses are kept in
Appendix~\ref{app:additional_results}.

\subsection{Temporal-Graph Implementation}
\label{app:impl_tg}

\paragraph{Host recursion and recursive interfaces.}
We use TGN as the common host for UCI temporal link prediction and
Wikipedia node classification. For a temporal query with propagation depth
$L$, TGN recursively constructs node representations from temporally sampled
neighborhoods. At each propagation layer, information collected from the
current temporal subtree is aggregated into an intermediate representation
and then reused by the parent recursion. Consequently, signals originating
deeper in the temporal neighborhood must pass through more successive
recursive interfaces before reaching the task prediction. These reused
intermediate representations are the recursive interfaces at which predictive
preservation is applied: information needed by later prediction must survive
the current aggregation before it can propagate further toward the root.

\paragraph{Method insertion.}
We instantiate the recursive interface at each internal propagation layer
$0<\ell<L$. TGN first computes its original temporal aggregation at layer
$\ell$. The compression map $\Gamma$ then receives the layer-specific own
state together with this vanilla aggregate and produces the representation
passed to the parent recursion. The leaf layer ($\ell=0$) retains the raw
host state, and the task-facing root ($\ell=L$) retains the original TGN
aggregation operator. The method therefore acts precisely on intermediate
representations that will subsequently be reused at another recursive interface after a further round of temporal
aggregation.

This construction is shared by both temporal-graph tasks. The host temporal
memory, neighborhood construction, temporal aggregation procedure, and task
pipeline are retained; the distinction between UCI and Wikipedia remains
the original benchmark task and evaluation protocol.

\paragraph{Predictive records and predictive feasibility optimization.}
During the same recursive query, the host adapter traces the eligible
internal interfaces together with the child--parent relationships actually
consumed by the TGN computation. These traces provide the aligned local predictive records used to estimate
the finite-test predictive value and empirical predictive deficiency introduced
in Section~\ref{sec:predictive-compression-learning}. The resulting predictive
sensitivities are used by the common constrained-update procedure, with the
predictive correction restricted to the compression parameter block
$\Gamma$. All remaining trainable TGN parameters continue to follow the host task
optimization.

The statistical configuration, hyperparameter selection, and optimization
settings associated with the predictive-preservation procedure are reported
jointly with the corresponding method-level configurations below.

\subsection{Language-Memory Implementation} 
\label{app:impl_language} 
 
\paragraph{Host and recursive memory.} 
We use CCM as the host for the language-memory experiments. Under its 
\texttt{merge\_recur} construction, dialogue history is compressed 
progressively rather than presented in full at every prediction step. 
After processing the first $t$ history turns, CCM aggregates the 
corresponding compression-token states into a compact memory $M_t$, 
which is then reused when the next turn is incorporated. The resulting 
computation therefore forms a chain 
\[ 
M_1 \rightarrow M_2 \rightarrow \cdots \rightarrow M_L 
\rightarrow \text{response prediction}. 
\] 
The language-model backbone, dialogue objective, and decoding procedure 
remain those of the CCM host. 
 
\paragraph{Recursive interface and method insertion.} 
This recursive memory update instantiates the recursive interface used by 
our method. Information from an early dialogue turn can influence the 
final response only if it survives the intermediate memory states through 
which the history is repeatedly propagated. In the notation of the main 
method, the accumulated information available before the current memory 
update plays the role of the pre-compression state $U_t$, while the memory 
reused by subsequent CCM computation is the compressed representation 
associated with $Z_t^\Gamma$. 
 
Let $\bar M_t$ denote the memory obtained from the native CCM merge at 
turn $t$. We retain this host computation and realize the method-specific compression 
map through the residual $R_\Gamma$ at the recursive merge interface: 
\[ 
M_t^\Gamma 
= 
\bar M_t 
+ 
R_\Gamma(M_{t-1}^\Gamma,h_t,t), 
\] 
where $h_t$ is the current compressed-turn representation. The residual 
map is initialized with zero output, so the initial computation coincides 
with the corresponding CCM host. The resulting $M_t^\Gamma$ is then 
consumed by the next CCM update in the same position as the original 
memory state. 
 
\paragraph{Predictive observations.} 
For a dialogue 
\[ 
[u_1,\ldots,u_L,c,y], 
\] 
we construct predictive observations at each eligible recursive memory 
interface. For an intermediate cut, the predictive observations are drawn 
from utterances that remain downstream of that memory state in the same 
dialogue trajectory, combining a nearby future event with the eventual 
response target. At the terminal cut, the remaining context $c$ and 
response $y$ provide the final predictive observation. Thus, every predictive 
observation is formed from information that is genuinely unavailable before 
the corresponding compression event and subsequently contributes to the same 
dialogue prediction. 
 
A fixed readout of $M_t^\Gamma$ provides the compressed representation used 
to estimate the finite-test predictive value. The readout carries no trainable 
parameters and serves only to evaluate predictive preservation at the recursive 
interface. 
 
\paragraph{Training integration.} 
The same predictive-preservation criterion is evaluated at these recursive 
interfaces, while language-memory optimization uses the aggregate predictive 
update. This is a specialization of the same training principle in which 
predictive preservation is enforced directly on $\Gamma$, rather than through 
conflict-triggered projection of the host update. The task gradient updates the 
trainable CCM compression parameters together with $\Gamma$, whereas the 
predictive gradient is restricted to $\Gamma$; the base language-model parameters 
remain frozen. Consequently, the intervention changes how predictive information 
is preserved across recursive memory updates without changing the dialogue task 
or evaluation protocol. At inference, the learned $\Gamma$ residual remains part 
of the recursive memory update, while the training-time observer and predictive 
statistic are absent. Optimization coefficients and formal training configurations 
are reported separately with the shared hyperparameter settings.
\subsection{Vision--Language--Action Implementation}
\label{app:impl_vla}

\paragraph{Host and recursive memory.}
We instantiate the method on MemoryVLA for the LIBERO-Mem experiments.
MemoryVLA maintains a compact perceptual--cognitive memory that is updated
and reused throughout sequential visuomotor execution. At each stage, the
current visual observation, language instruction, and previously retained
memory are integrated to form the working state used for subsequent memory
updates and action generation. The resulting computation therefore contains
the same recursive structure as the preceding domains: information acquired
at an earlier stage must survive a sequence of memory updates before it can
influence later actions.

\paragraph{Recursive interface and method insertion.}
The recursive working-memory update instantiates the recursive interface used
by our method. Let $U_v$ denote the working representation available before
the memory compression at stage $v$. We apply the compression map at this
interface,
\[
Z_v^\Gamma=\Gamma(U_v),
\]
and feed the resulting compact memory back into the original MemoryVLA
computation. This memory is subsequently reused by later perception,
memory-update, and action-generation stages. Hence, information discarded
at one interface cannot contribute to downstream control decisions, making
this recursive memory state the recursive interface at which predictive
preservation is applied.

The intervention is restricted to this internal recursive interface. The
original visual and language encoders, action-generation pathway, task
objective, and sequential execution procedure are otherwise retained.

\paragraph{Predictive observations.}
Predictive observations are constructed from future events along the same
demonstration trajectory. For each eligible recursive memory interface, the compressed
state is paired with recorded downstream actions that occur after that
memory state has been formed. These future actions instantiate the local
predictive events used by the empirical predictive criterion:
a memory state is favored when it retains information useful for later control.

The future trajectory is used only to construct training-time predictive
supervision. It is never provided to the working memory, compression map,
or policy input, and test-time rollouts follow causal MemoryVLA execution.

\paragraph{Training integration.}
The resulting interface-wise predictive records instantiate the same
predictive-preservation criterion used for temporal graphs and language memory. The
task loss continues to optimize the original VLA objective, while the
predictive-preservation signal constrains updates to the designated memory
compression block. Thus, the method changes how information is retained
across successive memory updates while leaving the host task and downstream
control computation intact.

The reported LIBERO-Mem experiments use the T3 setting throughout. The
predictive-preservation construction is applied unchanged across the repeated
stages of this multi-stage manipulation trajectory. Formal runs use dense
frame-level training with batch size 4 and gradient accumulation 4 (effective
batch 16), a constant schedule with zero warmup, a task-LR argument of
$10^{-5}$ and a Gamma learning rate of $2\times10^{-5}$. The host is frozen in
the Gamma-only stage, so only the Gamma optimizer updates method-specific
parameters. All formal VLA runs are capped at 25k optimizer steps and save a
checkpoint every 1k steps. Checkpoint selection and final held-out rollout
evaluation use disjoint data as specified in Appendix~\ref{app:vla}; the
official evaluation rollouts are not reused for model selection.


\subsection{Self-Recursive Improvement Implementation} 
\label{app:impl_sri} 
 
\paragraph{Host recursion and recursive interfaces.} 
We instantiate the method on Meta$^n$, where recursive improvement proceeds 
by repeatedly reusing execution experience from earlier search rounds. 
Layer-1 candidates produce task-solving programs and execution traces; at 
subsequent depths, the operator $\Omega$ consumes retained traces together 
with previously injected code to generate or refine the next candidate. 
The resulting candidate is evaluated and may contribute further experience 
to later recursive rounds. 
 
Before every \texttt{generate} or \texttt{refine} call, Meta$^n$ reduces the 
available historical context to satisfy a fixed context budget. This 
context-reduction step instantiates the recursive interface. Earlier 
heuristics, execution outcomes, failure modes, and injected repair routines 
can influence deeper search only if they survive this reduction. Repeated selection therefore creates the same preservation problem as repeated state compression in the preceding domains. 
 
\paragraph{Offline lineage construction.} 
Unlike the other hosts, Meta$^n$ does not expose a differentiable recursive 
state together with its future observations during a single forward pass. 
We therefore construct the predictive supervision from completed host 
trajectories. We first execute Meta$^n$ with its native context reduction 
and freeze the resulting archives. From each archive, we recover only 
lineages that actually occur in the recursive search, 
\[ 
c_v \longrightarrow c_{v+1} \longrightarrow c_{v+2}. 
\] 
Each candidate $c_v$ defines one recursive interface. Its pre-compression 
state is 
\[ 
U_v = 
\begin{cases} 
(\text{traces},\,\text{context stack}), 
    & \texttt{generate},\\[2mm] 
(\text{child traces},\,\text{context stack} 
    + \text{previous injection}), 
    & \texttt{refine}. 
\end{cases} 
\] 
The remaining task-level prompt information is retained as $C_v$ and is 
passed through unchanged. 
 
Only real child and grandchild states are used to construct future 
observations. For an eligible lineage, we define 
\[ 
Y_{v+1}=\Phi(T_{v+1}), 
\qquad 
Y_{v+2}=\Phi(T_{v+2}), 
\] 
where $T_{v+1}$ and $T_{v+2}$ are the execution traces produced by the 
actual child and grandchild candidates. Interfaces without an observed 
grandchild are excluded from the predictive record. 
 
\paragraph{Historical-trace compression and retention.} 
The items contained in $U_v$ are encoded by a frozen encoder $E_0$, 
\[ 
X_v= 
\bigl[ 
E_0(x_1),\ldots,E_0(x_{n_v}) 
\bigr], 
\qquad 
q_v=E_0(\text{current traces}), 
\] 
and passed to the learned compression map $\Gamma_\theta$. In the 
implementation, $\Gamma_\theta$ uses a query-conditioned slotted fusion: 
\[ 
Q_v = Q_0 + B_cA_c q_v, 
\qquad 
K=X_vW_K,\quad 
V=X_vW_V,\quad 
g=\sigma(X_vw_g), 
\] 
\[ 
A_v 
= 
\operatorname{softmax} 
\left( 
\frac{Q_vK^\top}{\sqrt{d_e}} 
+ 
\mathbf 1\log(g+\epsilon)^\top 
+ 
M_v 
\right), 
\] 
\[ 
Z_v 
= 
\operatorname{LN}(Q_v+A_vV). 
\] 
Here $Z_v$ is the compact representation used by the predictive statistic, 
while $A_v$ determines which host-produced historical items are retained. 
This is the discrete-retention specialization of the same predictive-
preservation principle used for continuous recursive states: rather than 
controlling which components survive in a continuous latent representation, 
$\Gamma_\theta$ controls which history items survive the capacity-limited 
recursive interface. The reducer returns the original \texttt{Trace} and 
\texttt{InjectedCode} objects selected under the same context budget as 
Meta$^n$; it does not synthesize additional experience or modify the 
downstream prompt construction. 
 
\paragraph{Predictive supervision and offline training.} 
The future side of each training record is obtained from the same recovered 
lineage. We encode the real child and grandchild observations as 
\[ 
f_v^{(1)} 
= 
E_0(\operatorname{serialize}(Y_{v+1})), 
\qquad 
f_v^{(2)} 
= 
E_0(\operatorname{serialize}(Y_{v+2})), 
\] 
and form fixed predictive tests 
\[ 
p_v^{(r)} 
= 
R_S^{(r)} 
\begin{bmatrix} 
f_v^{(1)}\\ 
f_v^{(2)} 
\end{bmatrix}. 
\] 
We additionally record the realized two-step improvement 
\[ 
R_v 
= 
\operatorname{score}(c_{v+2}) 
- 
\operatorname{score}(c_v). 
\] 
After projecting the compressed state to the statistic space, 
\[ 
b_v=P_b\,\operatorname{vec}(Z_v), 
\] 
the predictive and improvement statistics are 
\[ 
J_{\mathrm{LPSE}} 
= 
\frac{1}{4}\sum_{r=1}^{4} 
J_{\mathrm{CCA}}(B,P^{(r)}), 
\qquad 
J_{\mathrm{task}} 
= 
J_{\mathrm{CCA}}(B,R). 
\] 
Their sensitivities instantiate the same predictive-preservation training 
principle used elsewhere in the paper: the task objective favors histories 
associated with subsequent recursive improvement, while the predictive-preservation 
term discourages discarding information that remains predictive across the 
recovered interfaces. 
 
Training of $\Gamma_\theta$ is performed entirely offline on the frozen 
archive records. The language model and item encoder remain fixed, and no 
additional $\Omega$ calls are made during this optimization. 
 
\paragraph{Evaluation-time integration.} 
After offline training, $\Gamma_\theta$ is frozen and inserted only at the 
two native context-reduction points preceding \texttt{generate} and 
\texttt{refine}. Formal evaluation uses only the information available at 
the current recursive state; child and grandchild observations are used 
solely for constructing the offline training records and are never accessed 
during evaluation. 
 
The Meta$^n$ archive, parent-selection and acceptance rules, recursive 
candidate-generation procedure, $\Omega$ prompt and output schema, code 
executor, and evaluator remain unchanged. Meta$^n$ and Ours use the same 
context budget and recursive-search budget. Thus, SRI does not introduce a 
different preservation objective; it realizes the same interface-level 
principle in a discrete setting, where compression takes the form of deciding 
which components of recursive history survive a capacity-limited context 
interface so that they remain useful for later improvement.

\FloatBarrier

\subsection{Training and Compute Configuration}
\label{app:training-compute}

Table~\ref{tab:training_compute} summarizes the run-level configurations used
for the formal experiments. Method-specific statistical and feasibility
hyperparameters are reported separately in Appendix~\ref{app:hyperparameters}.
The table records the actual experimental configuration rather than host
repository defaults.

\begin{center}
\footnotesize
\captionof{table}{\textbf{Training and compute configuration used in the formal experiments.}}
\label{tab:training_compute}
\setlength{\tabcolsep}{3.2pt}
\renewcommand{\arraystretch}{1.10}
\begin{tabular}{@{}p{2.45cm}p{3.00cm}p{3.15cm}p{3.15cm}p{3.05cm}@{}}
\toprule
\textbf{Configuration} & \textbf{Temporal Graph} & \textbf{Language Memory} & \textbf{VLA} & \textbf{SRI} \\
\midrule
Host & TGN & CCM-merge & MemoryVLA & Meta$^n$ \\
Formal repetitions & 5 random seeds & fixed CCM-merge host + 3 Ours seeds & 3 independently trained runs & 1 paired search run / backbone \\
Optimization & Adam & \makecell[l]{AdamW; base LM frozen;\\ task gradient on CCM comp. + $\Gamma$;\\ predictive gradient on $\Gamma$} & AdamW, frozen host + $\Gamma$ & API-backed recursive search \\
Batch / accumulation & 200 & adaptive optimization windows & $4\times4$ (effective 16) & -- \\
Primary learning rate & task $10^{-4}$; repr. $3\times10^{-4}$ (UCI formal line) & frozen formal CCM schedule & task-LR argument $10^{-5}$; $\Gamma$ LR $2\times10^{-5}$ & -- \\
Schedule & validation-selected & cosine host schedule & constant; warmup 0 & fixed search budget \\
Maximum budget & setting-specific formal epoch budget & 60 optimization windows & 25k optimizer steps & fixed recursive-search budget \\
Checkpoint cadence & validation evaluated during training & CCM-M fixed once; Ours validation-only & every 1k steps; internal dev selection & frozen strategy after search \\
Evaluation repeats & split-level evaluation & split-level evaluation & 20 held-out T3 rollouts per selected checkpoint & 3 matched held-out evaluations \\
Compute & RTX 5090D & A100 80GB & A100 80GB & API-backed \\
Software & Python $\geq3.9$; PyTorch 2.8 environment & Python $\geq3.9$; PyTorch 2.8 environment & Python 3.10; PyTorch 2.2.0; Transformers 4.40.1; PEFT 0.11.1 & Python $\geq3.11$ \\
\bottomrule
\end{tabular}
\end{center}

For VLA, the 25k budget is an upper bound rather than a requirement to report
the endpoint. Checkpoints are saved every 1k steps and selected using only the
20-demonstration internal development split described in
Appendix~\ref{app:vla}. After selection, the checkpoint is frozen and evaluated
on a separate set of 20 official held-out T3 evaluation initializations. The reported SR/SGC rollouts are used for neither training nor checkpoint selection.

\FloatBarrier
\subsection{Efficiency and Model Overhead}
\label{app:efficiency}

We measure computational overhead on the TGN host, where the training and
inference paths can be timed under matched structural settings. All timings in
Tables~\ref{tab:tgn_training_efficiency} and~\ref{tab:tgn_efficiency} use the
same RTX 5090D environment within each comparison. The predictive-feasibility
machinery is active only during training; inference uses the ordinary host
recursion with the learned compressor and performs no feasibility projection.

\begin{table}[H]
\centering
\caption{\textbf{Training-time overhead on TGN.} Mean epoch time under matched settings.}
\label{tab:tgn_training_efficiency}
\small
\setlength{\tabcolsep}{6pt}
\begin{tabular}{lccc}
\toprule
\textbf{Setting} & \textbf{Vanilla} & \textbf{Ours} & \textbf{Ratio} \\
\midrule
3L10N & 87.1 s & 165.7 s & 1.9$\times$ \\
4L5N  & 88.6 s & 190.6 s & 2.2$\times$ \\
\bottomrule
\end{tabular}
\end{table}

\begin{table}[H]
\centering
\caption{\textbf{Inference-time comparison on TGN.} The projection is disabled at inference.}
\label{tab:tgn_efficiency}
\small
\setlength{\tabcolsep}{5.2pt}
\begin{tabular}{lrrrr}
\toprule
\textbf{Setting} & \textbf{Vanilla eval} & \textbf{Ours eval} & \textbf{Vanilla build} & \textbf{Ours build} \\
\midrule
3L10N & 68.1 s & 66.0 s & 1.6 s & 0.6 s \\
4L5N  & 86.7 s & 87.6 s & 0.5 s & 0.7 s \\
\bottomrule
\end{tabular}
\end{table}

The extra computation is therefore concentrated in training: epoch time is
$1.9$--$2.2\times$ the vanilla host in these settings. Test-time evaluation stays within about $\pm3\%$ of vanilla; inference uses no projection or replay.

\paragraph{Parameter overhead.}
The vanilla TGN host contains 3,774,385 unique parameters (3.77M). The
compressor adds 381,788 parameters (0.38M), giving approximately 4.16M unique
parameters in total, or a 10.1\% increase relative to the host. The serialized
checkpoint contains duplicated references to the compressor and therefore has a
larger stored parameter count; this duplication does not increase the unique
model parameterization used at runtime.

\FloatBarrier
\subsection{Ablation Study}
\label{app:ablation}

We further examine which components of the proposed training procedure contribute
to predictive preservation. All ablations are conducted on UCI link prediction
at propagation depth 3 using seed 0. Each configuration is run once under the
same data split, host architecture, training budget, and evaluation protocol.
Accordingly, this experiment is intended as a controlled mechanism-level
diagnosis of the full method rather than an estimate of cross-seed variability.
The corresponding multi-seed results for the main comparison are reported
separately.

\begin{table}[t]
\vspace{-0.4em}
\centering
\caption{
\textbf{Ablation study on UCI at depth 3 (seed 0).}
Each variant changes one component of the full training procedure while keeping
the remaining experimental protocol fixed.
}
\label{tab:ablation_uci_seed0}
\vspace{-0.4em}

\small
\setlength{\tabcolsep}{5.5pt}
\renewcommand{\arraystretch}{1.06}

\begin{tabular}{p{1.45cm}p{5.55cm}cc}
\toprule
\textbf{Component}
& \textbf{Ablation variant}
& \textbf{AP}
& \textbf{AUC}
\\
\midrule

\textbf{Objective}
& Task-Only Training
& 92.62 & 92.13
\\

& History-reconstruction objective
& 90.89 & 90.28
\\

& Direct Future Prediction
& 92.76 & 92.24
\\

\addlinespace[2pt]

\textbf{Supervision}
& Single-Event Predictive Criterion
& 92.41 & 91.93
\\

& Cross-level mispaired supervision
& 92.73 & 92.28
\\

& Predictive supervision without context
& 92.89 & 92.50
\\

\addlinespace[2pt]

\textbf{Statistics}
& Tree-specific statistical estimation
& 91.48 & 90.99
\\

& Unbalanced predictive-test aggregation
& 91.30 & 90.98
\\

& Diagonal-covariance approximation
& 91.51 & 91.04
\\

\addlinespace[2pt]

\textbf{Optimization}
& Additive Predictive-Gradient Update
& 92.64 & 92.18
\\

& Global Predictive Constraint
& 90.32 & 89.88
\\

\midrule

\textbf{Full method}
& \textbf{Our Method}
& \textbf{94.52}
& \textbf{94.25}
\\

\bottomrule
\end{tabular}

\vspace{-0.5em}
\end{table}

\paragraph{Does the choice of preservation objective matter?}
Yes. Task-Only Training reaches 92.62 AP and 92.13 AUC, providing the
reference performance when no explicit predictive-preservation objective is
introduced. Replacing predictive preservation with history reconstruction
reduces performance to 90.89 AP and 90.28 AUC, indicating that retaining
information sufficient to reconstruct the past does not necessarily retain the
information useful for subsequent prediction. Direct Future Prediction performs
better at 92.76 AP and 92.24 AUC, but still remains below the full method at
94.52 AP and 94.25 AUC. This suggests that predictive targets are more suitable
than reconstruction targets, while directly fitting future outcomes does not recover the full predictive criterion's benefit.

\paragraph{Does the structure of predictive supervision matter?}
The results indicate that it does. The Single-Event Predictive Criterion
reaches 92.41 AP and 91.93 AUC. Replacing the adjacent predictive pair
$S_v=(Y_v,Y_{\operatorname{pa}(v)})$ with a single predictive event weakens
the local supervision supplied at each recursive interface. Cross-level
mispaired supervision reaches 92.73 AP and 92.28 AUC. Its degradation relative
to the full method indicates that the correspondence between a compressed state
and its associated predictive event is important, even when comparable
predictive events are available elsewhere in the recursion. Removing the
retained context gives 92.89 AP and 92.50 AUC. The remaining gap to the full
method is consistent with the role of context in defining the conditional
predictive information that should survive each recursive interface.

\paragraph{Does the statistical realization of the predictive criterion matter?}
All three statistical simplifications lead to a larger reduction. Estimating
the predictive statistics separately for each tree obtains 91.48 AP and
90.99 AUC, suggesting that the shared empirical estimate benefits from pooling
aligned observations across recursive instances. Removing balancing gives
91.30 AP and 90.98 AUC, supporting the use of balanced predictive tests when
aggregating discrepancies across outcomes or interfaces. Using a
diagonal-covariance approximation gives 91.51 AP and 91.04 AUC, indicating that
cross-dimensional dependence in the empirical statistic contains useful
information. Together, these results show that implementing the predictive criterion matters beyond the supervision target.

\paragraph{Does constrained optimization contribute beyond the predictive objective?}
Yes. The Additive Predictive-Gradient Update yields 92.64 AP and
92.18 AUC, below the 94.52 AP and 94.25 AUC of the full method. This separates
the predictive objective from its predictive-feasibility realization: using the
same supervision as an additive gradient does not provide the same protection
against task proposals that violate predictive preservation. Replacing the
interface-wise constraints with the Global Predictive Constraint causes the
largest drop, to 90.32 AP and 89.88 AUC. The result is consistent with the
recursive structure of the problem: predictive loss can arise at individual
recursive interfaces, and an aggregate global restriction does not provide the
same local control over these violations.

\paragraph{Overall.}
Across the four groups, the full method achieves the highest AP and AUC in this
fixed seed-0 run. The ablations jointly support four aspects of the design:
predictive rather than reconstructive preservation, correctly aligned
interface-level supervision, the balanced empirical statistical realization,
and interface-wise constrained optimization. Since the experiment contains one
run per configuration, these comparisons are used to diagnose the internal
mechanism of the method; claims about cross-run variability are based on the
separate multi-seed evaluation.

\FloatBarrier

\subsection{Method Hyperparameters and Sensitivity} 
\label{app:hyperparameters} 
\label{app:robustness-efficiency} 
 
\paragraph{Method hyperparameters.} 
Table~\ref{tab:method_hyperparameters} summarizes the principal 
hyperparameters introduced by the predictive-preservation procedure across 
the four experimental domains. We report the frozen configurations used for 
the formal experiments rather than command-line defaults. Parameters that 
belong only to the underlying host models follow the corresponding host 
protocols and are omitted here. 
 
\begin{table*}[t] 
\centering 
\caption{ 
\textbf{Method-specific hyperparameters used in the formal experiments.} 
Only parameters introduced by the predictive-preservation procedure or its 
statistical realization are reported. 
} 
\label{tab:method_hyperparameters} 
 
\footnotesize 
\setlength{\tabcolsep}{3.0pt} 
\renewcommand{\arraystretch}{1.10} 
 
\begin{tabular}{@{}p{2.55cm}p{3.05cm}p{3.05cm}p{2.85cm}p{2.85cm}@{}} 
\toprule 
\textbf{Configuration} 
& \textbf{Temporal Graph} 
& \textbf{Language Memory} 
& \textbf{VLA} 
& \textbf{SRI} \\ 
\midrule 
 
\multicolumn{5}{@{}l}{\textit{Predictive statistic}} \\ 
 
Statistic 
& Full balancing 
& Full balancing + OAS 
& Full dual 
& CCA / OAS \\ 
 
Sketch dimension 
& $64$ 
& $4\times32$ 
& $64$ 
& $4\times8$ \\ 
 
Statistic width 
& $128$ 
& $128$ 
& -- 
& $32$ \\ 
 
Context / future width 
& $32/32$ 
& Fixed sketches 
& $64/64$ 
& $256$ encoder \\ 
 
Covariance ridge 
& $10^{-3}$ (UCI); $10^{-4}$ (Wiki) 
& $10^{-3}$ 
& $10^{-4}$ 
& $10^{-3}$ \\ 
 
Minimum statistical support 
& $896/1024$ trees 
& $128$ cuts 
& $128$ merges 
& $32$ trees/window \\ 
 
Predictive observations 
& $2$ 
& $2$ 
& $2$ 
& Child + grandchild \\ 
 
Fixed-map seed 
& $0$ 
& $0$ 
& $0$ 
& $0$ \\ 
 
\midrule 
\multicolumn{5}{@{}l}{\textit{Compression / actuation}} \\ 
 
$\Gamma$ capacity 
& width $128$ 
& hidden $64$ 
& rank $64$ 
& $d_e=256,\ h=64$ \\ 
 
Additional structure 
& -- 
& time dim.\ $16$ 
& -- 
& $4$ slots; cond.\ rank $32$ \\ 
 
Initialization 
& Host-aligned 
& Zero-output 
& Host-merge aligned 
& Fixed seeded init. \\ 
 
Predictive update scope 
& $\Gamma$ 
& $\Gamma$ 
& $\Gamma$ 
& $\Gamma$ \\ 
 
\midrule 
\multicolumn{5}{@{}l}{\textit{Optimization / feasibility}} \\ 
 
Feasibility tolerance $\kappa$ 
& $0.05$ 
& -- (direct aggregate specialization) 
& $0.05$ 
& $0.02$ \\ 
 
Certificate tolerance 
& $10^{-6}$ 
& -- (direct aggregate specialization) 
& $10^{-3}$ 
& $3\times10^{-4}$ \\ 
 
Predictive coefficient $\lambda$ 
& -- (projection mode) 
& $0.00594$ (LLaMA/Gemma) 
& -- (projection mode) 
& -- \\ 
 
\bottomrule 
\end{tabular} 
\end{table*} 
 
\paragraph{Hyperparameter selection.} 
The fixed measurement dimensions, statistical regularization, support 
thresholds, and compression capacities are selected using only the 
corresponding development or validation data and are frozen before formal 
evaluation. No method-specific hyperparameter is selected using the reported 
test results. The formal TGN and VLA runs use task-only optimizer proposals 
before predictive feasibility projection and therefore do not use a predictive 
coefficient $\lambda$; the UCI Additive Predictive-Gradient Update ablation uses 
$\lambda=0.00668$. The language-memory aggregate specialization uses 
$\lambda=0.00594$. The feasibility tolerance $\kappa$ directly controls the 
admissible predictive degradation of an update and is therefore treated as the 
principal geometric hyperparameter. We additionally examine the amount of statistical support 
through the eligible-interface fraction. 
 
\paragraph{Sensitivity analysis.} 
Figure~\ref{fig:hyperparameter-sensitivity} reports sensitivity on UCI as a 
representative projection-based setting in which both the feasibility 
mechanism and empirical predictive statistic are active. We vary 
\[ 
\kappa\in\{0.01,0.02,0.05,0.10,0.20\}, 
\] 
with all remaining settings fixed, and separately vary the fraction of 
eligible recursive interfaces used to estimate the empirical predictive 
statistic. Performance remains stable around the formal operating region, 
while substantially tighter or looser feasibility and reduced statistical 
support lead to lower performance. 
 
\begin{figure}[H] 
    \centering 
    \includegraphics[width=\textwidth]{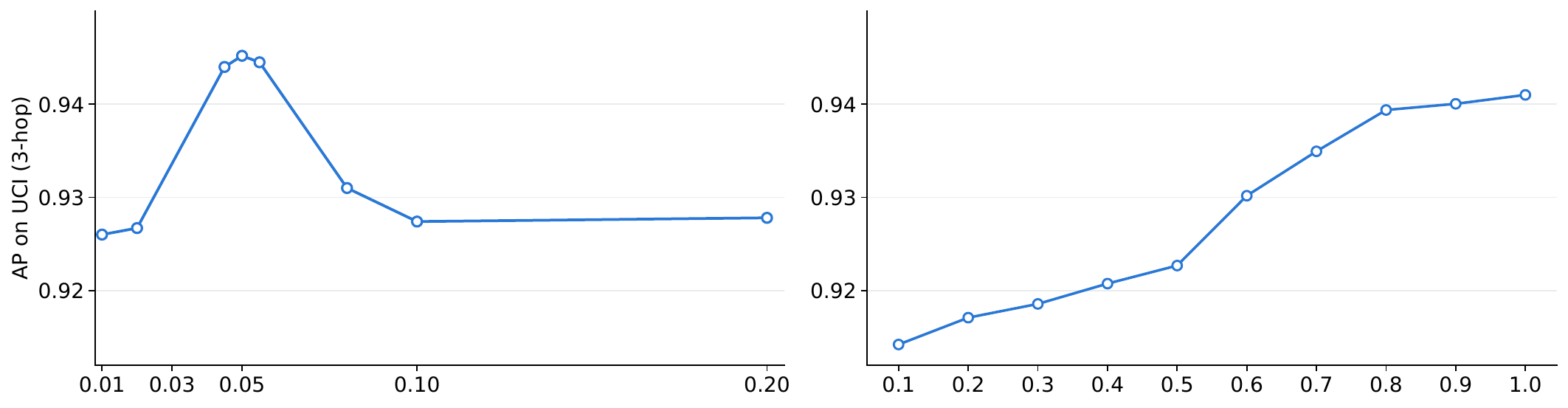} 
     
    \vspace{-0.5em} 
    \begin{minipage}[t]{0.48\linewidth} 
        \centering 
        \small (a) Feasibility tolerance $\kappa$ 
    \end{minipage} 
    \hfill 
    \begin{minipage}[t]{0.48\linewidth} 
        \centering 
        \small (b) Eligible-interface fraction 
    \end{minipage} 
     
    \vspace{-0.3em} 
    \caption{ 
    \textbf{Hyperparameter sensitivity on UCI.} 
    Test AP under variations of 
    (a) the feasibility tolerance $\kappa$ and 
    (b) the eligible-interface fraction used to estimate the empirical predictive 
    statistic. All remaining settings are fixed. 
    } 
    \label{fig:hyperparameter-sensitivity} 
\end{figure} 
 
\FloatBarrier

\clearpage
\section{Theoretical Results and Proofs}
\label{app:theory}

This appendix collects the formal results supporting the main text. Each
subsection is linked to the corresponding proposition, criterion, or guarantee
used in Sections~\ref{sec:problem-formulation} and~\ref{sec:predictive-compression-learning}.

\makeatletter
\@ifundefined{theorem}{\newtheorem{theorem}{Theorem}[section]}{}
\@ifundefined{proof}{%
  \newenvironment{proof}
    {\par\noindent\textit{Proof.}\ }%
    {\hfill\mbox{\rule{0.55em}{0.55em}}\par}%
}{}
\makeatother

\subsection{Proof of Recursive Closure}
\label{app:recursive-closure}

This subsection proves Proposition~\ref{prop:recursive-closure} from Section~\ref{sec:problem-formulation} and states its structural assumptions (M1)--(M3).

This appendix proves Proposition~\ref{prop:recursive-closure}.
The proof is formulated for the \emph{predictive-interface tree}: every
non-root vertex appearing below is associated with a predictive boundary
variable $Y_v$. Additional deterministic host operations between designated predictive
interfaces may be absorbed into the corresponding edge transition. All
population statements condition on the full retained context
$C=(C_v)_{v\in V}$; replacing $C$ by a local context $C_v$ is a separate
localization step and is not used in Proposition~\ref{prop:recursive-closure}.

Let $\mathcal T=(V,E)$ be a finite rooted tree with root $\rho$. Let
$\operatorname{pa}(v)$ denote the parent of each non-root node $v$, and let
$\operatorname{ch}(v)$ denote its children. At interface $v$, let $Y_v$
denote the predictive boundary state and let $U_v$ denote the complete
host pre-compression state. Let
$V_{\mathrm{comp}}\subseteq V\setminus\{\rho\}$ be the set of compressed
interfaces. For a fixed measurable compression map $\Gamma$, define
\[
Z_v^\Gamma:=
\begin{cases}
\Gamma(U_v), & v\in V_{\mathrm{comp}},\\
U_v, & v\notin V_{\mathrm{comp}},
\end{cases}
\]
so that uncompressed nodes are represented by the identity map. Write
\[
U:=(U_v)_{v\neq\rho},
\qquad
Z^\Gamma:=(Z_v^\Gamma)_{v\neq\rho},
\qquad
C:=(C_v)_{v\in V}.
\]
When product state spaces are needed below, write
\[
\mathsf U:=\prod_{v\neq\rho}\mathsf U_v,
\qquad
\mathsf Z:=\prod_{v\neq\rho}\mathsf Z_v.
\]
All random variables take values in standard Borel spaces.

For every non-root node $v$, define the adjacent predictive pair
\[
S_v:=(Y_v,Y_{\operatorname{pa}(v)}),
\]
and define the tuple of incoming compressed child states
\[
R_v^\Gamma
:=
Z_{\operatorname{ch}(v)}^\Gamma
:=
\bigl(Z_u^\Gamma\bigr)_{u\in\operatorname{ch}(v)}.
\]
For a leaf, $R_v^\Gamma$ is the unique element of the singleton empty-product
space. Within the proof we abbreviate $Z_v:=Z_v^\Gamma$ and
$R_v:=R_v^\Gamma$.

For each compressed node $v\in V_{\mathrm{comp}}$, define the local predictive deficiency
\[
\delta_v(\Gamma):=I(S_v;U_v\mid Z_v^\Gamma,C).
\]

\paragraph{Tree-local recursive Markov structure.}
Propositions~\ref{prop:recursive-closure} and
\ref{prop:root-predictive-stability} use the following structural assumptions for the recursive law
induced by the fixed map $\Gamma$.

\emph{(M1) Predictive-boundary Markov tree.}
For $P_C$-almost every $c$, the predictive boundary variables factorize as
\begin{equation}
P(dy_\rho,dy_{-\rho}\mid c)
=
\Pi_\rho(dy_\rho\mid c)
\prod_{v\neq\rho}
Q_v(dy_v\mid y_{\operatorname{pa}(v)},c).
\label{eq:predictive-tree-factorization}
\end{equation}

\emph{(M2) Recursive host aggregation through compressed child states.}
For every non-root node $v$, there is a regular conditional kernel
\[
K_v(du_v\mid s_v,r_v,c)
=
P(du_v\mid S_v=s_v,R_v=r_v,C=c),
\]
and the conditional law of $(U_v,Z_v)$ given $(S_v,R_v,C)$ is
\begin{equation}
\widetilde K_v(du_v,dz_v\mid s_v,r_v,c)
=
K_v(du_v\mid s_v,r_v,c)\,
\delta_{\gamma_v(u_v)}(dz_v),
\label{eq:host-joint-kernel}
\end{equation}
where $\gamma_v=\Gamma$ on $V_{\mathrm{comp}}$ and
$\gamma_v=\operatorname{id}$ otherwise. Conditioned on $(Y_v,C)$, the
child subtrees rooted at distinct children of $v$ are mutually independent.
Moreover, conditional on $(S_v,R_v,C)$, the local pair $(U_v,Z_v)$ is
independent of all remaining variables inside the child subtrees; equivalently,
the child subtrees can affect the host state at $v$ only through their root
compressed tuple $R_v$. After $R_v$ is realized, $(U_v,Z_v)$ is formed
according to Equation~\ref{eq:host-joint-kernel}. At the root, its child
subtrees are mutually independent conditioned on $(Y_\rho,C)$. More precisely,
conditional on $C=c$, the joint law is the finite child-before-parent
composition of $\Pi_\rho$, all $Q_v$, and all $\widetilde K_v$, with $r_v$
equal to the sampled child-root $z$ tuple. This factorization permits
\[
Z_{\operatorname{ch}(v)}^\Gamma\longrightarrow U_v\longrightarrow Z_v^\Gamma.
\]
In particular, no independence between $U_v$ and its incoming compressed
child states is assumed. A deterministic host update of the form
$U_v=F_v(R_v,X_v,C)$ is included as the degenerate-kernel special case,
provided the local host input $X_v$ is part of the state-generation mechanism
represented by $K_v$.

\emph{(M3) Complete-interface screening.}
For every compressed node $v$,
\begin{equation}
I\!\left(
S_v;R_v^\Gamma
\,\middle|\,
U_v,C
\right)=0,
\qquad\text{equivalently}\qquad
S_v\perp R_v^\Gamma\mid(U_v,C).
\label{eq:interface-screening}
\end{equation}
This is the screening-off property of a complete pre-compression interface
state: once $U_v$ and the retained context are known, the incoming compressed
child states carry no additional information about the adjacent predictive
interface $S_v$. Equation~\ref{eq:interface-screening} allows $U_v$ to depend on
$R_v^\Gamma$, including through a deterministic or stochastic aggregation
with local host information, provided the resulting interface law satisfies
the stated screening condition. This condition is additional to (M2). If
$R_v^\Gamma$ is a measurable component of $(U_v,C)$, as when the complete
interface record explicitly retains its incoming child tuple, (M3) holds
automatically.

\begin{proof}[Proof of Proposition~\ref{prop:recursive-closure}]
All regular conditional distributions used below exist because the state
spaces are standard Borel. All conditional-kernel identities are understood
outside the corresponding conditioning null sets. Since the tree is finite,
the finitely many almost-sure identities needed in an induction step can be
chosen simultaneously on a common full-measure set. The recursive kernels
below are versions of the conditional laws under the same induced joint
distribution, so substituting these almost-sure identities into the parent kernel preserves the measure.

\medskip
\noindent\textbf{Step 1: Zero local predictive deficiency and interface screening imply
augmented local closure.}

Fix $v\in V_{\mathrm{comp}}$ and abbreviate
$S_v=(Y_v,Y_{\operatorname{pa}(v)})$, $Z_v=Z_v^\Gamma$, and
$R_v=R_v^\Gamma$. By the hypothesis of Proposition~\ref{prop:recursive-closure},
\begin{equation}
I(S_v;U_v\mid Z_v,C)=0.
\label{eq:local-zero-defect}
\end{equation}
By complete-interface screening,
\begin{equation}
I(S_v;R_v\mid U_v,C)=0.
\label{eq:screening-zero}
\end{equation}
Because $Z_v=\Gamma(U_v)$ is a measurable function of $U_v$ at a compressed
node,
\[
\sigma(U_v,Z_v,C)=\sigma(U_v,C).
\]
Therefore Equation~\ref{eq:screening-zero} is unchanged if $Z_v$ is added to
the conditioning set:
\[
I(S_v;R_v\mid U_v,Z_v,C)=0.
\]
The chain rule for conditional mutual information now gives
\[
\begin{aligned}
I(S_v;U_v,R_v\mid Z_v,C)
&=
I(S_v;U_v\mid Z_v,C)
+I(S_v;R_v\mid U_v,Z_v,C)\\
&=0.
\end{aligned}
\]
Hence
\[
S_v\perp (U_v,R_v)\mid(Z_v,C).
\]
In particular, by decomposition of conditional independence,
\begin{align}
U_v&\perp S_v\mid(Z_v,R_v,C),
\label{eq:local-independence-given-r}\\
R_v&\perp S_v\mid(Z_v,C).
\label{eq:incoming-closed-by-z}
\end{align}
Equation~\ref{eq:local-independence-given-r} is the form needed in the
recursive kernel argument: the local sufficiency condition remains valid
even after the incoming compressed child tuple is exposed. Notice that this
conclusion is derived from the original defect and the structural screening
condition; no additional ``interaction'' or node-coupling defect is assumed.

For completeness, Equation~\ref{eq:local-zero-defect} itself is equivalent
to conditional independence because
\[
\begin{aligned}
I(S_v;U_v\mid Z_v,C)
&=
\mathbb E_{Z_v,C}\!\left[
D_{\mathrm{KL}}\!\left(
P_{S_v,U_v\mid Z_v,C}
\,\middle\|\,
P_{S_v\mid Z_v,C}\otimes P_{U_v\mid Z_v,C}
\right)
\right],
\end{aligned}
\]
and a non-negative Kullback--Leibler divergence has expectation zero only
when it vanishes almost surely.

\medskip
\noindent\textbf{Step 2: Augmented local closure gives a
boundary-independent reconstruction kernel.}

Fix first $v\in V_{\mathrm{comp}}$. Under standard Borel regularity, a
regular conditional distribution of $U_v$ given $(Z_v,R_v,C)$ exists. Define
\[
A_v(B\mid z_v,r_v,c)
:=
P(U_v\in B\mid Z_v=z_v,R_v=r_v,C=c),
\qquad
B\in\mathcal B(\mathsf U_v).
\]
Let $K_v^\Gamma$ be the pushforward of the local host kernel through
$\Gamma$:
\[
K_v^\Gamma(D\mid s_v,r_v,c)
:=
K_v\!\left(\Gamma^{-1}(D)\mid s_v,r_v,c\right),
\qquad
D\in\mathcal B(\mathsf Z_v).
\]
Thus $K_v^\Gamma(\cdot\mid s_v,r_v,c)$ is the conditional law of $Z_v$
given $(S_v,R_v,C)=(s_v,r_v,c)$.

Fix measurable sets $B\subseteq\mathsf U_v$ and
$D\subseteq\mathsf Z_v$. Disintegrating the local joint kernel in
Equation~\ref{eq:host-joint-kernel} with respect to $Z_v$ yields
\[
\begin{aligned}
&\widetilde K_v(B\times D\mid s_v,r_v,c)\\
&\quad=
\int_D
P(U_v\in B\mid S_v=s_v,Z_v=z_v,R_v=r_v,C=c)\,
K_v^\Gamma(dz_v\mid s_v,r_v,c).
\end{aligned}
\]
By Equation~\ref{eq:local-independence-given-r},
\[
P(U_v\in B\mid S_v=s_v,Z_v=z_v,R_v=r_v,C=c)
=
A_v(B\mid z_v,r_v,c)
\]
for $P_{S_v,Z_v,R_v,C}$-almost every $(s_v,z_v,r_v,c)$. Therefore,
for $P_{S_v,R_v,C}$-almost every $(s_v,r_v,c)$,
\[
\widetilde K_v(B\times D\mid s_v,r_v,c)
=
\int_D
A_v(B\mid z_v,r_v,c)\,
K_v^\Gamma(dz_v\mid s_v,r_v,c).
\]
Since standard Borel $\sigma$-algebras are countably generated, the identity
may be chosen on a common full-measure set for a generating class of
rectangles and extended by the monotone class theorem. Hence
\begin{equation}
\boxed{
\widetilde K_v(du_v,dz_v\mid s_v,r_v,c)
=
A_v(du_v\mid z_v,r_v,c)\,
K_v^\Gamma(dz_v\mid s_v,r_v,c)
}
\qquad P_{S_v,R_v,C}\text{-a.s.}
\label{eq:local-kernel-factorization}
\end{equation}
The reconstruction kernel may depend on the incoming compressed child tuple
$r_v$, as it should in a genuine recursive host, but it is independent of
the adjacent predictive pair $s_v$ once $(z_v,r_v,c)$ are fixed.

For an uncompressed non-root node, $Z_v=U_v$. Define
\begin{align}
K_v^\Gamma(dz_v\mid s_v,r_v,c)
&:=K_v(dz_v\mid s_v,r_v,c),\\
A_v(du_v\mid z_v,r_v,c)
&:=\delta_{z_v}(du_v).
\end{align}
Then Equation~\ref{eq:local-kernel-factorization} holds identically, without
requiring either a local predictive deficiency condition or Equation~\ref{eq:interface-screening}.
Consequently, Equation~\ref{eq:local-kernel-factorization} is available for
every non-root node.

\medskip
\noindent\textbf{Step 3: Local closure composes recursively over a subtree.}

For a non-root node $v$, let $\mathcal T_v$ denote the subtree containing
$v$ and all its descendants, and write
\[
U_{\mathcal T_v}:=(U_r)_{r\in V(\mathcal T_v)},
\qquad
Z_{\mathcal T_v}:=(Z_r)_{r\in V(\mathcal T_v)}.
\]
Let $\mathsf M_v^{U,\Gamma}$ denote the conditional joint kernel of
$(U_{\mathcal T_v},Z_{\mathcal T_v})$ given
$(Y_{\operatorname{pa}(v)},C)$, and let $\mathsf M_v^\Gamma$ denote its
$Z_{\mathcal T_v}$-marginal.

The recursive Markov structure in (M1)--(M2) gives the following precise
child-before-parent construction. For every bounded non-negative measurable
function $f$ on
$\mathsf U_{\mathcal T_v}\times\mathsf Z_{\mathcal T_v}$,
\begin{equation}
\begin{aligned}
&\int f(u_{\mathcal T_v},z_{\mathcal T_v})\,
\mathsf M_v^{U,\Gamma}
(du_{\mathcal T_v},dz_{\mathcal T_v}\mid y_{\operatorname{pa}(v)},c)\\
&=\int Q_v(dy_v\mid y_{\operatorname{pa}(v)},c)
\int
\bigotimes_{u\in\operatorname{ch}(v)}
\mathsf M_u^{U,\Gamma}
(du_{\mathcal T_u},dz_{\mathcal T_u}\mid y_v,c)\\
&\qquad\times
\int
\widetilde K_v
(du_v,dz_v\mid s_v,r_v,c)\,
 f\!\left(
 u_v,(u_{\mathcal T_u})_{u\in\operatorname{ch}(v)},
 z_v,(z_{\mathcal T_u})_{u\in\operatorname{ch}(v)}
 \right),
\end{aligned}
\label{eq:recursive-subtree-kernel}
\end{equation}
where
$s_v=(y_v,y_{\operatorname{pa}(v)})$ and
$r_v=(z_u)_{u\in\operatorname{ch}(v)}$ is read from the root compressed
state of each child subtree. At a leaf, the product kernel over children is
the point mass on the empty tuple. Equation~\ref{eq:recursive-subtree-kernel}
is a finite composition of probability kernels and therefore defines a
probability kernel by the usual kernel-composition theorem.

Marginalizing the same construction over all $U$-variables gives
\begin{equation}
\begin{aligned}
&\mathsf M_v^\Gamma
(dz_{\mathcal T_v}\mid y_{\operatorname{pa}(v)},c)\\
&=\int Q_v(dy_v\mid y_{\operatorname{pa}(v)},c)
\left[
\bigotimes_{u\in\operatorname{ch}(v)}
\mathsf M_u^\Gamma
(dz_{\mathcal T_u}\mid y_v,c)
\right]
K_v^\Gamma(dz_v\mid s_v,r_v,c),
\end{aligned}
\label{eq:recursive-z-kernel}
\end{equation}
where the right-hand side is understood as the same child-before-parent
kernel composition as in Equation~\ref{eq:recursive-subtree-kernel}; in
particular $r_v$ is evaluated from the sampled child-root $z$-coordinates
before $K_v^\Gamma$ is applied.

Define reconstruction kernels recursively. At a leaf,
\[
A_{\mathcal T_v}(du_v\mid z_v,c)
:=
A_v(du_v\mid z_v,r_v^{\emptyset},c),
\]
where $r_v^{\emptyset}$ is the empty tuple. At an internal node, define
\begin{equation}
\begin{aligned}
&A_{\mathcal T_v}
(du_{\mathcal T_v}\mid z_{\mathcal T_v},c)\\
&\quad:=
A_v
(du_v\mid z_v,z_{\operatorname{ch}(v)},c)
\bigotimes_{u\in\operatorname{ch}(v)}
A_{\mathcal T_u}
(du_{\mathcal T_u}\mid z_{\mathcal T_u},c).
\end{aligned}
\label{eq:recursive-reconstruction-kernel}
\end{equation}
Because $z_{\operatorname{ch}(v)}$ is a measurable coordinate projection of
$z_{\mathcal T_v}$ and the tree is finite, this is a well-defined probability
kernel from $(z_{\mathcal T_v},c)$ to $u_{\mathcal T_v}$.

We prove by induction on the height of $\mathcal T_v$ that, for every
non-root node $v$,
\begin{equation}
\begin{aligned}
&\mathsf M_v^{U,\Gamma}
(du_{\mathcal T_v},dz_{\mathcal T_v}
\mid y_{\operatorname{pa}(v)},c)\\
&\quad=
A_{\mathcal T_v}
(du_{\mathcal T_v}\mid z_{\mathcal T_v},c)\,
\mathsf M_v^\Gamma
(dz_{\mathcal T_v}\mid y_{\operatorname{pa}(v)},c)
\end{aligned}
\label{eq:subtree-factorization}
\end{equation}
for $P_{Y_{\operatorname{pa}(v)},C}$-almost every
$(y_{\operatorname{pa}(v)},c)$.

Equation~\ref{eq:subtree-factorization} is the subtree closure statement. It
implies that a version of the conditional law satisfies
\[
P(
 du_{\mathcal T_v}
 \mid
 Z_{\mathcal T_v}=z_{\mathcal T_v},
 Y_{\operatorname{pa}(v)}=y_{\operatorname{pa}(v)},
 C=c)
=
A_{\mathcal T_v}
(du_{\mathcal T_v}\mid z_{\mathcal T_v},c),
\]
so reconstruction of the complete pre-compression subtree state depends on
the compressed subtree and retained context but not on the external boundary
state.

\smallskip
\noindent\emph{Base case.}
Suppose $v$ is a leaf. Then $r_v$ is the empty tuple and
Equation~\ref{eq:recursive-subtree-kernel}, applied to a measurable rectangle
$B_v\times D_v$, gives
\[
\begin{aligned}
&\mathsf M_v^{U,\Gamma}
(B_v\times D_v\mid y_{\operatorname{pa}(v)},c)\\
&\quad=
\int Q_v(dy_v\mid y_{\operatorname{pa}(v)},c)\,
\widetilde K_v
(B_v\times D_v\mid s_v,r_v^{\emptyset},c).
\end{aligned}
\]
Substituting Equation~\ref{eq:local-kernel-factorization},
\[
\begin{aligned}
&\mathsf M_v^{U,\Gamma}
(B_v\times D_v\mid y_{\operatorname{pa}(v)},c)\\
&\quad=
\int Q_v(dy_v\mid y_{\operatorname{pa}(v)},c)
\int_{D_v}
A_v(B_v\mid z_v,r_v^{\emptyset},c)\,
K_v^\Gamma(dz_v\mid s_v,r_v^{\emptyset},c).
\end{aligned}
\]
The reconstruction factor is independent of both boundary variables once
$(z_v,c)$ are fixed. Tonelli's theorem therefore gives
\[
\begin{aligned}
&\mathsf M_v^{U,\Gamma}
(B_v\times D_v\mid y_{\operatorname{pa}(v)},c)\\
&\quad=
\int_{D_v}
A_{\mathcal T_v}(B_v\mid z_v,c)\,
\mathsf M_v^\Gamma
(dz_v\mid y_{\operatorname{pa}(v)},c),
\end{aligned}
\]
where the second factor is exactly the leaf case of
Equation~\ref{eq:recursive-z-kernel}. Equality on rectangles extends to the
full product $\sigma$-algebra by the monotone class theorem, proving
Equation~\ref{eq:subtree-factorization} at every leaf.

\smallskip
\noindent\emph{Induction step.}
Let $v$ be an internal node and assume
Equation~\ref{eq:subtree-factorization} holds for every
$u\in\operatorname{ch}(v)$. Fix measurable rectangles
\[
B_{\mathcal T_v}
:=
B_v\times
\prod_{u\in\operatorname{ch}(v)}B_u,
\qquad
D_{\mathcal T_v}
:=
D_v\times
\prod_{u\in\operatorname{ch}(v)}D_u,
\]
with
$B_v\in\mathcal B(\mathsf U_v)$,
$D_v\in\mathcal B(\mathsf Z_v)$,
$B_u\in\mathcal B(\mathsf U_{\mathcal T_u})$, and
$D_u\in\mathcal B(\mathsf Z_{\mathcal T_u})$.
Applying Equation~\ref{eq:recursive-subtree-kernel} and then the induction
hypothesis to each child gives
\[
\begin{aligned}
&\mathsf M_v^{U,\Gamma}
(B_{\mathcal T_v}\times D_{\mathcal T_v}
\mid y_{\operatorname{pa}(v)},c)\\
&=\int Q_v(dy_v\mid y_{\operatorname{pa}(v)},c)
\int_{\prod_{u}D_u}
\left[
\prod_{u\in\operatorname{ch}(v)}
A_{\mathcal T_u}(B_u\mid z_{\mathcal T_u},c)
\right]\\
&\qquad\qquad\times
\widetilde K_v
(B_v\times D_v\mid s_v,r_v,c)
\prod_{u\in\operatorname{ch}(v)}
\mathsf M_u^\Gamma
(dz_{\mathcal T_u}\mid y_v,c).
\end{aligned}
\]
Here and below $r_v=(z_u)_{u\in\operatorname{ch}(v)}$ is the tuple of child
root coordinates extracted from the sampled $z_{\mathcal T_u}$'s. Substitute
the local factorization
Equation~\ref{eq:local-kernel-factorization}:
\[
\begin{aligned}
&\mathsf M_v^{U,\Gamma}
(B_{\mathcal T_v}\times D_{\mathcal T_v}
\mid y_{\operatorname{pa}(v)},c)\\
&=\int Q_v(dy_v\mid y_{\operatorname{pa}(v)},c)
\int_{\prod_{u}D_u}
\int_{D_v}
\left[
A_v(B_v\mid z_v,r_v,c)
\prod_{u\in\operatorname{ch}(v)}
A_{\mathcal T_u}(B_u\mid z_{\mathcal T_u},c)
\right]\\
&\qquad\qquad\times
K_v^\Gamma(dz_v\mid s_v,r_v,c)
\prod_{u\in\operatorname{ch}(v)}
\mathsf M_u^\Gamma
(dz_{\mathcal T_u}\mid y_v,c).
\end{aligned}
\]
All factors are non-negative, so Tonelli's theorem applies. The bracketed
reconstruction factor depends on $(z_{\mathcal T_v},c)$ only: in particular,
although the parent reconstruction is allowed to depend on the incoming child
states $r_v$, it does not depend on either $y_v$ or
$y_{\operatorname{pa}(v)}$. By
Equation~\ref{eq:recursive-reconstruction-kernel}, the bracket is exactly
\[
A_{\mathcal T_v}
(B_{\mathcal T_v}\mid z_{\mathcal T_v},c).
\]
The remaining $y_v$- and $z$-kernel composition is exactly
$\mathsf M_v^\Gamma
(dz_{\mathcal T_v}\mid y_{\operatorname{pa}(v)},c)$ by
Equation~\ref{eq:recursive-z-kernel}. Hence
\[
\begin{aligned}
&\mathsf M_v^{U,\Gamma}
(B_{\mathcal T_v}\times D_{\mathcal T_v}
\mid y_{\operatorname{pa}(v)},c)\\
&\quad=
\int_{D_{\mathcal T_v}}
A_{\mathcal T_v}
(B_{\mathcal T_v}\mid z_{\mathcal T_v},c)\,
\mathsf M_v^\Gamma
(dz_{\mathcal T_v}\mid y_{\operatorname{pa}(v)},c).
\end{aligned}
\]
Equality on measurable rectangles again extends to the full product
$\sigma$-algebra by the monotone class theorem. This completes the induction
and proves Equation~\ref{eq:subtree-factorization} for every non-root subtree.

\medskip
\noindent\textbf{Step 4: Subtree closure gives global conditional
independence at the root.}

By (M2), the child subtrees of the root are conditionally independent given
$(Y_\rho,C)$. Since they form a disjoint partition of all non-root nodes,
\[
U=(U_{\mathcal T_u})_{u\in\operatorname{ch}(\rho)},
\qquad
Z^\Gamma=(Z_{\mathcal T_u})_{u\in\operatorname{ch}(\rho)}.
\]
Therefore, for $P_{Y_\rho,C}$-almost every $(y_\rho,c)$,
\[
P(du,dz\mid y_\rho,c)
=
\bigotimes_{u\in\operatorname{ch}(\rho)}
\mathsf M_u^{U,\Gamma}
(du_{\mathcal T_u},dz_{\mathcal T_u}\mid y_\rho,c).
\]
Applying Equation~\ref{eq:subtree-factorization} to each child factor gives
\[
\begin{aligned}
P(du,dz\mid y_\rho,c)
&=
\left[
\bigotimes_{u\in\operatorname{ch}(\rho)}
A_{\mathcal T_u}
(du_{\mathcal T_u}\mid z_{\mathcal T_u},c)
\right]\\
&\qquad\times
\left[
\bigotimes_{u\in\operatorname{ch}(\rho)}
\mathsf M_u^\Gamma
(dz_{\mathcal T_u}\mid y_\rho,c)
\right].
\end{aligned}
\]
Define
\[
\begin{aligned}
A_{\mathcal T}(du\mid z,c)
&:=
\bigotimes_{u\in\operatorname{ch}(\rho)}
A_{\mathcal T_u}
(du_{\mathcal T_u}\mid z_{\mathcal T_u},c),
\\
\mathsf M_\rho^\Gamma(dz\mid y_\rho,c)
&:=
\bigotimes_{u\in\operatorname{ch}(\rho)}
\mathsf M_u^\Gamma
(dz_{\mathcal T_u}\mid y_\rho,c).
\end{aligned}
\]
Equivalently, because the tree is finite,
\[
A_{\mathcal T}(du\mid z,c)
=
\bigotimes_{v\neq\rho}
A_v\!\left(
 du_v\mid z_v,z_{\operatorname{ch}(v)},c
\right),
\]
where the factors are kernels in the $u_v$ coordinates with the full $z$
held fixed. Thus
\begin{equation}
P(du,dz\mid y_\rho,c)
=
A_{\mathcal T}(du\mid z,c)\,
\mathsf M_\rho^\Gamma(dz\mid y_\rho,c)
\qquad P_{Y_\rho,C}\text{-a.s.}
\label{eq:root-factorization}
\end{equation}
Equation~\ref{eq:root-factorization} is a disintegration of the conditional
joint law of $(U,Z^\Gamma)$ given $(Y_\rho,C)$. Hence
$A_{\mathcal T}$ is a version of the conditional law of $U$ given
$(Y_\rho,Z^\Gamma,C)$:
\begin{equation}
P(du\mid y_\rho,z,c)
=
A_{\mathcal T}(du\mid z,c)
\label{eq:root-reconstruction-given-y}
\end{equation}
for $P_{Y_\rho,Z^\Gamma,C}$-almost every $(y_\rho,z,c)$.

It remains only to identify the same kernel as the conditional law given
$(Z^\Gamma,C)$. For measurable $B\subseteq\mathsf U$, the tower property and
Equation~\ref{eq:root-reconstruction-given-y} give
\[
\begin{aligned}
P(U\in B\mid z,c)
&=
\int
P(U\in B\mid y_\rho,z,c)\,
P(dy_\rho\mid z,c)\\
&=
\int
A_{\mathcal T}(B\mid z,c)\,
P(dy_\rho\mid z,c)\\
&=
A_{\mathcal T}(B\mid z,c).
\end{aligned}
\]
Therefore
\[
P(du\mid y_\rho,z,c)
=
P(du\mid z,c)
\]
for $P_{Y_\rho,Z^\Gamma,C}$-almost every $(y_\rho,z,c)$, i.e.
\begin{equation}
Y_\rho\perp U\mid(Z^\Gamma,C).
\label{eq:global-independence}
\end{equation}

\medskip
\noindent\textbf{Step 5: Global conditional independence preserves the root
law.}

Since $Z^\Gamma$ is a measurable function of $U$,
\[
\sigma(U,Z^\Gamma,C)=\sigma(U,C).
\]
Let $D$ be a measurable subset of the state space of $Y_\rho$. Therefore
\[
P(Y_\rho\in D\mid U,C)
=
P(Y_\rho\in D\mid U,Z^\Gamma,C)
\qquad P\text{-a.s.}
\]
By Equation~\ref{eq:global-independence},
\[
P(Y_\rho\in D\mid U,Z^\Gamma,C)
=
P(Y_\rho\in D\mid Z^\Gamma,C)
\qquad P\text{-a.s.}
\]
Combining the two equalities and using a countable determining class for the
standard Borel state space of $Y_\rho$ yields
\[
\mathcal L(Y_\rho\mid U,C)
=
\mathcal L(Y_\rho\mid Z^\Gamma,C)
\qquad P\text{-a.s.}
\]
Finally, Equation~\ref{eq:global-independence} and the conditional
Kullback--Leibler characterization of mutual information give
\[
I(Y_\rho;U\mid Z^\Gamma,C)=0.
\]
This proves the theorem.
\end{proof}

\makeatletter
\@ifundefined{prop}{\newtheorem{prop}{Proposition}[section]}{}
\@ifundefined{proof}{%
  \newenvironment{proof}
    {\par\noindent\textit{Proof.}\ }%
    {\hfill\mbox{\rule{0.55em}{0.55em}}\par}%
}{}
\makeatother

\subsection{Local Variational Characterization}
\label{app:local-variational-characterization}

This subsection establishes the exact variational identity underlying the predictive-deficiency construction in Section~\ref{sec:predictive-compression-learning}. In particular, Proposition~\ref{prop:local-variational-characterization} gives the population criterion from which the empirical surrogate is derived.

This appendix gives the exact variational form of the local predictive
deficiency used in the main text. For a compressed node
$v\in V_{\mathrm{comp}}$, write
\[
S_v:=(Y_v,Y_{\operatorname{pa}(v)}),
\qquad
Z_v^\Gamma:=\Gamma(U_v),
\]
and
\[
\delta_v(\Gamma)
:=
I(S_v;U_v\mid Z_v^\Gamma,C).
\]
All variables are assumed to take values in standard Borel spaces and
$\Gamma$ is measurable, so the regular conditional distributions below
exist. Kullback--Leibler divergences are understood in the extended-real
sense.

For any measurable conditional distribution
$q_v(\cdot\mid Z_v^\Gamma,C)$ of $S_v$, define
\[
\mathcal V_v(\Gamma,q_v)
:=
\mathbb E\!\left[
D_{\mathrm{KL}}\!\left(
P_{S_v\mid U_v,C}
\,\middle\|\,
q_v(\cdot\mid Z_v^\Gamma,C)
\right)
\right].
\]

\begin{prop}[Local variational characterization]
\label{prop:local-variational-characterization}
For every measurable compressor $\Gamma$ and every compressed node $v$,
\begin{equation}
\boxed{
\mathcal V_v(\Gamma,q_v)
=
\delta_v(\Gamma)
+
\mathbb E\!\left[
D_{\mathrm{KL}}\!\left(
P_{S_v\mid Z_v^\Gamma,C}
\,\middle\|\,
q_v(\cdot\mid Z_v^\Gamma,C)
\right)
\right].
}
\label{eq:local-kl-decomposition}
\end{equation}
Consequently,
\begin{equation}
\boxed{
\delta_v(\Gamma)
=
\inf_{q_v}\mathcal V_v(\Gamma,q_v),
}
\label{eq:local-variational-characterization}
\end{equation}
and the infimum is attained by
\[
q_v^\star(\cdot\mid Z_v^\Gamma,C)
=
P_{S_v\mid Z_v^\Gamma,C}.
\]
If the common value is finite, every minimizer agrees with
$q_v^\star$ for $P_{Z_v^\Gamma,C}$-almost every conditioning value.
\end{prop}

\begin{proof}
Suppress the node index and write
\[
S:=S_v,\qquad U:=U_v,\qquad Z:=Z_v^\Gamma.
\]
Because $Z=\Gamma(U)$ is measurable with respect to $U$,
\begin{equation}
P_{S\mid U,Z,C}
=
P_{S\mid U,C}
\qquad P\text{-a.s.}
\label{eq:z-measurable-from-u}
\end{equation}
For $P_{Z,C}$-almost every $(z,c)$, the conditional relative-entropy
chain rule gives
\[
\begin{aligned}
&\int
D_{\mathrm{KL}}\!\left(
P_{S\mid U=u,Z=z,C=c}
\,\middle\|\,
q_v(\cdot\mid z,c)
\right)
P_{U\mid Z=z,C=c}(du)\\
&\quad=
I(S;U\mid Z=z,C=c)
+
D_{\mathrm{KL}}\!\left(
P_{S\mid Z=z,C=c}
\,\middle\|\,
q_v(\cdot\mid z,c)
\right).
\end{aligned}
\]
Integrating over $(Z,C)$ and using
Equation~\ref{eq:z-measurable-from-u} yields
Equation~\ref{eq:local-kl-decomposition}.

The second term in Equation~\ref{eq:local-kl-decomposition} is non-negative
and vanishes at
$q_v=q_v^\star=P_{S_v\mid Z_v^\Gamma,C}$. Therefore
\[
\inf_{q_v}\mathcal V_v(\Gamma,q_v)
=
\delta_v(\Gamma).
\]
If the optimum is finite, equality requires the remaining conditional
Kullback--Leibler divergence to vanish almost surely, which gives the stated
uniqueness up to $P_{Z_v^\Gamma,C}$-null sets.
\end{proof}

\paragraph{Zero-defect interpretation.}
Since $Z_v^\Gamma$ is a measurable function of $U_v$,
\begin{equation}
\boxed{
\delta_v(\Gamma)=0
\quad\Longleftrightarrow\quad
\mathcal L(S_v\mid U_v,C)
=
\mathcal L(S_v\mid Z_v^\Gamma,C)
\quad P\text{-a.s.}
}
\label{eq:local-zero-defect-sufficiency}
\end{equation}
Thus this appendix introduces no additional objective: it is an exact
reformulation of the same local predictive deficiency used by the recursive
closure result. The finite-test and empirical restrictions begin only in
Appendix~\ref{app:tractable-local-objective}.

\makeatletter
\@ifundefined{theorem}{\newtheorem{theorem}{Theorem}[section]}{}
\@ifundefined{lemma}{\newtheorem{lemma}{Lemma}[section]}{}
\@ifundefined{corollary}{\newtheorem{corollary}{Corollary}[section]}{}
\@ifundefined{proof}{%
  \newenvironment{proof}
    {\par\noindent\textit{Proof.}\ }%
    {\hfill\mbox{\rule{0.55em}{0.55em}}\par}%
}{}
\makeatother

\subsection{From Local Predictive Deficiency to Constrained Updates}
\label{app:tractable-local-objective}

This subsection supplies the approximation and calibration chain used in Section~\ref{sec:predictive-compression-learning} and the projection derivation used in Section~\ref{sec:predictive-constraints}. In particular, Appendix~\ref{app:predictive-feasibility-proof} proves the dual characterization stated in Proposition~\ref{prop:predictive-feasibility}.

This appendix connects the ideal local predictive deficiency in the main text to
its computable training statistic. The purpose is deliberately limited. We do
not introduce a second root-level theorem, nor do we identify the empirical
statistic with conditional mutual information. Instead, we isolate the
successive restrictions used by the implementation and make their residuals
explicit:
\begin{equation}
\boxed{
\delta_v
\;\longrightarrow\;
\Delta_v^{\rm test}
\;\longrightarrow\;
\Delta_v^{\rm bal}
\;\longrightarrow\;
\widehat{\mathcal D}_{v,\mathcal B}.
}
\label{eq:appendix-surrogate-chain}
\end{equation}
The main full-covariance route has three residuals:
$\varepsilon_{{\rm test},v}$ for finite-test coverage,
$\varepsilon_{{\rm lin},v}$ for the readout class, and
$\varepsilon_{{\rm stat},v}$ for finite-sample estimation. The optional
diagonal-covariance variant introduces $D_v^{\rm diag}$ between the profiled
value and its empirical diagonal estimate, with an additional
$\varepsilon_{{\rm diag},v}$ residual.

Throughout, let
\[
S_v:=(Y_v,Y_{\operatorname{pa}(v)}),
\qquad
Z_v^\Gamma:=\Gamma(U_v),
\qquad
C:=(C_w)_{w\in V},
\]
and use the main-text local predictive deficiency
\begin{equation}
\boxed{
\delta_v(\Gamma)
:=
I(S_v;U_v\mid Z_v^\Gamma,C).
}
\label{eq:appendix-local-defect}
\end{equation}
If an interface is usable only when both events of the adjacent predictive pair are present,
let $M_v$ be its validity indicator and write
\[
\mathbb E_v[F]
:=
\frac{\mathbb E[M_vF]}{\mathbb E[M_v]},
\qquad
\mathbb P(M_v=1)>0.
\]
The calibration statements below are evaluated under this same
eligible-interface law. For Hilbert-valued $f$, write
$\|f\|_{L^2_v}^2:=\mathbb E_v\|f\|^2$. Censored observations are omitted
rather than encoded as zero-valued targets. When every theoretical interface
is observed, this law is the main-text law. With censoring, identifying the
eligible-interface local predictive deficiency with the unconditional main-text local predictive deficiency requires the
corresponding support/selection condition; without that condition, the bounds
below calibrate the local predictive deficiency on the observed eligible interfaces only.

\paragraph{Scope.}
The empirical statistic is not claimed to be an exact finite-sample estimator
of the local predictive deficiency. For a fixed interface law, the
exact local criterion is reduced through finite predictive tests, a restricted
readout class, and full-covariance empirical moments. We bound their residual
terms under
the stated regularity conditions. The root-level information-theoretic
guarantee in the main text applies to the ideal local predictive deficiencies. The calibration
results below quantify the corresponding local approximation; the pooled
training statistic used by Algorithm~\ref{alg:predictive-feasibility} is treated separately when its
interface-wise differential is constructed.

\subsubsection{Exact Local Criterion and Context Localization}
\label{app:global-local-context}

The main-text criterion in Equation~\ref{eq:appendix-local-defect} conditions
on the retained global context $C$, whereas the implemented statistic uses the
context $C_v$ visible at interface $v$. Let
$C_{-v}:=(C_w)_{w\in V\setminus\{v\}}$. We use the following localization
condition only to identify these two conditioning choices:
\begin{equation}
S_v\perp C_{-v}\mid(U_v,C_v),
\qquad
S_v\perp C_{-v}\mid(Z_v^\Gamma,C_v).
\label{eq:context-localization}
\end{equation}
By the chain rule,
\[
\begin{aligned}
I(S_v;U_v\mid Z_v^\Gamma,C_v,C_{-v})
={}&I(S_v;U_v\mid Z_v^\Gamma,C_v)\\
&+I(S_v;C_{-v}\mid U_v,Z_v^\Gamma,C_v)
-I(S_v;C_{-v}\mid Z_v^\Gamma,C_v).
\end{aligned}
\]
The last two terms vanish under
Equation~\ref{eq:context-localization}, using that
$Z_v^\Gamma$ is measurable with respect to $U_v$. Hence
\begin{equation}
\boxed{
\delta_v(\Gamma)
=
I(S_v;U_v\mid Z_v^\Gamma,C_v).
}
\label{eq:global-local-defect-equivalence}
\end{equation}
Without Equation~\ref{eq:context-localization}, the construction below should
be read as a local-context surrogate. The pre-compression reference is $U_v$
throughout this appendix.

For the comparison with the main-text defect below, take $M_v=1$ almost
surely. With censored interfaces, the same derivations apply to the defect
and conditional laws under the eligible-interface law; identifying that
defect with the unconditional one needs an additional selection assumption.
Under full observation, the variational identity remains exact after
localization:
\begin{equation}
\delta_v(\Gamma)
=
\inf_{q_v}
\mathbb E_v\!\left[
D_{\rm KL}\!\left(
P_{S_v\mid U_v,C_v}
\,\Vert\,
q_v(\cdot\mid Z_v^\Gamma,C_v)
\right)
\right].
\label{eq:localized-variational-defect}
\end{equation}
Every step below is a restriction or approximation of this local predictive
problem, not a replacement for Equation~\ref{eq:localized-variational-defect}.

\subsubsection{Finite Predictive Tests}
\label{app:predictive-risk-energy}

Following Section~\ref{sec:predictive-compression-learning}, let
\[
P_v
:=
T_S\!\left(
\chi(C_v)\otimes
\phi(Y_v,Y_{\operatorname{pa}(v)})
\right)
\in\mathbb R^{d_P},
\]
where $\chi$, $\phi$, and $T_S$ are fixed. The test vector supplies
supervision only; it is not an additional input to the recursive host. Define
\[
\overline P_v
:=
P_v-\mathbb E_v[P_v\mid C_v],
\qquad
m_{X,v}
:=
\mathbb E_v[\overline P_v\mid X_v,C_v],
\quad
X_v\in\{U_v,Z_v^\Gamma\}.
\]
Because the $C_v$-measurable baseline can be absorbed into the predictor,
predicting $P_v$ and predicting $\overline P_v$ have the same Bayes-risk
difference. Let
\[
\mathcal R_v^\star(P_v\mid X_v,C_v)
:=
\inf_g
\mathbb E_v
\|P_v-g(X_v,C_v)\|_2^2.
\]
Conditional expectation is the $L^2$ projection onto
$\sigma(X_v,C_v)$. Since $Z_v^\Gamma=\Gamma(U_v)$, the tower property gives
$m_{Z,v}=\mathbb E_v[m_{U,v}\mid Z_v^\Gamma,C_v]$, and therefore
\begin{equation}
\boxed{
\begin{aligned}
\Delta_v^{\rm test}(\Gamma)
&:=
\mathcal R_v^\star(P_v\mid Z_v^\Gamma,C_v)
-
\mathcal R_v^\star(P_v\mid U_v,C_v)\\
&=
\mathbb E_v\|m_{U,v}-m_{Z,v}\|_2^2
\ge 0 .
\end{aligned}}
\label{eq:finite-test-gap}
\end{equation}
This is exactly the finite-test value gap used in the main text. The identity
is exact for the selected tests; the approximation is the use of those tests
in place of the full conditional law of $S_v$.

To make this approximation explicit, let $\psi_v(S_v,C_v)$ take values in a
separable Hilbert test space $\mathcal G_v^{\rm test}$. Let $\Pi_{m,v}$ be
the orthogonal projector onto the finite test span used by $P_v$, and let
$Q_{m,v}$ be an isometric coordinate map on that span. Define the full
conditional predictive change
\[
e_v(\Gamma)
:=
\mathbb E_v[\psi_v(S_v,C_v)\mid U_v,C_v]
-
\mathbb E_v[\psi_v(S_v,C_v)\mid Z_v^\Gamma,C_v].
\]
We use the following quantitative coverage condition on the admissible law and
compressor class: for some finite $c_v^{(0)}>0$,
\begin{equation}
\delta_v(\Gamma)
\le
c_v^{(0)}
\|e_v(\Gamma)\|_{L^2_v}^2.
\label{eq:test-domination-condition}
\end{equation}
Equation~\ref{eq:test-domination-condition} is a regularity condition relating
the predictive test geometry to conditional KL on the admissible law class;
no universal reverse Pinsker inequality is assumed.

\begin{lemma}[Finite-test coverage]
\label{lem:finite-test-coverage}
Suppose Equation~\ref{eq:test-domination-condition} holds and
\[
P_v=Q_{m,v}\Pi_{m,v}\psi_v(S_v,C_v).
\]
Then
\begin{equation}
\boxed{
\delta_v(\Gamma)
\le
c_v^{(0)}\Delta_v^{\rm test}(\Gamma)
+
\varepsilon_{{\rm test},v}(\Gamma),
}
\label{eq:finite-test-control}
\end{equation}
where
\[
\varepsilon_{{\rm test},v}(\Gamma)
:=
c_v^{(0)}
\|(I-\Pi_{m,v})e_v(\Gamma)\|_{L^2_v}^2.
\]
For a nested complete test family,
$\varepsilon_{{\rm test},v}(\Gamma)\to0$ pointwise in $\Gamma$.
\end{lemma}

\begin{proof}
Orthogonality of $\Pi_{m,v}$ gives
\[
\|e_v\|_{L^2_v}^2
=
\|\Pi_{m,v}e_v\|_{L^2_v}^2
+
\|(I-\Pi_{m,v})e_v\|_{L^2_v}^2.
\]
Linearity of conditional expectation and the coordinate isometry imply
\[
m_{U,v}-m_{Z,v}
=
Q_{m,v}\Pi_{m,v}e_v.
\]
Hence
$\Delta_v^{\rm test}
=\|\Pi_{m,v}e_v\|_{L^2_v}^2$.
Substitution into Equation~\ref{eq:test-domination-condition} yields
Equation~\ref{eq:finite-test-control}.
\end{proof}

The domination condition is non-vacuous in standard discrete predictive
settings. For example, suppose
$S_v\in\{1,\ldots,K\}$, take the complete one-hot test
$\psi_v(S_v)=e_{S_v}$, and write
\[
p_U(\cdot):=
\mathbb P_v(S_v=\cdot\mid U_v,C_v),
\qquad
p_Z(\cdot):=
\mathbb P_v(S_v=\cdot\mid Z_v^\Gamma,C_v).
\]
If
\[
p_Z(s)\ge\alpha_v>0
\quad
\text{a.s. for every }s,
\]
then
\[
\begin{aligned}
\delta_v(\Gamma)
&=
\mathbb E_v D_{\rm KL}(p_U\Vert p_Z)\\
&\le
\mathbb E_v\chi^2(p_U,p_Z)\\
&=
\mathbb E_v\sum_{s=1}^K
\frac{(p_U(s)-p_Z(s))^2}{p_Z(s)}\\
&\le
\frac{1}{\alpha_v}
\mathbb E_v\|p_U-p_Z\|_2^2 .
\end{aligned}
\]
Thus Equation~\ref{eq:test-domination-condition} holds with
$c_v^{(0)}=\alpha_v^{-1}$, and for the complete one-hot family the test
coverage residual is zero. More general finite sketches need not determine the
full conditional law; their uncovered component is what $\varepsilon_{{\rm test},v}$ records.

Thus $\varepsilon_{{\rm test},v}$ is a representation/test approximation
term, not a sampling error.

\subsubsection{Profiled Predictive Value}
\label{app:lifted-profiled-risk}
\label{app:full-balanced-energy}

The finite-test value in Equation~\ref{eq:finite-test-gap} uses the raw test
coordinates, as in the main text. The profiled population criterion
additionally balances the target coordinates. Define
\[
\Sigma_{PP,v}
:=
\mathbb E_v[\overline P_v\overline P_v^\top],
\qquad
W_{P,v}
:=
(\Sigma_{PP,v}+\varepsilon I)^{-1/2},
\qquad
\widetilde P_v:=W_{P,v}\overline P_v,
\]
with $\varepsilon>0$, and let
\[
\widetilde m_{X,v}
:=
\mathbb E_v[\widetilde P_v\mid X_v,C_v]
=
W_{P,v}m_{X,v}.
\]
The corresponding balanced Bayes value gap is
\[
\Delta_v^{{\rm test},{\rm bal}}(\Gamma)
:=
\mathcal R_v^\star(\widetilde P_v\mid Z_v^\Gamma,C_v)
-
\mathcal R_v^\star(\widetilde P_v\mid U_v,C_v)
=
\mathbb E_v
\|W_{P,v}(m_{U,v}-m_{Z,v})\|_2^2 .
\]
Let $s_v:=\sigma_{\min}(W_{P,v})>0$. Then
\begin{equation}
\Delta_v^{\rm test}(\Gamma)
\le
s_v^{-2}
\Delta_v^{{\rm test},{\rm bal}}(\Gamma).
\label{eq:raw-to-balanced-test-gap}
\end{equation}
This step changes only the metric on the fixed finite test coordinates.

We next restrict the Bayes predictor to the fixed readout coordinates used by
the implementation and profile the readout exactly. For notational
simplicity, write $X_v$ for those fixed coordinates; if a fixed lift is used,
$X_v$ below denotes that lift. Define
\[
\overline X_v
:=
X_v-\mathbb E_v[X_v\mid C_v],
\]
and
\[
\Sigma_{XX,v}:=\mathbb E_v[\overline X_v\overline X_v^\top],
\qquad
\Sigma_{PX,v}:=\mathbb E_v[\overline P_v\overline X_v^\top],
\qquad
\Sigma_{XP,v}:=\Sigma_{PX,v}^\top.
\]
For $\lambda>0$, consider
\begin{equation}
\mathcal R_{\lambda,\varepsilon,v}^{\rm lin}(X_v\mid C_v)
:=
\min_A
\left\{
\mathbb E_v\|\widetilde P_v-A\overline X_v\|_2^2
+\lambda\|A\|_F^2
\right\}.
\label{eq:profiled-ridge-risk}
\end{equation}

\begin{prop}[Profiled predictive value]
\label{prop:profiled-predictive-energy}
The minimizer of Equation~\ref{eq:profiled-ridge-risk} is
\[
A_{X,v}^\star
=
W_{P,v}\Sigma_{PX,v}(\Sigma_{XX,v}+\lambda I)^{-1},
\]
and
\begin{equation}
\boxed{
\begin{aligned}
\mathcal R_{\lambda,\varepsilon,v}^{\rm lin}(X_v\mid C_v)
={}&
\operatorname{tr}\!\left[
(\Sigma_{PP,v}+\varepsilon I)^{-1}\Sigma_{PP,v}
\right]
-J_{{\rm bal},v}(X_v\mid C_v),\\
J_{{\rm bal},v}(X_v\mid C_v)
:={}&
\operatorname{tr}\!\left[
(\Sigma_{PP,v}+\varepsilon I)^{-1}
\Sigma_{PX,v}
(\Sigma_{XX,v}+\lambda I)^{-1}
\Sigma_{XP,v}
\right].
\end{aligned}}
\label{eq:profiled-predictive-energy}
\end{equation}
Equivalently,
\[
J_{{\rm bal},v}(X_v\mid C_v)
=
\left\|
(\Sigma_{PP,v}+\varepsilon I)^{-1/2}
\Sigma_{PX,v}
(\Sigma_{XX,v}+\lambda I)^{-1/2}
\right\|_F^2.
\]
\end{prop}

\begin{proof}
Let $G_{X,v}:=\Sigma_{XX,v}+\lambda I\succ0$. Expanding
Equation~\ref{eq:profiled-ridge-risk} and completing the square gives
\[
\begin{aligned}
&\mathbb E_v\|\widetilde P_v-A\overline X_v\|_2^2
+\lambda\|A\|_F^2\\
&\quad=
\operatorname{tr}(W_{P,v}\Sigma_{PP,v}W_{P,v})
-
\operatorname{tr}\!\left(
W_{P,v}\Sigma_{PX,v}G_{X,v}^{-1}
\Sigma_{XP,v}W_{P,v}
\right)\\
&\qquad+
\left\|
(A-W_{P,v}\Sigma_{PX,v}G_{X,v}^{-1})
G_{X,v}^{1/2}
\right\|_F^2.
\end{aligned}
\]
The final term is non-negative and vanishes at the stated minimizer.
Cyclicity of trace and $\operatorname{tr}(BB^\top)=\|B\|_F^2$ yield the two
forms of $J_{{\rm bal},v}$.
\end{proof}

The constant term in Equation~\ref{eq:profiled-predictive-energy} is common to
$U_v$ and $Z_v^\Gamma$. Consistently with the main text, define
\[
\boxed{
\Delta_v^{\rm bal}(\Gamma)
:=
J_{{\rm bal},v}(U_v\mid C_v)
-
J_{{\rm bal},v}(Z_v^\Gamma\mid C_v).
}
\]
Profiling itself is exact within the ridge class. The approximation restricts predictors to that class.

For $X_v\in\{U_v,Z_v^\Gamma\}$ define its ridge oracle residual
\[
\begin{aligned}
r_{{\rm lin},v}(X_v)
&:=
\mathcal R_{\lambda,\varepsilon,v}^{\rm lin}(X_v\mid C_v)
-
\mathcal R_v^\star(\widetilde P_v\mid X_v,C_v)\\
&=
\inf_A
\left\{
\mathbb E_v\|\widetilde m_{X,v}-A\overline X_v\|_2^2
+\lambda\|A\|_F^2
\right\}
\ge0 .
\end{aligned}
\]
The second equality follows from the $L^2$ orthogonality of
$\widetilde P_v-\widetilde m_{X,v}$ to every square-integrable function of
$(X_v,C_v)$. Since
\[
\Delta_v^{\rm bal}
=
\Delta_v^{{\rm test},{\rm bal}}
+r_{{\rm lin},v}(Z_v^\Gamma)
-r_{{\rm lin},v}(U_v),
\]
we have
\begin{equation}
\boxed{
\left|
\Delta_v^{{\rm test},{\rm bal}}(\Gamma)
-
\Delta_v^{\rm bal}(\Gamma)
\right|
\le
\varepsilon_{{\rm lin},v}(\Gamma),
}
\label{eq:linear-gap-control}
\end{equation}
where
\[
\varepsilon_{{\rm lin},v}(\Gamma)
:=
r_{{\rm lin},v}(U_v)+r_{{\rm lin},v}(Z_v^\Gamma).
\]
The term vanishes if both balanced conditional means lie in the readout span as regularization vanishes.

Combining Equations~\ref{eq:finite-test-control},
\ref{eq:raw-to-balanced-test-gap}, and \ref{eq:linear-gap-control} gives the
population profiled bridge
\begin{equation}
\boxed{
\delta_v(\Gamma)
\le
c_v\Delta_v^{\rm bal}(\Gamma)
+
\varepsilon_{{\rm test},v}(\Gamma)
+
c_v\varepsilon_{{\rm lin},v}(\Gamma),
\qquad
c_v:=\frac{c_v^{(0)}}{s_v^2}.
}
\label{eq:population-profiled-bridge}
\end{equation}

\paragraph{Width and the spectral oracle.}
\label{app:ky-fan-energy-deficiency}
A width-$r$ transmitted state does not, by dimensionality alone, induce a
rank-$r$ orthogonal projection in the balanced covariance geometry. A Ky--Fan
rank-$r$ construction can still be used as a diagnostic linear oracle when
projection realizability is imposed, but it is not needed for
Equation~\ref{eq:population-profiled-bridge} and is not identified with the
general nonlinear compressor. We therefore do not use a Ky--Fan equality in
the theoretical bridge from $\delta_v$ to the implemented statistic.

\subsubsection{Empirical Predictive Deficiency}
\label{app:empirical-predictive-deficiency}

We treat diagonal covariance as optional before deriving finite-sample bounds for the full-covariance method.

\paragraph{Population diagonal geometry.}
Define the stabilized diagonal marginal covariances
\[
D_{P,v}
:=
\operatorname{diag}(\Sigma_{PP,v})+\varepsilon I,
\qquad
D_{X,v}
:=
\operatorname{diag}(\Sigma_{XX,v})+\lambda I,
\]
and
\[
J_{{\rm diag},v}(X_v\mid C_v)
:=
\operatorname{tr}\!\left[
D_{P,v}^{-1}\Sigma_{PX,v}D_{X,v}^{-1}\Sigma_{XP,v}
\right].
\]
The corresponding population diagonal deficiency is
\[
D_v^{\rm diag}(\Gamma)
:=
J_{{\rm diag},v}(U_v\mid C_v)
-
J_{{\rm diag},v}(Z_v^\Gamma\mid C_v).
\]

Let
\[
E_{P,v}:=\Sigma_{PP,v}-\operatorname{diag}(\Sigma_{PP,v}),
\qquad
E_{X,v}:=\Sigma_{XX,v}-\operatorname{diag}(\Sigma_{XX,v}),
\]
and define the relative off-diagonal departures
\[
\rho_{P,v}
:=
\left\|
D_{P,v}^{-1/2}E_{P,v}D_{P,v}^{-1/2}
\right\|_{\rm op},
\qquad
\rho_{X,v}
:=
\left\|
D_{X,v}^{-1/2}E_{X,v}D_{X,v}^{-1/2}
\right\|_{\rm op}.
\]
Here and below $\|\cdot\|_{\rm op}$ denotes spectral norm.

\begin{lemma}[Diagonal covariance perturbation]
\label{lem:diagonal-perturbation}
For $X_v\in\{U_v,Z_v^\Gamma\}$, suppose
$\|\Sigma_{PX,v}\|_F\le K_{X,v}$. Then
\begin{equation}
\begin{aligned}
&\left|
J_{{\rm bal},v}(X_v\mid C_v)
-
J_{{\rm diag},v}(X_v\mid C_v)
\right|\\
&\quad\le
K_{X,v}^2
\left(
\frac{\|E_{P,v}\|_{\rm op}}{\varepsilon^2\lambda}
+
\frac{\|E_{X,v}\|_{\rm op}}{\varepsilon\lambda^2}
\right)
=:\beta_{X,v}.
\end{aligned}
\label{eq:diagonal-perturbation-bound}
\end{equation}
If additionally $\rho_{P,v}<1$ and $\rho_{X,v}<1$, the first display can be
replaced by the relative bound
\begin{equation}
\beta_{X,v}^{\rm rel}
:=
\frac{K_{X,v}^2}{\varepsilon\lambda}
\left[
\frac{\rho_{P,v}}
{(1-\rho_{P,v})(1-\rho_{X,v})}
+
\frac{\rho_{X,v}}{1-\rho_{X,v}}
\right].
\label{eq:diagonal-relative-bound}
\end{equation}
Consequently,
\begin{equation}
\boxed{
\left|
\Delta_v^{\rm bal}(\Gamma)-D_v^{\rm diag}(\Gamma)
\right|
\le
\varepsilon_{{\rm diag},v}(\Gamma),
}
\label{eq:diagonal-gap-control}
\end{equation}
with
$\varepsilon_{{\rm diag},v}:=\beta_{U,v}+\beta_{Z,v}$, or with the
corresponding relative bounds when their conditions hold.
\end{lemma}

\begin{proof}
Abbreviate
$A:=\Sigma_{PP,v}+\varepsilon I$,
$D_P:=D_{P,v}$,
$B:=\Sigma_{XX,v}+\lambda I$,
$D_X:=D_{X,v}$, and $C:=\Sigma_{PX,v}$. The resolvent identity gives
\[
\|A^{-1}-D_P^{-1}\|_{\rm op}
\le
\frac{\|E_{P,v}\|_{\rm op}}{\varepsilon^2},
\qquad
\|B^{-1}-D_X^{-1}\|_{\rm op}
\le
\frac{\|E_{X,v}\|_{\rm op}}{\lambda^2},
\]
while
$\|A^{-1}\|_{\rm op},\|D_P^{-1}\|_{\rm op}\le\varepsilon^{-1}$ and
$\|B^{-1}\|_{\rm op},\|D_X^{-1}\|_{\rm op}\le\lambda^{-1}$.
A two-term telescope gives
\[
\begin{aligned}
A^{-1}CB^{-1}C^\top-D_P^{-1}CD_X^{-1}C^\top
={}&(A^{-1}-D_P^{-1})CB^{-1}C^\top\\
&+D_P^{-1}C(B^{-1}-D_X^{-1})C^\top.
\end{aligned}
\]
Using
$|\operatorname{tr}(MN)|\le\|M\|_{\rm op}\|N\|_\ast$ and
$\|CBC^\top\|_\ast\le\|C\|_F^2\|B\|_{\rm op}$ proves
Equation~\ref{eq:diagonal-perturbation-bound}.

For the relative form, write
$A=D_P^{1/2}(I+R_P)D_P^{1/2}$ with
$\|R_P\|_{\rm op}=\rho_{P,v}$, and analogously for $B$. If
$\rho<1$, then
$\|(I+R)^{-1}\|_{\rm op}\le(1-\rho)^{-1}$ and
$\|(I+R)^{-1}-I\|_{\rm op}\le\rho/(1-\rho)$. Substituting these two bounds
into the same telescope gives
Equation~\ref{eq:diagonal-relative-bound}. Applying the single-state bound to
$U_v$ and $Z_v^\Gamma$ and using the triangle inequality proves
Equation~\ref{eq:diagonal-gap-control}.
\end{proof}

Equation~\ref{eq:diagonal-gap-control} makes precise the sense in which
diagonalization is a geometric approximation. The error is controlled by the
off-diagonal departure of the population covariance geometry and is distinct
from sampling noise.

\paragraph{Finite-sample moments.}
Let $\mathcal B$ contain $n_{\mathcal B}$ independent sampling units from the
eligible-interface law. In the tree setting, a sampling unit can be taken as
one independent tree/trajectory cluster; for i.i.d. interfaces it is one
interface. Let hats denote the corresponding empirical centered covariance
blocks. The main empirical value uses full regularized covariance blocks:
\begin{equation}
\boxed{
\widehat J_{{\rm bal},v,\mathcal B}(X_v)
:=
\operatorname{tr}\!\left[
(\widehat\Sigma_{PP,v}+\varepsilon I)^{-1}
\widehat\Sigma_{PX,v}
(\widehat\Sigma_{XX,v}+\lambda I)^{-1}
\widehat\Sigma_{XP,v}
\right].
}
\label{eq:empirical-balanced-value}
\end{equation}
The main empirical deficiency is
\begin{equation}
\boxed{
\widehat{\mathcal D}_{v,\mathcal B}(\Gamma)
:=
\widehat J_{{\rm bal},v,\mathcal B}(U_v)
-
\widehat J_{{\rm bal},v,\mathcal B}(Z_v^\Gamma).
}
\label{eq:computable-energy-loss}
\end{equation}
The $U_v$ term is a stop-gradient reference in optimization. The diagonal
variant replaces both regularized covariance inverses in
Equation~\ref{eq:empirical-balanced-value} by their regularized diagonal
inverses. Denote its value and gap by
$\widehat J_{{\rm diag},v,\mathcal B}$ and
$\widehat{\mathcal D}^{\rm diag}_{v,\mathcal B}$.

\begin{lemma}[Finite-sample perturbation]
\label{lem:finite-sample-perturbation}
Fix $\Gamma$ independently of the sampling units in $\mathcal B$. Suppose
$\varepsilon\ge\varepsilon_0>0$, $\lambda\ge\lambda_0>0$, and
\[
\|\Sigma_{PX,v}\|_F
\vee
\|\widehat\Sigma_{PX,v}\|_F
\le K_v,
\qquad X\in\{U,Z\}.
\]
On an event on which, simultaneously for $X\in\{U,Z\}$,
\[
\begin{aligned}
\|\widehat\Sigma_{PP,v}-\Sigma_{PP,v}\|_{\rm op}&\le\omega_{PP,v},\\
\|\widehat\Sigma_{XX,v}-\Sigma_{XX,v}\|_{\rm op}&\le\omega_{XX,v},\\
\|\widehat\Sigma_{PX,v}-\Sigma_{PX,v}\|_F&\le\omega_{PX,v}^F,
\end{aligned}
\]
one has
\begin{equation}
\left|
\widehat J_{{\rm bal},v,\mathcal B}(X_v)
-J_{{\rm bal},v}(X_v\mid C_v)
\right|
\le
\Lambda_{X,v},
\label{eq:sample-value-perturbation}
\end{equation}
where
\[
\Lambda_{X,v}
:=
\frac{K_v^2}{\varepsilon_0^2\lambda_0}\omega_{PP,v}
+
\frac{K_v^2}{\varepsilon_0\lambda_0^2}\omega_{XX,v}
+
\frac{2K_v}{\varepsilon_0\lambda_0}\omega_{PX,v}^F.
\]
Consequently,
\begin{equation}
\boxed{
\left|
\widehat{\mathcal D}_{v,\mathcal B}(\Gamma)
-\Delta_v^{\rm bal}(\Gamma)
\right|
\le
\varepsilon_{{\rm stat},v},
\qquad
\varepsilon_{{\rm stat},v}:=\Lambda_{U,v}+\Lambda_{Z,v}.
}
\label{eq:finite-sample-gap-control}
\end{equation}
If the three displayed covariance-deviation bounds hold jointly with
probability at least $1-\eta$, then
Equation~\ref{eq:finite-sample-gap-control} holds with the same probability.
Under standard bounded or sub-Gaussian sampling assumptions, the deviation
radii $\omega_{PP,v}$, $\omega_{XX,v}$, and $\omega_{PX,v}^F$ admit the usual
dimension-dependent $n_{\mathcal B}^{-1/2}$ concentration, with lower-order
$n_{\mathcal B}^{-1}$ terms as appropriate. For clustered trees,
$n_{\mathcal B}$ denotes the number of independent clusters and the
concentration constants depend on the corresponding cluster-norm bound.
\end{lemma}

\begin{proof}
The covariance bounds apply to the full regularized blocks. A four-term
telescope for the two inverse factors and the two cross-covariance factors,
together with the resolvent identity
$\widehat A^{-1}-A^{-1}=\widehat A^{-1}(A-\widehat A)A^{-1}$, yields
Equation~\ref{eq:sample-value-perturbation}. Applying it to $U_v$ and
$Z_v^\Gamma$ gives Equation~\ref{eq:finite-sample-gap-control}. The probability statement follows whenever the covariance-deviation event is
supplied by an appropriate concentration bound, for example standard
bounded/sub-Gaussian covariance concentration or matrix Bernstein for
independent bounded cluster summaries. For the diagonal variant, diagonal
extraction is nonexpansive in spectral norm; the same perturbation radius
bounds its empirical value relative to $J_{{\rm diag},v}$.
\end{proof}

When several fixed sketches are used, the statistic in
Equation~\ref{eq:computable-energy-loss} is computed for each sketch and
averaged before differentiation. The same bounds apply sketchwise and combine
by averaging (or by a union bound when a simultaneous high-probability event
is required).

\begin{corollary}[Controlled surrogate for the local predictive deficiency]
\label{cor:controlled-surrogate}
Under context localization, the test-domination condition, and the regularity
conditions of Lemmas~\ref{lem:finite-test-coverage} and
\ref{lem:finite-sample-perturbation}, suppose additionally that the covariance
deviation event in Lemma~\ref{lem:finite-sample-perturbation} holds with
probability at least $1-\eta$. Then, with probability at least $1-\eta$,
\begin{equation}
\boxed{
\begin{aligned}
\delta_v(\Gamma)
\le{}&
 c_v\widehat{\mathcal D}_{v,\mathcal B}(\Gamma)
+\varepsilon_{{\rm test},v}(\Gamma)\\
&+c_v\Big(
\varepsilon_{{\rm lin},v}(\Gamma)
+\varepsilon_{{\rm stat},v}
\Big).
\end{aligned}}
\label{eq:controlled-surrogate-final}
\end{equation}
For the diagonal variant, if Lemma~\ref{lem:diagonal-perturbation} also
holds, replace $\widehat{\mathcal D}_{v,\mathcal B}$ by
$\widehat{\mathcal D}^{\rm diag}_{v,\mathcal B}$ and add
$c_v\varepsilon_{{\rm diag},v}$ to the right-hand side. Absorbing $c_v$ into
the last two residual definitions for the main method yields
the compact form
\[
\delta_v(\Gamma)
\le
c_v\widehat{\mathcal D}_{v,\mathcal B}(\Gamma)
+\varepsilon_{{\rm test},v}
+\widetilde\varepsilon_{{\rm lin},v}
+\widetilde\varepsilon_{{\rm stat},v}.
\]
\end{corollary}

\begin{proof}
Successively apply Equations~\ref{eq:finite-test-control},
\ref{eq:linear-gap-control}, and \ref{eq:finite-sample-gap-control} along the
chain in Equation~\ref{eq:appendix-surrogate-chain}. For the diagonal variant,
insert Equation~\ref{eq:diagonal-gap-control} before applying the same
finite-sample perturbation estimate to the diagonal empirical value.
\end{proof}

Corollary~\ref{cor:controlled-surrogate} is the only bridge needed here from
the ideal local criterion to the implemented statistic. It does not assert
that $\widehat{\mathcal D}_{v,\mathcal B}=\delta_v$, nor that every stochastic
optimizer step monotonically decreases the Shannon defect. Instead, it
separates the discrepancy into identifiable terms that shrink under richer
tests, a better readout class, and more statistical support, respectively.
The optional diagonal variant also incurs the covariance-geometry residual.

The preceding calibration statement is stratum-wise: it applies to any fixed
interface law for which the displayed moments and residuals are defined. The
implementation does not estimate a separate covariance problem for every
interface. Instead, aligned eligible interfaces are pooled and the same moment
functional is evaluated under the resulting pooled empirical law. Because
covariance inversion is nonlinear, we do not identify this pooled statistic
with a weighted average of the stratum-wise statistics, nor with an
independent estimator of every $\delta_v$. Pooling is used to obtain a
higher-support shared statistic. The projection guarantee below applies
directly to this actual pooled statistic through its interface-wise
differential.

\subsubsection{Interface Constraints and Projection}
\label{app:interface-constraints}
\label{app:predictive-feasibility-proof}

Although $\widehat{\mathcal D}_{\mathcal B}$ is estimated jointly from all
eligible interfaces in the current window, its differential decomposes
through the individual compressed states. Enumerate them as
$z_1^\Gamma,\ldots,z_{n_{\mathcal B}}^\Gamma$ and define, at the current
parameters $\Gamma_t$,
\[
a_i
:=
\nabla_{z_i^\Gamma}
\widehat{\mathcal D}_{\mathcal B}(\Gamma_t),
\qquad
q_i
:=
\left(\nabla_\Gamma z_i^\Gamma\right)^\top a_i
\in\mathbb R^p.
\]
The adjoint $a_i$ is obtained from the jointly estimated statistic; it is not
the gradient of a separately estimated single-interface covariance problem.
The chain rule gives
\[
\nabla_\Gamma\widehat{\mathcal D}_{\mathcal B}(\Gamma_t)
=
\sum_{i=1}^{n_{\mathcal B}}q_i.
\]
Holding the other interface routes fixed, the first-order change transmitted
through interface $i$ under a parameter displacement $d$ is
\begin{equation}
q_i^\top d+o(\|d\|_2).
\label{eq:interface-first-order-change}
\end{equation}

Let $t_t\in\mathbb R^p$ be the task proposal produced by the host optimizer.
For each group $\mathcal G_g$ used in the main-text implementation, put
$\bar q_g:=\sum_{i\in\mathcal G_g}q_i$. For $\bar q_g\ne0$, require
\begin{equation}
\bar q_g^\top d
\le
\kappa\|\bar q_g\|_2\|t_t\|_2.
\label{eq:interface-halfspace}
\end{equation}
Defining $h_g:=\bar q_g/\|\bar q_g\|_2$ gives
$h_g^\top d\le\kappa\|t_t\|_2$. Groups with $\bar q_g=0$ impose no
first-order restriction. Stack the remaining $h_g^\top$ as the rows of
$H_t$, let
\[
b_t:=\kappa\|t_t\|_2\mathbf 1_m,
\qquad
\mathcal F_t:=\{d\in\mathbb R^p:H_td\le b_t\},
\]
and define
\begin{equation}
d_t^\star
:=
\operatorname*{arg\,min}_{d\in\mathcal F_t}
\frac12\|d-t_t\|_2^2.
\label{eq:projection-primal}
\end{equation}
This is the predictive feasibility projection used in Proposition~\ref{prop:predictive-feasibility}
and Algorithm~\ref{alg:predictive-feasibility}.

Because $\kappa\ge0$, $0\in\mathcal F_t$. The feasible set is therefore a
nonempty closed convex polyhedron, while the objective in
Equation~\ref{eq:projection-primal} is strictly convex and coercive. Hence
$d_t^\star$ exists and is unique. It is the constrained update closest to
the task proposal, and if $t_t\in\mathcal F_t$ then
$d_t^\star=t_t$.

For completeness, introduce multipliers $\alpha\ge0$. Stationarity of the
Lagrangian gives
\[
d=t_t-H_t^\top\alpha.
\]
The equivalent dual problem is
\begin{equation}
\boxed{
\alpha_t^\star
\in
\operatorname*{arg\,min}_{\alpha\ge0}
\left\{
\frac12\|H_t^\top\alpha\|_2^2
+\alpha^\top(b_t-H_tt_t)
\right\},
\qquad
 d_t^\star=t_t-H_t^\top\alpha_t^\star.
}
\label{eq:projection-dual}
\end{equation}
The primal solution is unique even if the dual multiplier is not. The KKT
conditions show that the correction $t_t-d_t^\star$ is a non-negative
combination of active predictive directions.

\paragraph{Computational complexity.}
Let $p$ be the number of trainable compressor parameters, $m$ the number of
grouped predictive constraints, and $K$ the number of dual-solver iterations.
Given $H_t\in\mathbb R^{m\times p}$, a matrix-free dual solve uses products
with $H_t$ and $H_t^\top$ and costs $O(Kmp)$ time and $O(mp)$ storage. If the
Gram matrix $G_t=H_tH_t^\top$ is cached, forming it costs $O(m^2p)$, the dual
iterations cost $O(Km^2)$, and the final correction $H_t^\top\alpha_t^\star$
costs $O(mp)$, with $O(mp+m^2)$ storage. Thus the optimization is carried out
in the $m$-dimensional constraint space rather than as a $p$-dimensional QP;
here $m$ counts active constraint groups and is much smaller than $p$.

The Euclidean projection theorem gives
\[
(t_t-d_t^\star)^\top(d-d_t^\star)\le0,
\qquad
\forall d\in\mathcal F_t.
\]
Taking $d=0$ yields
\[
t_t^\top d_t^\star
\ge
\|d_t^\star\|_2^2,
\qquad
\|d_t^\star\|_2\le\|t_t\|_2.
\]
Thus the projection does not enlarge the proposed step. Moreover,
$h_g^\top t_t\le\|t_t\|_2$ by Cauchy--Schwarz, so if $\kappa\ge1$ the task
proposal is always feasible, whereas $\kappa=0$ makes every retained
group contribution non-increasing to first order.

Finally, summing Equation~\ref{eq:interface-first-order-change} within
each group and applying primal feasibility gives
\[
\boxed{
\bar q_g^\top d_t^\star
\le
\kappa\|\bar q_g\|_2\|t_t\|_2,
\qquad g=1,\ldots,m,
}
\]
and hence
\[
\Delta_g(d_t^\star)
\le
\kappa\|\bar q_g\|_2\|t_t\|_2
+o(\|d_t^\star\|_2),
\]
where $\Delta_g(d)$ denotes the combined change through group $g$.
Group feasibility does not imply the same bound for each $q_i^\top d_t^\star$.
This is intentionally a first-order guarantee on the implemented predictive
statistic. It is not a claim that each finite stochastic update monotonically
decreases the ideal mutual-information defect.

\paragraph{Summary of the theoretical boundary.}
The main text provides the ideal information-theoretic implication
\[
\text{local predictive deficiency}
\;\Longrightarrow\;
\text{root predictive stability}.
\]
This appendix provides the separate full-covariance local calibration bridge
\[
\begin{aligned}
\text{local predictive deficiency}
&\;\longrightarrow\;\text{finite tests}
\;\longrightarrow\;\text{profiled balanced value}\\
&&\;\longrightarrow\;\text{empirical statistic},
\end{aligned}
\]
with explicit residuals
$\varepsilon_{\rm test}$,
$\varepsilon_{\rm lin}$, and
$\varepsilon_{\rm stat}$. The optional diagonal variant adds
$\varepsilon_{\rm diag}$. The calibration result is stated for a fixed
interface law. Algorithm~\ref{alg:predictive-feasibility} uses the corresponding pooled moment functional
and projects the task proposal to control its interface-wise first-order
variation. No step requires the identification
$\widehat{\mathcal D}_{\mathcal B}=\delta_v$, and the pooled statistic is not
claimed to be an independent estimator of every local predictive deficiency.

\subsection{Finite-Horizon Control along Predictive Feasibility Optimization}
\label{app:training-dynamics}

This subsection gives the finite-horizon training statement associated with the projected updates in Section~\ref{sec:predictive-optimization} and connects that training-time control to Proposition~\ref{prop:root-predictive-stability}.

Proposition~\ref{prop:root-predictive-stability} is pointwise in the compressor
$\Gamma$: it converts small population local predictive deficiencies into a
small discrepancy at the root. Proposition~\ref{prop:predictive-feasibility},
by contrast, controls the linearized change of the pooled empirical statistic
used by the optimizer at one training iterate. This appendix connects the two
statements over a finite optimization horizon. The argument accounts for window refresh and pooled-to-local calibration.

At iteration $t$, write the optimization window as $\mathcal B_t$ and make the
iteration dependence of the grouped sensitivities explicit by writing
$\bar q_{g,t}$. These are the same grouped directions as $\bar q_g$ in
Appendix~\ref{app:interface-constraints}, evaluated at $(\Gamma_t,\mathcal B_t)$.
The host optimizer proposes $t_t$, predictive feasibility projection returns
$d_t^\star$, and Algorithm~\ref{alg:predictive-feasibility} sets
$\Gamma_{t+1}=\Gamma_t+d_t^\star$.

\begin{prop}[Finite-horizon control of the pooled predictive statistic]
\label{prop:finite-horizon-predictive-control}
Suppose that, at every iterate $t$, the retained groups partition the eligible
interfaces entering $\widehat{\mathcal D}_{\mathcal B_t}$, and that
$\Gamma\mapsto\widehat{\mathcal D}_{\mathcal B_t}(\Gamma)$ has an
$L_t$-Lipschitz gradient on the line segment joining $\Gamma_t$ and
$\Gamma_{t+1}$. Then the projected update satisfies
\begin{equation}
\begin{aligned}
\widehat{\mathcal D}_{\mathcal B_t}(\Gamma_{t+1})
\le{}&
\widehat{\mathcal D}_{\mathcal B_t}(\Gamma_t)
+\kappa\|t_t\|_2\sum_{g=1}^{m}\|\bar q_{g,t}\|_2
+\frac{L_t}{2}\|t_t\|_2^2.
\end{aligned}
\label{eq:one-step-pooled-drift}
\end{equation}
Consequently, for any $T\ge1$,
\begin{equation}
\begin{aligned}
\widehat{\mathcal D}_{\mathcal B_T}(\Gamma_T)
\le{}&
\widehat{\mathcal D}_{\mathcal B_0}(\Gamma_0)\\
&+\sum_{t=0}^{T-1}
\left[
\kappa\|t_t\|_2\sum_{g=1}^{m}\|\bar q_{g,t}\|_2
+\frac{L_t}{2}\|t_t\|_2^2
\right]\\
&+\sum_{t=0}^{T-1}
\left|
\widehat{\mathcal D}_{\mathcal B_{t+1}}(\Gamma_{t+1})
-\widehat{\mathcal D}_{\mathcal B_t}(\Gamma_{t+1})
\right|.
\end{aligned}
\label{eq:finite-horizon-pooled-drift}
\end{equation}
\end{prop}

\begin{proof}
For a fixed window $\mathcal B_t$, the chain-rule decomposition in
Appendix~\ref{app:interface-constraints} gives
\[
\nabla_\Gamma\widehat{\mathcal D}_{\mathcal B_t}(\Gamma_t)
=\sum_{g=1}^{m}\bar q_{g,t}.
\]
By $L_t$-smoothness,
\[
\widehat{\mathcal D}_{\mathcal B_t}(\Gamma_{t+1})
\le
\widehat{\mathcal D}_{\mathcal B_t}(\Gamma_t)
+\sum_{g=1}^{m}\bar q_{g,t}^{\top}d_t^\star
+\frac{L_t}{2}\|d_t^\star\|_2^2.
\]
Primal feasibility gives
$\bar q_{g,t}^{\top}d_t^\star
\le\kappa\|\bar q_{g,t}\|_2\|t_t\|_2$
for every retained group, while the projection property established in
Appendix~\ref{app:interface-constraints} gives
$\|d_t^\star\|_2\le\|t_t\|_2$.
Substitution proves Equation~\ref{eq:one-step-pooled-drift}.

When the optimization window is refreshed, add and subtract
$\widehat{\mathcal D}_{\mathcal B_t}(\Gamma_{t+1})$ to obtain
\[
\begin{aligned}
\widehat{\mathcal D}_{\mathcal B_{t+1}}(\Gamma_{t+1})
\le{}&
\widehat{\mathcal D}_{\mathcal B_t}(\Gamma_t)
+\kappa\|t_t\|_2\sum_{g=1}^{m}\|\bar q_{g,t}\|_2
+\frac{L_t}{2}\|t_t\|_2^2\\
&+\left|
\widehat{\mathcal D}_{\mathcal B_{t+1}}(\Gamma_{t+1})
-\widehat{\mathcal D}_{\mathcal B_t}(\Gamma_{t+1})
\right|.
\end{aligned}
\]
Summing this inequality from $t=0$ to $T-1$ gives
Equation~\ref{eq:finite-horizon-pooled-drift}.
\end{proof}

Equation~\ref{eq:finite-horizon-pooled-drift} is a finite-horizon drift bound,
not a monotonicity statement. The first term in each increment is the allowed
first-order margin, the quadratic term is the fixed-window linearization
remainder, and the final sum records the change caused by refreshing the
optimization window. In particular, the projection prevents an unconstrained
accumulation of first-order predictive drift only to the extent quantified by
these terms.

\begin{corollary}[Connection to root predictive stability]
\label{cor:training-to-root-stability}
Let $\mathcal B_T^{\rm eval}$ be an evaluation sample independent of the
training trajectory. Suppose the conditions of
Corollary~\ref{cor:controlled-surrogate} hold at $\Gamma_T$ for every compressed
interface jointly on an event of probability at least $1-\eta$. Then, on that
event, for every $v\in V_{\mathrm{comp}}$,
\begin{equation}
\begin{aligned}
\delta_v(\Gamma_T)
\le{}&
 c_v\widehat{\mathcal D}_{\mathcal B_T}(\Gamma_T)
+c_v\left|
\widehat{\mathcal D}_{v,\mathcal B_T^{\rm eval}}(\Gamma_T)
-\widehat{\mathcal D}_{\mathcal B_T}(\Gamma_T)
\right|\\
&+\varepsilon_{{\rm test},v}(\Gamma_T)
+c_v\Big(
\varepsilon_{{\rm lin},v}(\Gamma_T)
+\varepsilon_{{\rm stat},v}
\Big).
\end{aligned}
\label{eq:dynamic-local-deficiency-control}
\end{equation}
Consequently, under (M1)--(M3),
\begin{equation}
\begin{aligned}
&\mathbb E\!\left[
\operatorname{TV}\!\left(
\mathcal L(Y_\rho\mid U,C),
\mathcal L(Y_\rho\mid Z^{\Gamma_T},C)
\right)
\right]\\
&\quad\le
\Bigg\{\frac12
\sum_{v\in V_{\mathrm{comp}}}
\Bigg[
 c_v\widehat{\mathcal D}_{\mathcal B_T}(\Gamma_T)
+c_v\left|
\widehat{\mathcal D}_{v,\mathcal B_T^{\rm eval}}(\Gamma_T)
-\widehat{\mathcal D}_{\mathcal B_T}(\Gamma_T)
\right|\\
&\hspace{8.8em}
+\varepsilon_{{\rm test},v}(\Gamma_T)
+c_v\Big(
\varepsilon_{{\rm lin},v}(\Gamma_T)
+\varepsilon_{{\rm stat},v}
\Big)
\Bigg]\Bigg\}^{1/2}.
\end{aligned}
\label{eq:finite-horizon-root-control}
\end{equation}
Combining Equation~\ref{eq:finite-horizon-root-control} with
Equation~\ref{eq:finite-horizon-pooled-drift} gives an explicit finite-horizon
connection from the projected training updates to the population root
stability bound.
\end{corollary}

\begin{proof}
Apply Corollary~\ref{cor:controlled-surrogate} on the independent evaluation
sample $\mathcal B_T^{\rm eval}$ and add and subtract
$c_v\widehat{\mathcal D}_{\mathcal B_T}(\Gamma_T)$. Taking the absolute value
of the resulting pooled-to-local difference gives
Equation~\ref{eq:dynamic-local-deficiency-control}. At uncompressed interfaces,
$Z_v^{\Gamma_T}=U_v$ and the local predictive deficiency is zero. Summing the local bounds and applying Proposition~\ref{prop:root-predictive-stability} yields Equation~\ref{eq:finite-horizon-root-control}.
\end{proof}

The additional pooled-to-local term in
Equation~\ref{eq:dynamic-local-deficiency-control} is essential: the pooled
training statistic is not identified with an independent estimator of every
$\delta_v$. Thus the dynamic bridge requires no stronger identification than
the one already used in Appendix~\ref{app:tractable-local-objective}. It states
precisely when the optimization control is informative for the population
root guarantee: the cumulative projected drift, window-refresh discrepancy,
local calibration residuals, and pooled-to-local discrepancy must remain
controlled.

\subsection{Proof of Root Predictive Stability}
\label{app:root-predictive-stability}

This subsection proves Proposition~\ref{prop:root-predictive-stability} from Section~\ref{sec:predictive-compression-learning}.

This appendix proves Proposition~\ref{prop:root-predictive-stability}
under (M1)--(M3) of Appendix~\ref{app:recursive-closure}. In particular,
the host state $U_v$ may depend on the incoming compressed child tuple
$R_v^\Gamma=Z_{\operatorname{ch}(v)}^\Gamma$. We use the induced joint law
for a fixed measurable $\Gamma$ and condition throughout on the full retained
context $C$. All conditional kernels exist on standard Borel state spaces.

\begin{proof}[Proof of Proposition~\ref{prop:root-predictive-stability}]
Abbreviate $Z_v^\Gamma$ by $Z_v$, let $Y_V=(Y_v)_{v\in V}$, and set
\[
\Delta_{\mathrm{loc}}(\Gamma)
:=\sum_{v\ne\rho}I(S_v;U_v\mid Z_v,C).
\]
An uncompressed interface has $Z_v=U_v$ and contributes zero. If the sum is
infinite, the claimed bound is immediate, so assume it is finite.

\medskip\noindent\textbf{1. Factorization conditional on the predictive
tree and all compressed states.}
The child-before-parent construction in (M1)--(M2) has conditional joint
kernel
\[
\Pi_\rho(dy_\rho\mid c)
\prod_{v\ne\rho}Q_v(dy_v\mid y_{\operatorname{pa}(v)},c)
\prod_{v\ne\rho}
\widetilde K_v(du_v,dz_v\mid s_v,r_v,c),
\]
where the products denote the finite child-before-parent kernel
composition and $r_v=z_{\operatorname{ch}(v)}$. For each local kernel,
write its disintegration with respect to the compressed coordinate as
\[
\widetilde K_v(du_v,dz_v\mid s_v,r_v,c)
=K_v^\Gamma(dz_v\mid s_v,r_v,c)\,
L_v(du_v\mid s_v,z_v,r_v,c),
\]
where $K_v^\Gamma$ is the $Z_v$ marginal and $L_v$ is a version of
$P_{U_v\mid S_v,Z_v,R_v,C}$. In the recursive composition, every
$K_v^\Gamma$ factor depends only on $(y,z,c)$: its incoming state $r_v$
is a coordinate of the full $z$ tuple. The remaining $u$ factors are the
$L_v$ kernels and have disjoint output coordinates. Disintegrating the joint
law with respect to $(Y_V,Z,C)$ therefore gives
\begin{equation}
P_{U\mid Y_V,Z,C}(du\mid y,z,c)
=\bigotimes_{v\ne\rho}
P_{U_v\mid S_v,Z_v,R_v,C}(du_v\mid s_v,z_v,r_v,c)
\quad P_{Y_V,Z,C}\text{-a.s.}
\label{eq:recursive-posterior-factorization}
\end{equation}
This conditions on the \emph{full} compressed state tuple. It allows
$U_v$ to depend on its incoming child states and makes no independence claim
for the $U_v$'s given only $(Y_V,C)$.

\medskip\noindent\textbf{2. Compare with a reference that retains the
incoming child states.}
For fixed $(z,c)$, define
\[
R_{U\mid Z,C}(du\mid z,c)
:=\bigotimes_{v\ne\rho}
P_{U_v\mid Z_v,R_v,C}(du_v\mid z_v,r_v,c).
\]
By Equation~\ref{eq:recursive-posterior-factorization} and additivity of
relative entropy for finite products,
\begin{equation}
\begin{aligned}
\mathbb E\!\left[
D_{\mathrm{KL}}\!\left(P_{U\mid Y_V,Z,C}\,
\middle\|\,R_{U\mid Z,C}\right)\right]
&=\sum_{v\ne\rho}I(S_v;U_v\mid Z_v,R_v,C)\\
&=I(Y_V;U\mid Z,C)
+\mathbb E\!\left[
D_{\mathrm{KL}}\!\left(P_{U\mid Z,C}\,
\middle\|\,R_{U\mid Z,C}\right)\right].
\end{aligned}
\label{eq:recursive-information-decomposition}
\end{equation}
The second equality follows by inserting $P_{U\mid Z,C}$ into the
log-density ratio and averaging over $Y_V$ conditional on $(Z,C)$. The
finite-defect assumption below guarantees that this calculation has no
indeterminate $\infty-\infty$ term. In particular,
$I(Y_V;U\mid Z,C)$ is at most the sum in the first line.

For each compressed interface, (M3) and the fact that $Z_v$ is a function
of $U_v$ give $I(S_v;R_v\mid U_v,Z_v,C)=0$. The chain rule in its two orders
therefore gives
\[
\begin{aligned}
I(S_v;U_v\mid Z_v,R_v,C)
&=I(S_v;U_v\mid Z_v,C)
-I(S_v;R_v\mid Z_v,C)\\
&\le I(S_v;U_v\mid Z_v,C).
\end{aligned}
\]
At an uncompressed node, the left and right sides are both zero without
using (M3). Combining these facts with
Equation~\ref{eq:recursive-information-decomposition} and the chain rule
from $Y_V$ to its coordinate $Y_\rho$ yields
\begin{equation}
I(Y_\rho;U\mid Z^\Gamma,C)
\le I(Y_V;U\mid Z^\Gamma,C)
\le\sum_{v\ne\rho}I(S_v;U_v\mid Z_v^\Gamma,C).
\label{eq:root-cmi-local-defects}
\end{equation}

\medskip\noindent\textbf{3. Pass from conditional information to total
variation.}
Because $Z^\Gamma$ is a measurable function of $U$,
\[
I(Y_\rho;U\mid Z^\Gamma,C)
=\mathbb E D_{\mathrm{KL}}\!\left(
P_{Y_\rho\mid U,C}\,\middle\|\,
P_{Y_\rho\mid Z^\Gamma,C}\right).
\]
Conditional Pinsker and Jensen's inequality imply
\begin{equation}
\begin{aligned}
&\mathbb E\!\left[\operatorname{TV}\!\left(
\mathcal L(Y_\rho\mid U,C),
\mathcal L(Y_\rho\mid Z^\Gamma,C)\right)\right]\\
&\quad\le
\left\{\frac12 I(Y_\rho;U\mid Z^\Gamma,C)\right\}^{1/2}\\
&\quad\le
\left\{\frac12\sum_{v\ne\rho}
I(S_v;U_v\mid Z_v^\Gamma,C)\right\}^{1/2}.
\end{aligned}
\label{eq:root-tv-final-bound}
\end{equation}
This is precisely Proposition~\ref{prop:root-predictive-stability}.
\end{proof}

\paragraph{Scope of the optimization statement.}
The bound holds for each fixed $\Gamma$ under (M1)--(M3). The empirical
calibration in Corollary~\ref{cor:controlled-surrogate} additionally assumes
test coverage, readout and covariance conditions, and a suitable independent
evaluation sample. Proposition~\ref{prop:predictive-feasibility} controls the
linearized change of the training statistic; by itself it gives no bound on
the value of that statistic at a later iterate, particularly when the margin
is positive or projection is not used. 

\makeatletter
\@ifundefined{theorem}{\newtheorem{theorem}{Theorem}[section]}{}
\@ifundefined{corollary}{\newtheorem{corollary}[theorem]{Corollary}}{}
\@ifundefined{proof}{%
  \newenvironment{proof}
    {\par\noindent\textit{Proof.}\ }%
    {\hfill\mbox{\rule{0.55em}{0.55em}}\par}%
}{}
\makeatother

\subsection{Root-Target Supervision Does Not Imply Local Predictive Sufficiency}
\label{app:endpoint-task-quotient}

This subsection formalizes the limitation of endpoint-only supervision discussed in Section~\ref{sec:predictive-compression-learning}: matching the root-target posterior does not, in general, imply local predictive sufficiency at recursive interfaces.

This appendix isolates the limited claim used in the main text: root-target sufficiency does not, in general, imply local predictive sufficiency
for the adjacent predictive pair.  We use the same pre-compression state
$U_v$ as in the main text and the preceding appendices; no additional
source state is introduced.

Fix a compressed node $v$ and suppress the subscript $v$ when no confusion
can arise.  Let
\[
S:=S_v=(Y_v,Y_{\operatorname{pa}(v)}),
\qquad
U:=U_v,
\qquad
Z^\Gamma:=\Gamma(U),
\]
and let $T:=Y_\rho$ denote the root target.  Define the probability-valued
posterior states
\[
\mathsf R:=P_{S\mid U,C},
\qquad
\mathsf Q:=P_{T\mid U,C}.
\]
We assume that the root-target posterior is a measurable readout of the
local recursive posterior,
\begin{equation}
\mathsf Q=\Lambda(\mathsf R,C)
\qquad P\text{-a.s.}
\label{eq:endpoint-posterior-readout}
\end{equation}
for some measurable $\Lambda$.  A sufficient condition is the downstream
Markov relation
\[
U\longrightarrow S\longrightarrow T
\qquad\text{conditional on }C,
\]
because the root-target posterior is then obtained by integrating the downstream
task kernel against $P_{S\mid U,C}$.

For any representation $Z$ generated from $(U,C)$ without access to $(S,T)$,
define
\[
\delta_S(Z):=I(S;U\mid Z,C),
\qquad
\delta_T(Z):=I(T;U\mid Z,C).
\]
In particular,
$\delta_S(Z^\Gamma)=I(S_v;U_v\mid Z_v^\Gamma,C)=\delta_v(\Gamma)$.

\begin{theorem}[Task-posterior quotient]
\label{thm:task-posterior-quotient}
Under Equation~\ref{eq:endpoint-posterior-readout}, define
\[
\Delta
:=
I(S;U\mid \mathsf Q,C).
\]
Then
\begin{equation}
\boxed{
\Delta
=
I(S;\mathsf R\mid\mathsf Q,C)
=
\mathbb E\!\left[
D_{\mathrm{KL}}\!\left(
\mathsf R
\,\middle\|\,
P_{S\mid\mathsf Q,C}
\right)
\right],
\qquad
P_{S\mid\mathsf Q,C}
=
\mathbb E[\mathsf R\mid\mathsf Q,C].
}
\label{eq:endpoint-quotient-gap}
\end{equation}
Moreover,
\[
\Delta=0
\quad\Longleftrightarrow\quad
\mathsf R
\text{ is measurable with respect to }
\sigma(\mathsf Q,C)
\quad P\text{-a.s.}
\]

The root-target posterior itself is root-target sufficient but can fail to be
locally predictively sufficient:
\begin{equation}
\boxed{
\delta_T(\mathsf Q)=0,
\qquad
\delta_S(\mathsf Q)=\Delta.
}
\label{eq:endpoint-counterexample}
\end{equation}
Hence, whenever $\Delta>0$, root-target sufficiency does not imply recursive
predictive sufficiency.  In particular, under such a law there is no
distribution-free function $g$ with $g(0)=0$ satisfying
\begin{equation}
\delta_S(Z)\le g(\delta_T(Z))
\label{eq:no-endpoint-reverse-control}
\end{equation}
for every admissible representation $Z$.
\end{theorem}

\begin{proof}
Because $\mathsf R=P_{S\mid U,C}$, the posterior state $\mathsf R$ is
sufficient for $S$ relative to $U$:
\[
S\perp U\mid(\mathsf R,C).
\]
Likewise, because $\mathsf Q=P_{T\mid U,C}$,
\[
T\perp U\mid(\mathsf Q,C).
\]
Since $\mathsf R$ and $\mathsf Q$ are measurable functions of $(U,C)$ and
$\mathsf Q$ is measurable from $(\mathsf R,C)$ under
Equation~\ref{eq:endpoint-posterior-readout},
\[
I(S;U\mid\mathsf Q,C)
=
I(S;\mathsf R\mid\mathsf Q,C).
\]
The conditional mutual-information representation then gives
\[
I(S;\mathsf R\mid\mathsf Q,C)
=
\mathbb E\!\left[
D_{\mathrm{KL}}\!\left(
P_{S\mid\mathsf R,\mathsf Q,C}
\,\middle\|\,
P_{S\mid\mathsf Q,C}
\right)
\right].
\]
Because $\mathsf R$ is the posterior of $S$ given $(U,C)$ and
$\mathsf Q$ is a function of $(\mathsf R,C)$,
$P_{S\mid\mathsf R,\mathsf Q,C}=\mathsf R$ almost surely.  The tower
property gives
$P_{S\mid\mathsf Q,C}=\mathbb E[\mathsf R\mid\mathsf Q,C]$, proving
Equation~\ref{eq:endpoint-quotient-gap}.  The same identity shows that
$\Delta=0$ exactly when
$\mathsf R=P_{S\mid\mathsf Q,C}$ almost surely, equivalently when
$\mathsf R$ is $\sigma(\mathsf Q,C)$-measurable.

Taking $Z=\mathsf Q$, posterior sufficiency gives
\[
\delta_T(\mathsf Q)
=
I(T;U\mid\mathsf Q,C)
=
0,
\]
while by definition
\[
\delta_S(\mathsf Q)
=
I(S;U\mid\mathsf Q,C)
=
\Delta.
\]
Thus, if $\Delta>0$, $Z=\mathsf Q$ is an explicit representation that is
sufficient for the root target but not for the adjacent predictive pair.  Applying any proposed bound
$\delta_S(Z)\le g(\delta_T(Z))$ with $g(0)=0$ to this representation yields
$\Delta\le0$, a contradiction.  This proves
Equation~\ref{eq:no-endpoint-reverse-control}.
\end{proof}

\paragraph{An explicit compressor of the main-text form.}
Let $C$ be constant, let $A$ and $B$ be independent fair bits, and take
$U=(A,B)$, $S=(A,B)$, and $T=A$. The one-bit map $\Gamma(U)=A$ gives
$Z^\Gamma=\mathsf Q=A$. Hence $I(T;U\mid Z^\Gamma)=0$ while
$I(S;U\mid Z^\Gamma)=H(B)=\log 2$ (in nats). This example establishes the
claim within the class $Z^\Gamma=\Gamma(U)$.

\begin{corollary}[One-way control under a downstream task channel]
\label{cor:one-way-task-control}
If
$U\to S\to T$ conditional on $C$, then for every admissible representation
$Z$,
\begin{equation}
\boxed{
\delta_S(Z)
=
\delta_T(Z)
+
I(S;U\mid T,Z,C)
\ge
\delta_T(Z).
}
\label{eq:one-way-defect-decomposition}
\end{equation}
Local predictive sufficiency implies root-target sufficiency, but Theorem~\ref{thm:task-posterior-quotient} gives a counterexample with $\Delta>0$.
\end{corollary}

\begin{proof}
The downstream Markov relation and admissibility of $Z$ imply
\[
T\perp(U,Z)\mid(S,C),
\]
and hence $I(T;U\mid S,Z,C)=0$.  Applying the chain rule in the two orders
gives
\[
\begin{aligned}
I(S,T;U\mid Z,C)
&=
I(S;U\mid Z,C)
+
I(T;U\mid S,Z,C)\\
&=
\delta_S(Z),
\end{aligned}
\]
and
\[
I(S,T;U\mid Z,C)
=
I(T;U\mid Z,C)
+
I(S;U\mid T,Z,C).
\]
Equating the two expressions proves
Equation~\ref{eq:one-way-defect-decomposition}.
\end{proof}

\paragraph{Relation to the main-text claim.}
At node $v$, take $S=S_v$, $U=U_v$, and $T=Y_\rho$.  When the downstream task posterior collapses
distinct local recursive posteriors, $\Delta_v>0$.  Then the representation
$Z_v=\mathsf Q_v$ is perfectly sufficient for the root target while
\[
I(S_v;U_v\mid Z_v,C)=\Delta_v>0.
\]
Therefore root-target supervision alone provides no distribution-free reverse
control of the local predictive deficiency.  This is the only conclusion
from this appendix used by the main text; no minimum-rate information
bottleneck comparison and no additive joint-training guarantee are required.

\end{document}